\documentclass{article}
\usepackage{iclr2026_conference,times}
\iclrpreprint

\usepackage{hyperref}
\usepackage{url}

\usepackage{microtype}
\usepackage{graphicx}
\usepackage{subcaption}
\usepackage{booktabs}

\usepackage{amsmath}
\usepackage{amssymb}
\usepackage{mathtools}
\usepackage{amsthm}
\usepackage{wrapfig}
\usepackage{algorithm}
\usepackage{algorithmic}
\usepackage{booktabs, tabularx}
\usepackage{comment}
\usepackage{siunitx}
\usepackage{multirow}
\usepackage{ulem}
\usepackage[utf8]{inputenc}
\usepackage{titletoc}
\usepackage[T1]{fontenc}

\usepackage[capitalize,noabbrev]{cleveref}

\theoremstyle{plain}
\newtheorem{theorem}{Theorem}[section]

\newtheorem{lemma}[theorem]{Lemma}
\newtheorem{corollary}[theorem]{Corollary}
\theoremstyle{definition}

\newtheorem{assumption}[theorem]{Assumption}
\theoremstyle{remark}

\usepackage[textsize=tiny]{todonotes}

\usepackage[dvipsnames]{xcolor}

\title{Out-of-Distribution Generalisation
with Sequence Models in Offline Multi-Agent Reinforcement Learning}

\author{
\parbox{0.95\textwidth}{
\centering
\small
\vspace{0.5cm}
Oussama Hidaoui$^{1,2*}$ \;
Omer Ebead$^{1,2*}$ \;
Ulrich Mbou Sob$^{1*}$ \;
Siddarth Singh$^{1*}$ \;
Claude Formanek$^{1*}$ \;
Felix Chalumeau$^{1}$ \;
Omayma Mahjoub$^{1}$ \;
Sasha Abramowitz$^{1}$ \;
Ruan John de Kock$^{1}$ \;
Wiem Khlifi$^{1}$ \;
Louay Ben Nessir$^{1}$ \;
Simon Verster Du Toit$^{1}$ \;
Daniel Rajaonarivonivelomanantsoa$^{1,3}$ \;
Asim Awad Osman$^{1,2}$ \;
Arnol Manuel Fokam$^{1}$ \;
Refiloe Shabe$^{1}$ \;
Arnu Pretorius$^{1}$ \; \\[0.2cm]
$^{1}$ InstaDeep \quad
$^{2}$ AIMS \quad
$^{3}$ Stellenbosch University \\[-2pt]
$^{*}$ Equal contribution \quad
Corresponding author: \texttt{u.mbousob@instadeep.com}
}
}

\begin{document}

\maketitle
\begin{abstract}
Generalising to unseen tasks remains a fundamental challenge in offline multi-agent reinforcement learning (MARL). In this work, we present a principled analysis of zero-shot task generalisation in the offline setting and conduct an extensive empirical investigation into the scaling behaviour governing task diversity, dataset size, and network capacity. To facilitate this study, we extend offline sequence modelling architectures to handle multi-task observation and action spaces alongside variable agent counts across tasks. Our primary finding is that scaling task diversity—rather than sheer dataset size—is the dominant factor in achieving robust zero-shot transfer. Through large-scale experiments across four challenging environments (Connector, RWARE, SMAX, and LBF), we demonstrate that our multi-task approach achieves a mean improvement of \textbf{3.2x} on held-out test tasks compared to single-task models and consistently outperforms strong behaviour cloning baselines. These results suggest that the development of generalisable MARL agents should prioritise the diversity of the training distribution with varying numbers of agents, providing a roadmap for scaling offline MARL effectively.
\end{abstract}
\begin{figure}[th]
    \centering
    \includegraphics[width=0.6\linewidth]{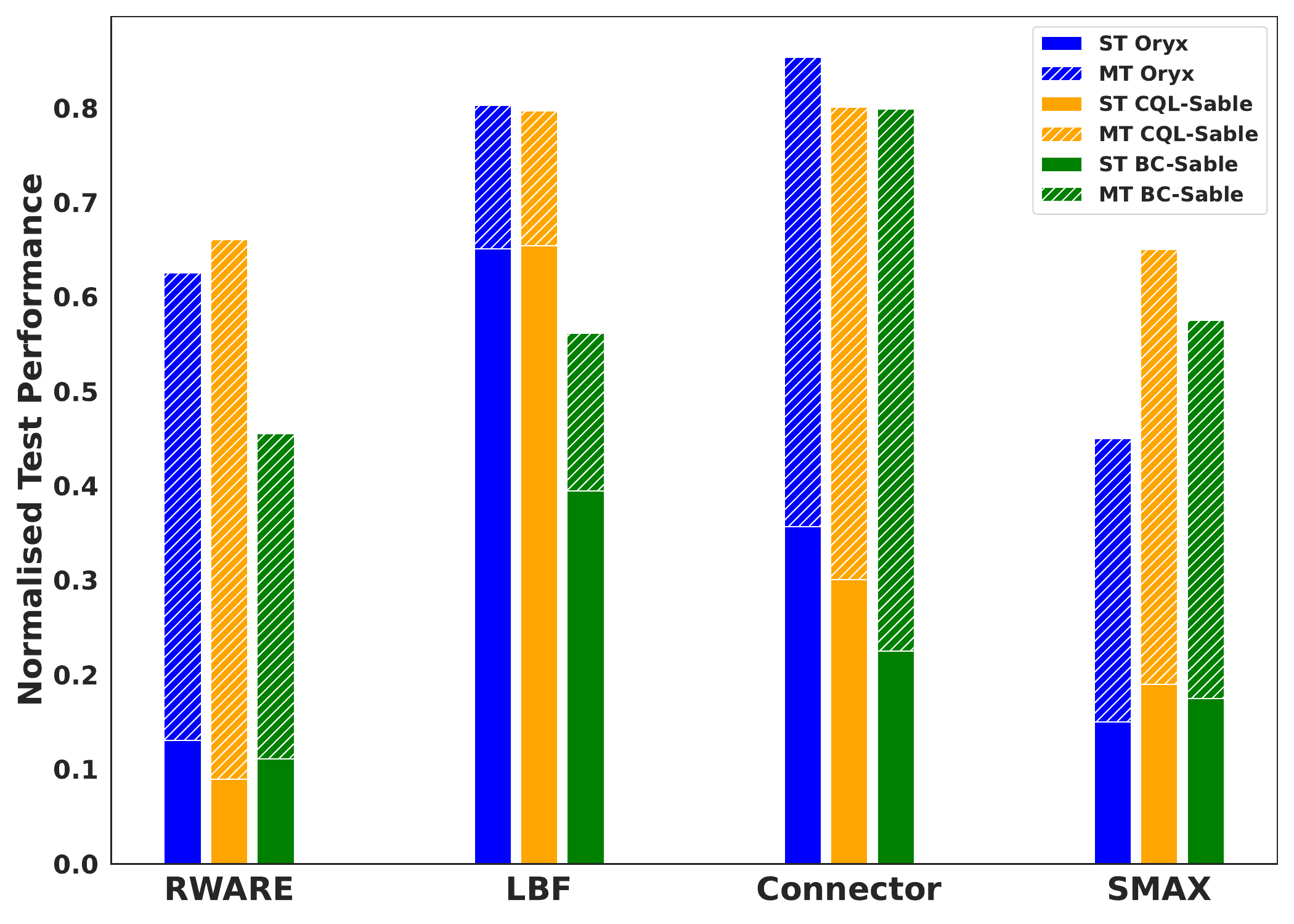} 
    \caption{\textit{Test task performance difference between single-task and multi-task sequence models.} Three multi-agent sequence models---CQL-Sable, BC-Sable and Oryx~\citep{formanek2025oryx}---were trained using either a single task (ST) or a set of multiple training tasks (MT). Average zero-shot performance was measured across a held-out set of test tasks. The upper bar represents the performance gap between ST and MT sequence models on unseen test tasks. \textbf{Averaged across all three algorithms, we observe a test performance increase of approximately 5.4x on \texttt{RWARE}, 1.3x on \texttt{LBF}, 2.9x on \texttt{Connector} and 3.2x on \texttt{SMAX}.}}
    \label{fig:abstract}
\end{figure}

\section{Introduction}

Building agents that generalise to tasks beyond those present in their training data is a central challenge in reinforcement learning (RL), and a prerequisite for deploying agents in the real world~\citep {kirk2023survey}. In many domains, collecting fresh data online by interacting with a live system is costly or risky, so practitioners turn to offline RL from logged trajectories~\citep{levine2020offline}. While single-agent work has studied the train--test generalisation gap~\citep{medirattageneralization}, the multi-agent case remains under-explored. Despite recent progress in offline multi-agent reinforcement learning (MARL)~\citep{yang2021believe,shao2023counterfactual,meng2023offline,li2025dof,formanek2025oryx}, prior work have largely been restricted to training and evaluating on the same task, without examining generalisation to unseen tasks. 


In this work, we study the generalisation of single-task models, and then introduce a challenging multi-task benchmark for offline MARL, which builds on widely adopted MARL environments including \texttt{Level Based Foraging} (\texttt{LBF}), \texttt{Multi-Robot Warehouse} (\texttt{RWARE})~\citep{papoudakis2021benchmarking}, \texttt{Connector}~\citep{bonnetjumanji} and \texttt{SMAX}~\citep{jaxmarl}. Using this benchmark, we evaluate three state-of-the-art offline multi-agent sequence models, namely Oryx~\citep{formanek2025oryx}, and two offline versions of Sable~\citep{mahjoub2025sable} (CQL-Sable and BC-Sable). Across all four environments, we show that these models exhibit poor generalisation when trained only on a dataset from a single task. However, when trained \textit{simultaneously} on a dataset consisting of a diverse set of multiple tasks, their ability to zero-shot transfer to unseen tasks significantly improves. Furthermore, we verify that similar results cannot be obtained by simply increasing the size of the dataset for a fixed number of tasks, but rather that the key driver is increasing dataset diversity by adding more tasks, which consistently leads to improved test performance. Finally, we find that for a fixed data budget, increasing the model's capacity has a positive impact on generalisation for challenging tasks.


Our findings show that offline MARL sequence models trained on diverse multi-task datasets exhibit promising generalisation to unseen tasks compared to single-task alternatives. In contrast to the findings of~\citet{medirattageneralization}, we observe that offline sequence models consistently outperform behaviour cloning, which remains a surprisingly strong baseline. Finally, our work provides the first promising evidence of performance scaling~\citep{hilton2023scalinglawssingleagentreinforcement} with increasing model capacity for offline MARL on challenging unseen tasks. 

In summary, our main contributions are as follows:
\begin{itemize}
    \item We develop a challenging multi-task offline MARL evaluation suite, which includes \num{28} large training sets and \num{34} test sets across four diverse environments, including \texttt{LBF}, \texttt{Connector}, \texttt{RWARE} and \texttt{SMAX}.
    \item We conduct an in-depth empirical analysis of out-of-distribution generalisation as a function of task diversity, data and model size, showing that zero-shot generalisation of sequence models scales significantly (3.2x on average) as the number of tasks in the training data increases, that sheer dataset size is not the main driver of test performance, and that for difficult tasks, model scaling positively affects generalisation.
    \item We provide a theoretical analysis of generalisation in offline multi-task MARL, establishing \textit{task coverage} to be a key driver of out-of-distribution performance gains, formally grounding our empirical findings on the roles of task diversity and model capacity.
    \item \textbf{All of our (anonymised) code is available for download.\footnote{\url{https://sites.google.com/view/multi-task-marl}}  We will make all of our code and datasets publicly available upon publication.}
\end{itemize}

\begin{figure*}
    \centering
    \includegraphics[width=0.95\linewidth]{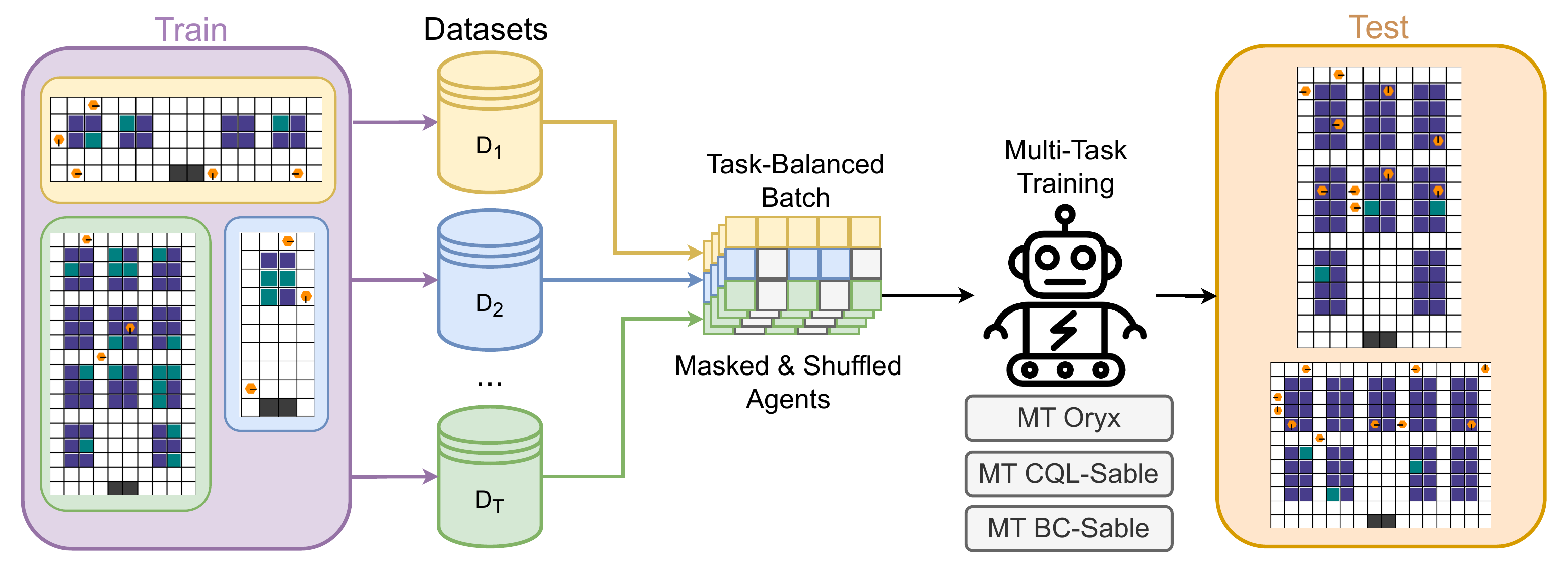}
    \caption{\textit{Our offline multi-task multi-agent training and testing setup}. In this setup, there is a set of training tasks, each with a static dataset of pre-collected trajectories that together form a diverse multi-task dataset. This dataset is then used for training, without any additional online interactions with either the training tasks or the testing tasks. At evaluation time, the trained model is evaluated on each of the held-out test tasks, and the average test performance is calculated. In this illustration, we used RWARE tasks in the train and test sets.}
    \label{fig:visual-problem-forrmulation}
\end{figure*}

\section{Multi-Task Sequence Modelling for Offline MARL} \label{sec:method}

\subsection{Preliminaries}
\textbf{Problem formulation.} 
We formalise a cooperative MARL \textit{task} as a Decentralised Partially Observable Markov Decision Process (Dec-POMDP) \citep{kaelbling1998planning}, defined by the tuple 
$\mathcal{M}_{\dagger} = \langle \mathcal{N}, \mathcal{S}, \boldsymbol{\mathcal{A}}, P, R, \{\Omega^i\}_{i \in \mathcal{N}}, \{E_i\}_{i \in \mathcal{N}}, \gamma \rangle$, where $\dagger$ denotes the particular task selected from an environment. For example, in a simulated robotic warehouse environment, a task corresponds to a specific warehouse layout and the number of robotic workers collecting and depositing requested shelf items. 
At each timestep $t$ within a task, the environment is in state $s_t \in \mathcal{S}$. 
Each agent $i \in \mathcal{N}$ selects an action $a^i_t \in \mathcal{A}^i$ based on its local action-observation history 
$\tau^i_t = (o^i_0, a^i_0, \dots, o^i_t)$. 
The agents’ actions form a joint action $\boldsymbol{a}_t \in \boldsymbol{\mathcal{A}} = \prod_{i \in \mathcal{N}} \mathcal{A}^i$, 
which, when executed, yields a shared reward $r_t = R(s_t, \boldsymbol{a}_t)$, transitions the environment to 
$s_{t+1} \sim P(\cdot | s_t, \boldsymbol{a}_t)$, and provides each agent $i$ with a new observation 
$o^i_{t+1} \sim E_i(\cdot | s_{t+1}, \boldsymbol{a}_t)$. 
The agent then updates its history as $\tau^i_{t+1} = (\tau^i_t, a^i_t, o^i_{t+1})$. 
The task-specific objective is to learn a joint policy $\pi(\boldsymbol{a} | \boldsymbol{\tau})$ that maximises the expected discounted return over a horizon of timesteps $H$:
$
J_{\dagger
}(\boldsymbol{\pi}) = \mathbb{E}_{\pi} \!\left[ \sum_{t=0}^{H} \gamma^{t} r_t \right].$

To create our train-test evaluation setup, we consider offline datasets $\mathcal{D}_\text{train} = \{\mathcal{D}_\dagger : \dagger \in \mathcal{T}_\text{train}\}$ collected from a set of training tasks $\mathcal{T}_\text{train}$.  
Our objective is to learn a single joint policy $\boldsymbol{\pi}_\text{train}$, using only the fixed multi-task training data (i.e.\ without any additional online interaction), to maximise the expected zero-shot performance on a set of \textit{unseen} test tasks $\mathcal{T}_\text{test}$, given as  
\begin{equation*}
J(\boldsymbol{\pi}) = \mathbb{E}_{\dagger \sim \mathcal{T}_{\text{test}}} \!\left[ J_\dagger(\boldsymbol{\pi}) | \boldsymbol{\pi} = \boldsymbol{\pi}_\text{train} \right].    
\end{equation*}
By optimising the above objective, we are minimising the generalisation gap between training and test tasks.
A simplified visual representation of the problem setting is depicted in \autoref{fig:visual-problem-forrmulation}.

\textbf{Multi-Agent Sequence Models.}
Centralised control, where a single policy outputs the joint action, is theoretically optimal but scales poorly due to an exponential growth of the action space~\citep{de2025exponentially}.
However, autoregressive factorisation is an efficient way to parametrise the joint policy, by expressing the joint distribution over $n$ agents as a product of conditional distributions:
\begin{equation*}
\pi(\boldsymbol{a} | \boldsymbol{\tau}) = 
\prod_{k=1}^{n} \pi^{i_k}\!\left(a^{i_k} \mid \boldsymbol{\tau}, a^{i_1}, \dots, a^{i_{k-1}}\right).
\end{equation*}
Here $i_{k}$ denotes an agent index from an ordered set $\{i_1, \dots, i_n\} \in S_n$, where $S_n$ is the set of permutations of $\{1, ..., n\}$.
This factorisation decomposes joint decision-making into a sequence of conditional actions, 
enabling scalable coordination, efficient parallel training and, in certain cases, providing desirable convergence properties~\citep{zhong2024heterogeneous}.
Sequence models provide a natural parameterisation of such policies, closely mirroring the autoregressive next token prediction process in text and image generation, and have been demonstrated to work well on a large range of MARL settings~\citep{wen2022multi,mahjoub2025sable,daniel2024multi,formanek2025oryx}. 

\subsection{Multi-Task Sequence Models for Offline MARL}
\label{ssec:mt-sec-models}

Building on existing multi-agent sequence models for offline MARL \citep{formanek2025oryx}, we propose a few simple yet essential modifications that enable training on multiple tasks with varying numbers of agents simultaneously, while allowing seamless zero-shot transfer. By design, our multi-task sequence models do not receive explicit task IDs or have task-specific output heads, since this would limit their zero-shot transferability to new tasks. Instead, our models have to infer task information from observations, agent counts, and environment dynamics. 

\textbf{Dynamic agent padding, shuffling and masking.}
In order to dynamically handle variable numbers of agents across tasks, we zero-pad the inputs for absent agents and mask their contributions in the loss. Moreover, we randomise the ordering of both active and inactive agents at each training update, which encourages the model to share representations and transfer knowledge across agents. 

\textbf{Multi-task training loss.}
Given a set of training tasks $\mathcal{T}_\text{train}=\{\dagger_1,\ldots,\dagger_M\}$, with offline buffers $\{\mathcal{D}_{\dagger}\}_{\tau\in\mathcal{T}_\text{train}}$, we train a multi-task sequence model by minimising the average per-task loss
\begin{equation*}
\min_{\theta}\;\; \frac{1}{M}\sum_{\dagger\in\mathcal{T}_\text{train}}
\Big[\mathcal{L}(\theta;\mathcal{D}_\dagger)\Big].
\label{eq:mt-erm}
\end{equation*}
The loss $\mathcal{L}$ changes depending on the algorithm used, which in our case includes autoregressive versions of behaviour cloning (BC)~\citep{pomerleau1988alvinn,bain1995framework}, Conservative Q-learning (CQL)~\citep{kumar2020conservative} and Implicit Constraint Q-learning (ICQ)~\citep{yang2021believe,formanek2025oryx}.

\textbf{Task-balanced batching.}
For each training update, we build a single unified mini-batch by evenly sampling across different tasks. Given a batch size $B$, we compute $q=\lfloor~{B}~/~{|\mathcal{T}_\text{train}|}~\rfloor$ and $r~=~B-q|\mathcal{T}_\text{train}|$. Each task $\dagger \in \mathcal{T}_\text{train}$, contributes $q$ samples; the remaining $r$ samples are assigned by round-robin across tasks up to the value $r$.
This yields stochastic gradients that are unbiased over a uniform mixture of tasks (each task equally weighted), rather than a size-weighted mixture. The resulting task-balanced batching also mitigates ``head-task'' dominance seen with dataset-proportional sampling, a known issue in domain generalisation from long-tailed datasets \citep{cui2019class}.

\textbf{Value function learning via classification.}
To mitigate gradient interference from varying reward scales across tasks, we replace scalar TD regression with a classification objective. Specifically, we use HL-Gauss \citep{imani2018improving,farebrother2024stop}, which projects each scalar TD target onto a discrete support by smoothing with a Gaussian distribution, and trains the value function with categorical cross-entropy over the resulting histogram. This choice, consistent with prior multi-task training architectures \citep{kumar2022offline}, improves stability and reduces loss-scale sensitivity compared to mean squared error.

We provide ablation studies validating these specific design choices in \autoref{sec:design-choices-ablations}.

\section{Empirical Analysis}
\label{sec:empirical-analysis}

\subsection{Experimental design}

\textbf{Tasks.} We considered four challenging MARL environments, \texttt{LBF}, \texttt{RWARE}~\citep{papoudakis2021benchmarking}, \texttt{Connector}~\citep{bonnetjumanji} and \texttt{SMAX}~\citep{jaxmarl}. These are all widely used MARL benchmarks, with \texttt{RWARE} also proposed as a suitable multi-task benchmark in previous work~\citep{lukas2022rwaremt} and \texttt{Connector} being of particular interest due to its agent scaling properties (see \citet{formanek2025oryx}). For each environment, we selected several different level configurations to serve as distinct tasks. These tasks were then partitioned into train and test sets (see \autoref{sec:task-splits}), taking care to ensure that the test tasks were different in meaningful ways to the training tasks, as shown in \autoref{fig:train-test scatter}.

\begin{figure*}
    \centering
    \begin{subfigure}[b]{0.24\linewidth}
        \includegraphics[width=\linewidth]{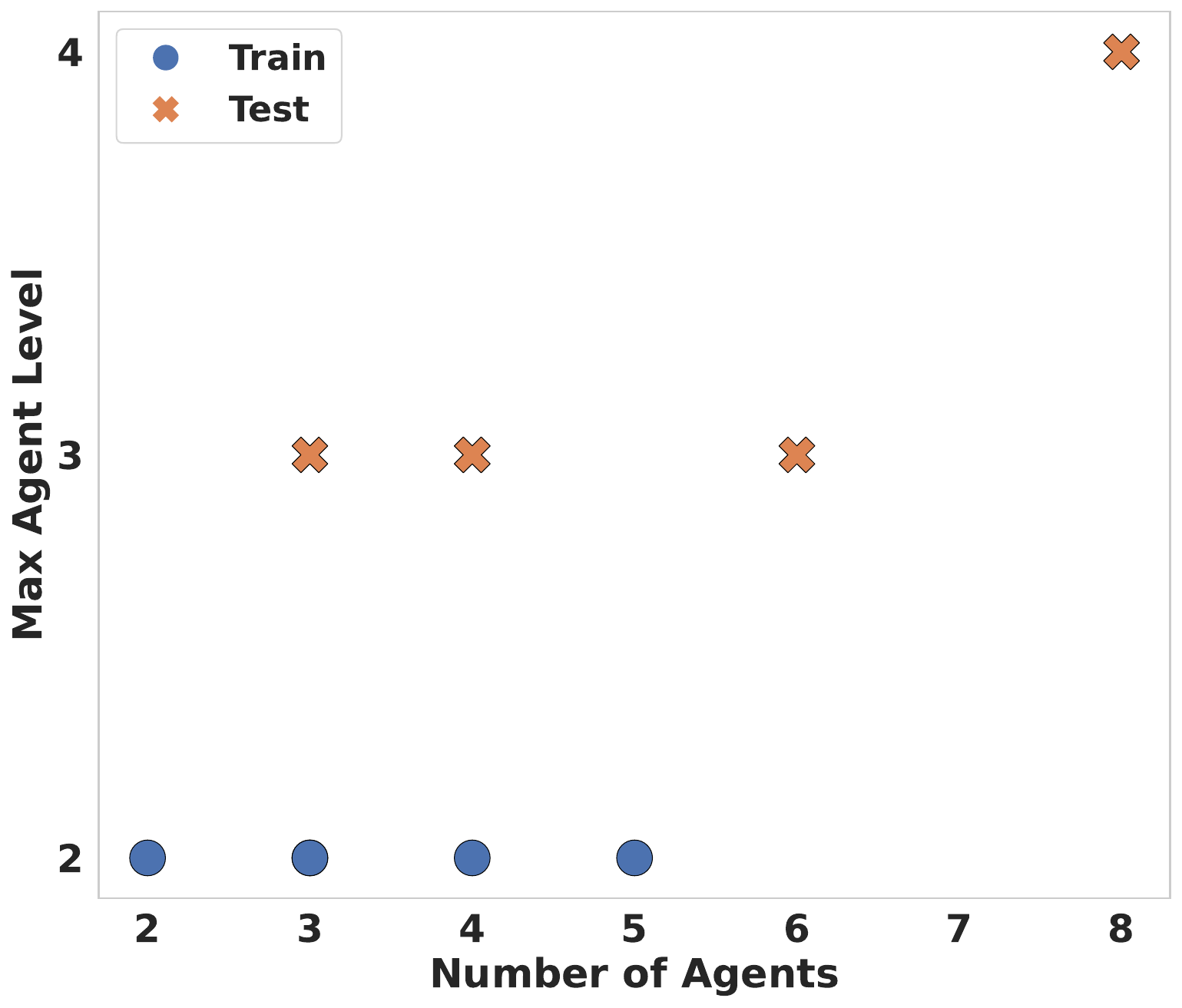}
        \caption{\texttt{LBF}}
    \end{subfigure}
    \hfill
    \begin{subfigure}[b]{0.24\linewidth}
        \includegraphics[width=\linewidth]{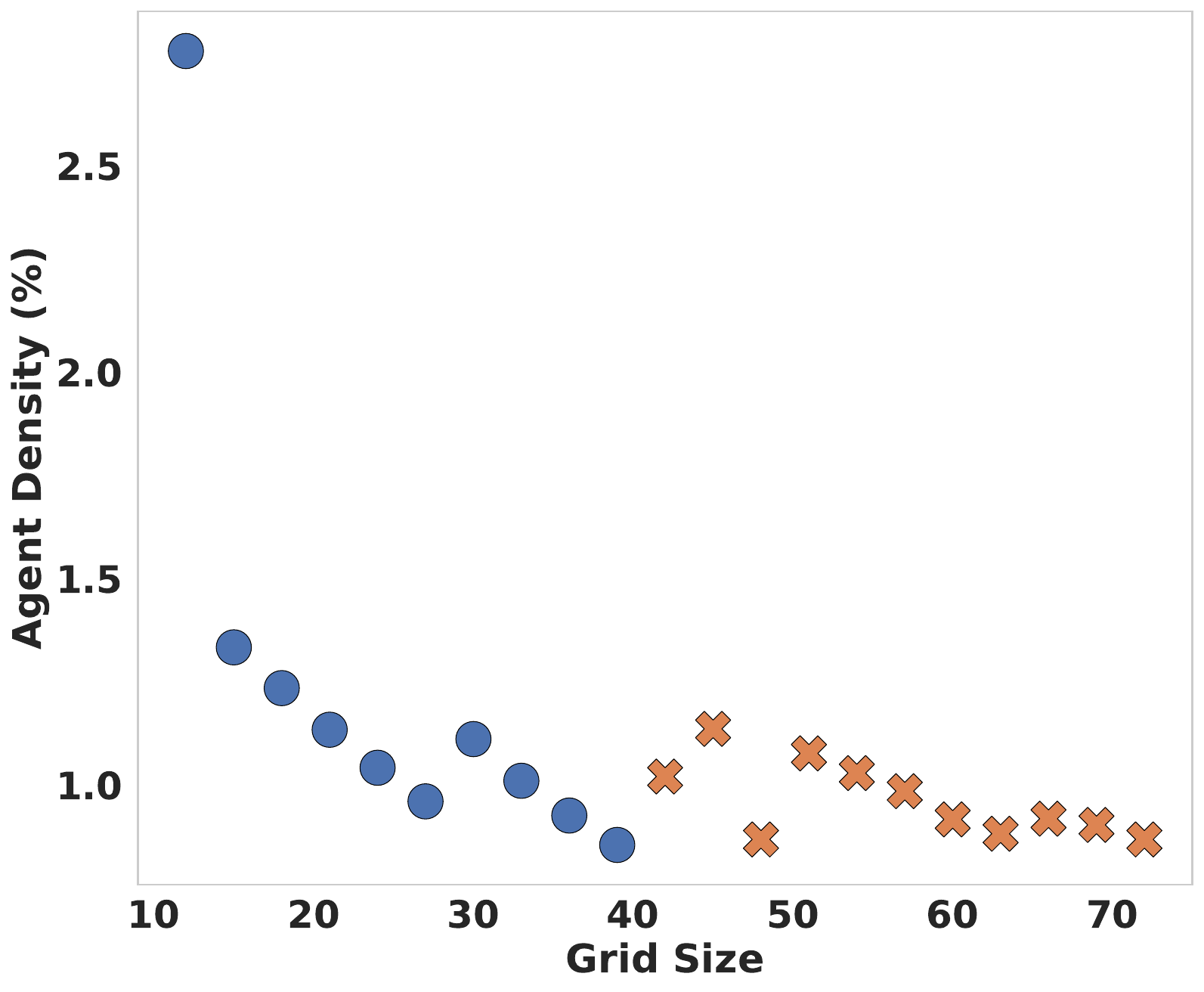}
        \caption{\texttt{Connector}}
    \end{subfigure}
    \hfill
    \begin{subfigure}[b]{0.24\linewidth}
        \includegraphics[width=\linewidth]{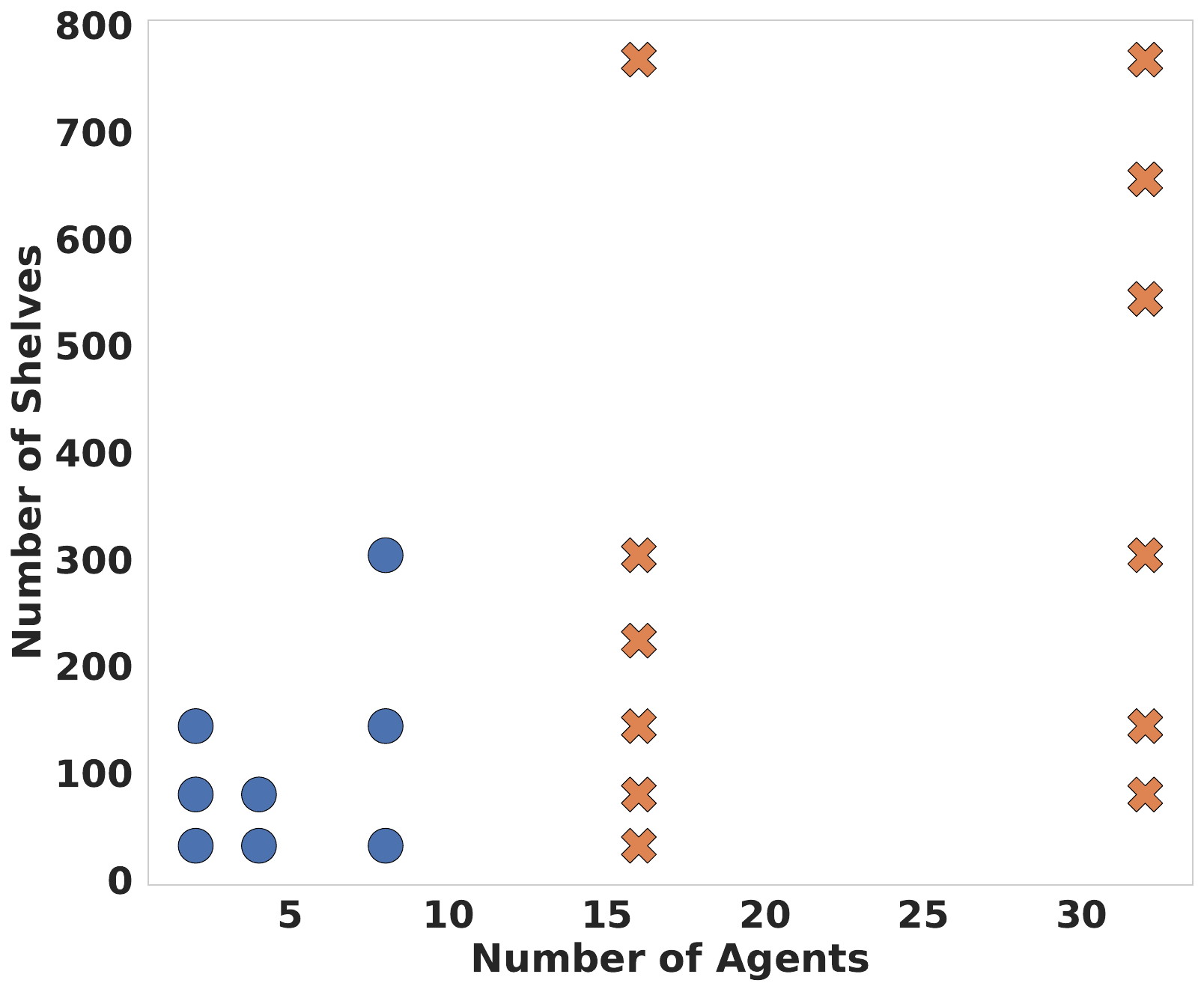}
        \caption{\texttt{RWARE}}
    \end{subfigure}
    \hfill 
    \begin{subfigure}[b]{0.24\linewidth}
        \includegraphics[width=\linewidth]{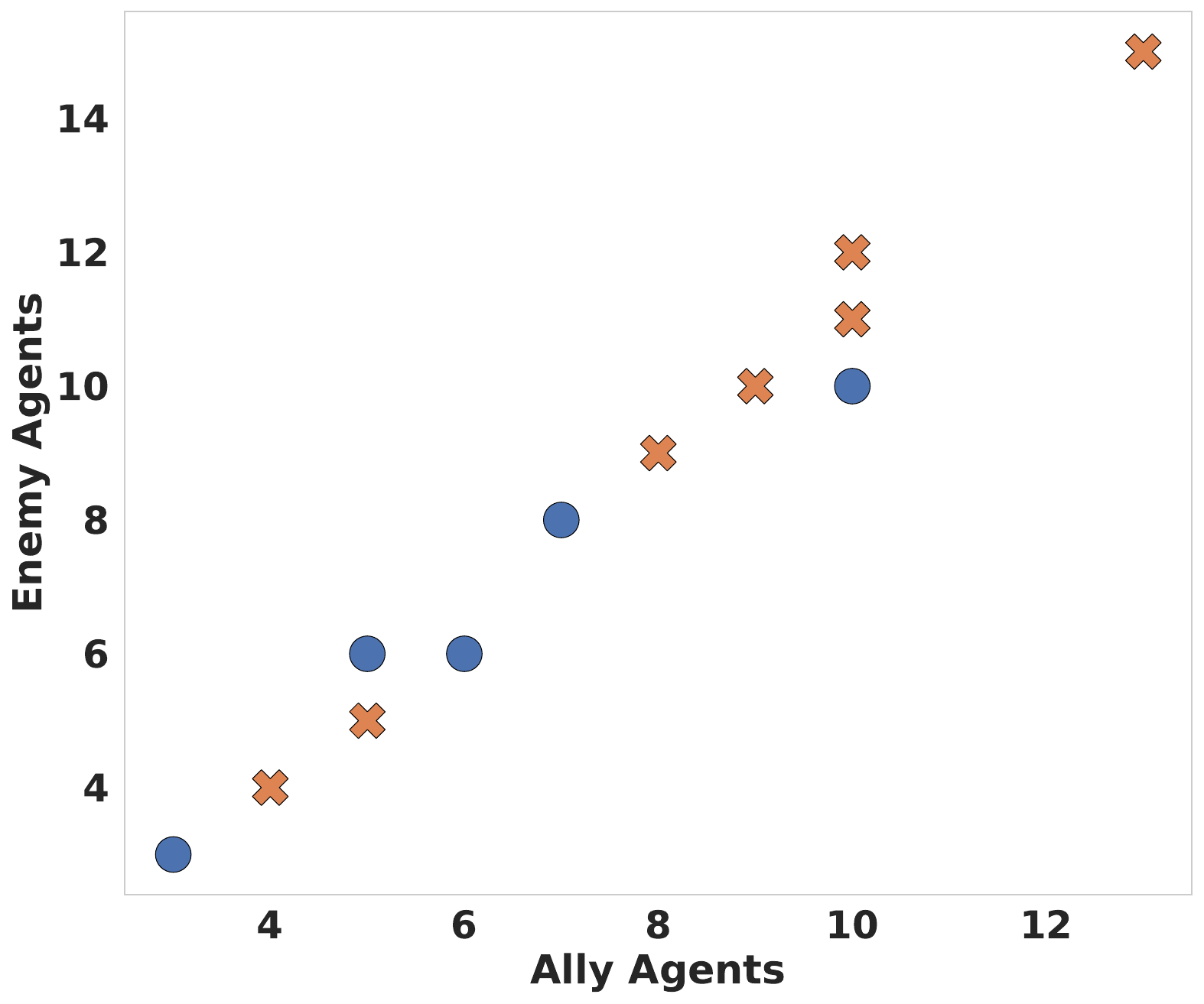}
        \caption{\texttt{SMAX}}
    \end{subfigure}
     \caption{\textit{Distributional shift between train and test tasks.} Each point represents a task, with key task properties plotted to illustrate distributional differences. In \texttt{LBF}, we plot the number of agents against the maximum agent level; in \texttt{Connector}, grid size against agent density; in \texttt{RWARE}, the number of agents against the number of shelves; and in \texttt{SMAX}, the number of ally agents against the number of enemy agents. While these dimensions are important for distinguishing tasks, additional parameters vary across tasks but are not shown here (e.g., the layout of shelves in \texttt{RWARE} tasks).}
     \label{fig:train-test scatter}
\end{figure*}


\textbf{Datasets.}
For each task, we construct an offline dataset $\mathcal{D}_\dagger$ by recording a set of rollouts at fixed intervals from an online training run of \textsc{Sable}~\citep{mahjoub2025sable}, a state-of-the-art MARL sequence model. This yields a mixed dataset with the same number of rollouts per task but not necessarily the same number of transitions, since episode lengths differ across tasks, hence the necessity for task-balanced batching. Observations and actions are standardised per environment. For sequence modeling, we sample fixed-length trajectory chunks (context length reported with other hyperparameters in \autoref{sec:hyperparameters}). Rewards are left unclipped during training and for comparability across tasks, we report normalised returns, where each task's episode return is normalised by the final episode return achieved by the online system on that task. Further details on the datasets can be found in \autoref{sec:appendix_datasets}.

\textbf{Algorithms.} 
The first algorithm we consider is an adapted version of Oryx~\citep{formanek2025oryx}, which we modify for multi-task training. As described in \autoref{sec:method}, these modifications include (i) dynamic padding, masking, and agent shuffling, (ii) task-balanced batching, and (iii) value learning using HL-Gauss~\citep{farebrother2024stop}. We refer to this variant as \textbf{Multi-Task (MT) Oryx}.

The second algorithm, \textbf{MT BC-Sable}, is an offline variant of Sable that uses behaviour cloning to train an autoregressive policy, together with dynamic padding and masking of agents and task-balanced batching. The third algorithm, \textbf{MT CQL-Sable}, is another offline variant of Sable that employs an autoregressive version of the CQL loss~\citep{kumar2020conservative}, while incorporating the same multi-task enhancements as MT Oryx.

All three algorithms share the same Sable network backbone; thus, the primary distinction between them lies in the choice of loss function~$\mathcal{L}$. We select CQL due to its demonstrated generalisation and scaling capabilities in the single-agent setting~\citep{kumar2022offline,chebotar2023q}, and behaviour cloning for its strong and competitive generalisation performance reported in prior work~\citep{medirattageneralization}. Hyperparameter details for all three algorithms are provided in \autoref{sec:hyperparameters}, and the computational requirements are described in \autoref{sec:compute_requirements}.

\textbf{Evaluation protocol.} 
In our experiments, we evaluate the expected zero-shot performance of trained models on held-out test tasks. To this end, we compute the absolute episode return~\citep{gorsane2022towards} by evaluating the best-performing checkpoint during training over \num{320} independent episodes and averaging the resulting returns for each test task. 

To enable comparisons across tasks and environments with potentially different reward scales, we normalise the absolute episode return by dividing it by the maximum expected episode return achieved for the corresponding task during online data collection. Each experimental configuration is repeated over three random seeds, and we report the mean and standard deviation across runs.

\begin{figure*}[t]
  \centering

  \begin{subfigure}[t]{0.24\linewidth}
    \centering
    \includegraphics[width=\linewidth]{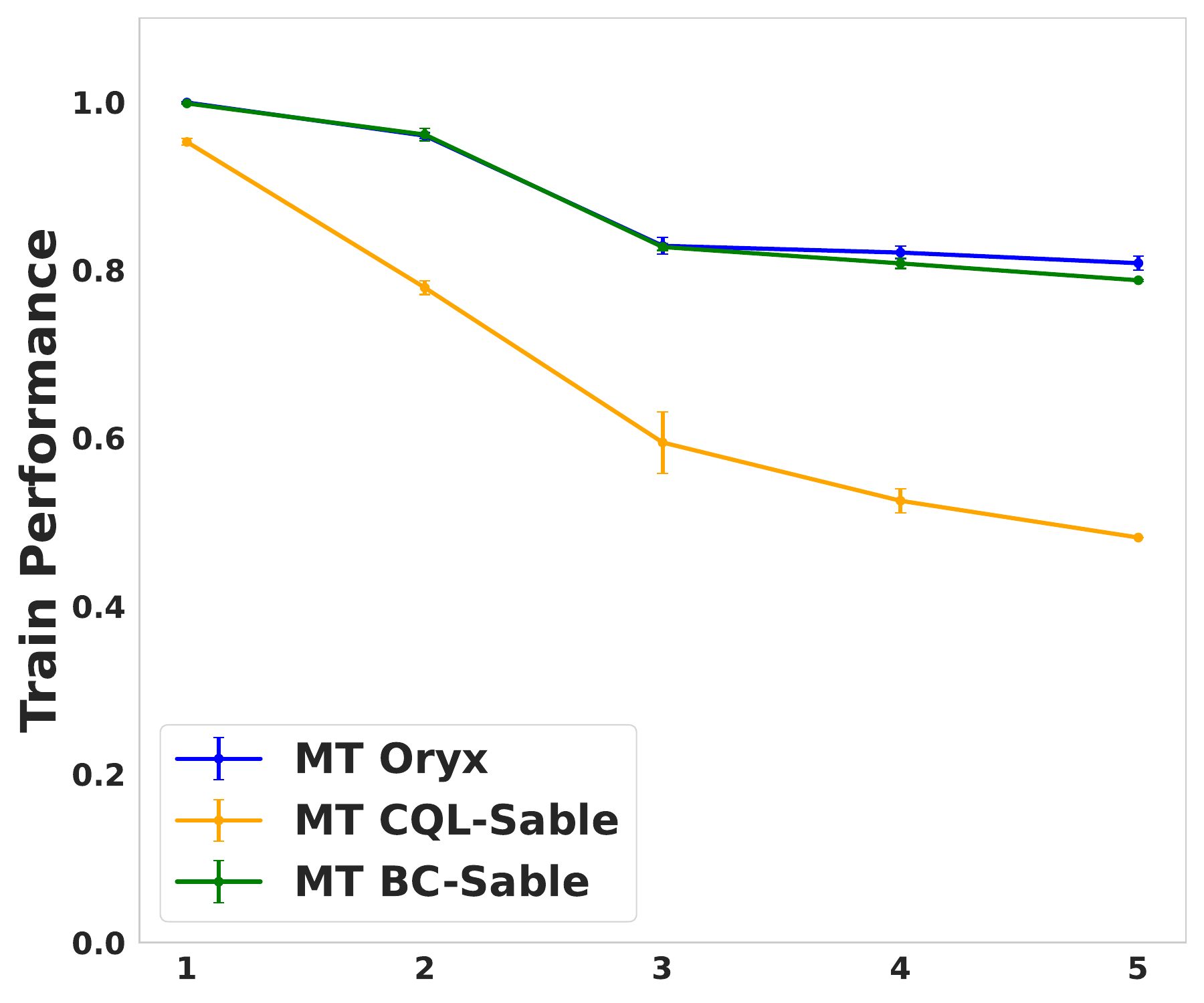}
  \end{subfigure}\hfill
  \begin{subfigure}[t]{0.24\linewidth}
    \centering
    \includegraphics[width=\linewidth]{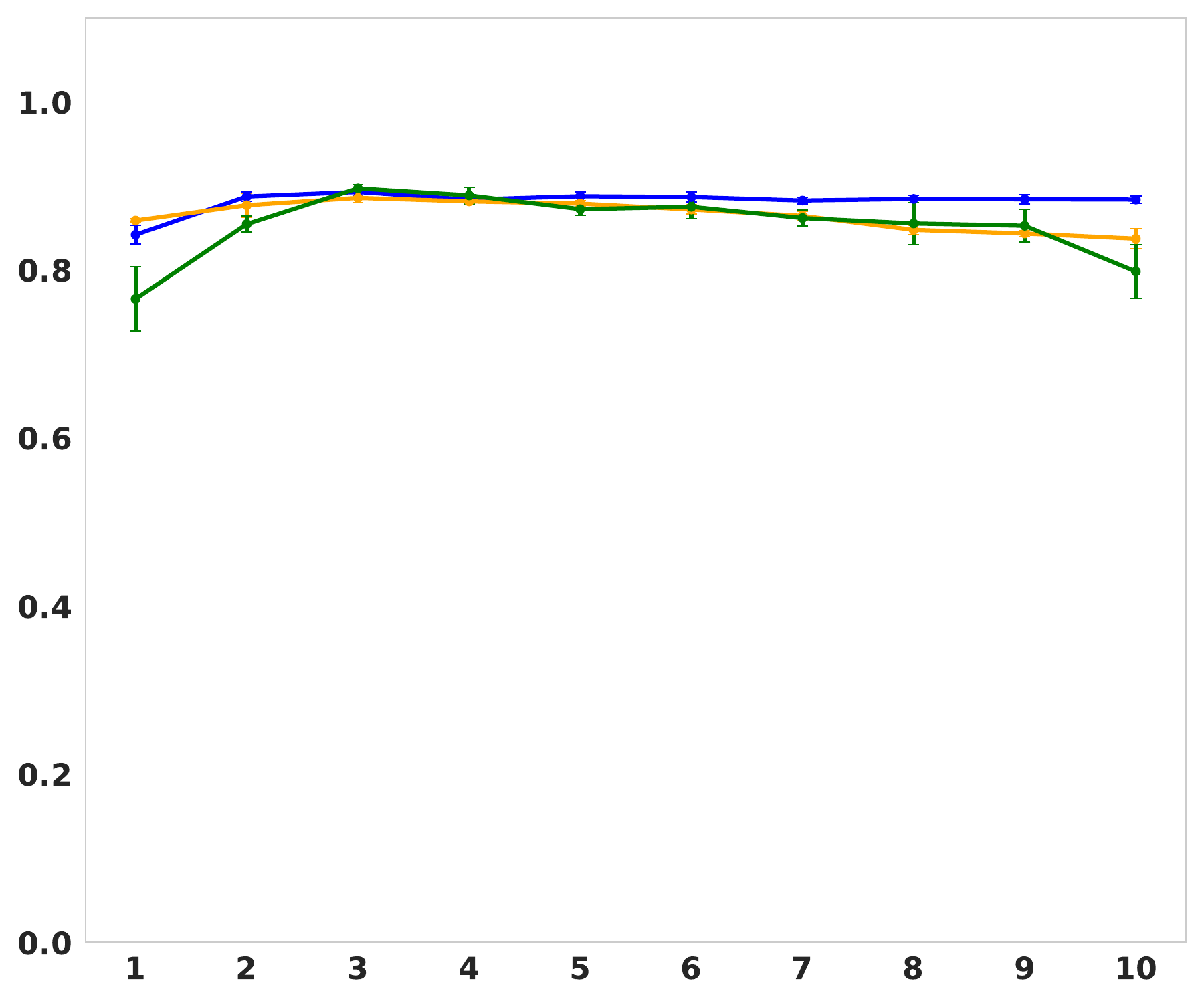}
  \end{subfigure}\hfill
  \begin{subfigure}[t]{0.24\linewidth}
    \centering
    \includegraphics[width=\linewidth]{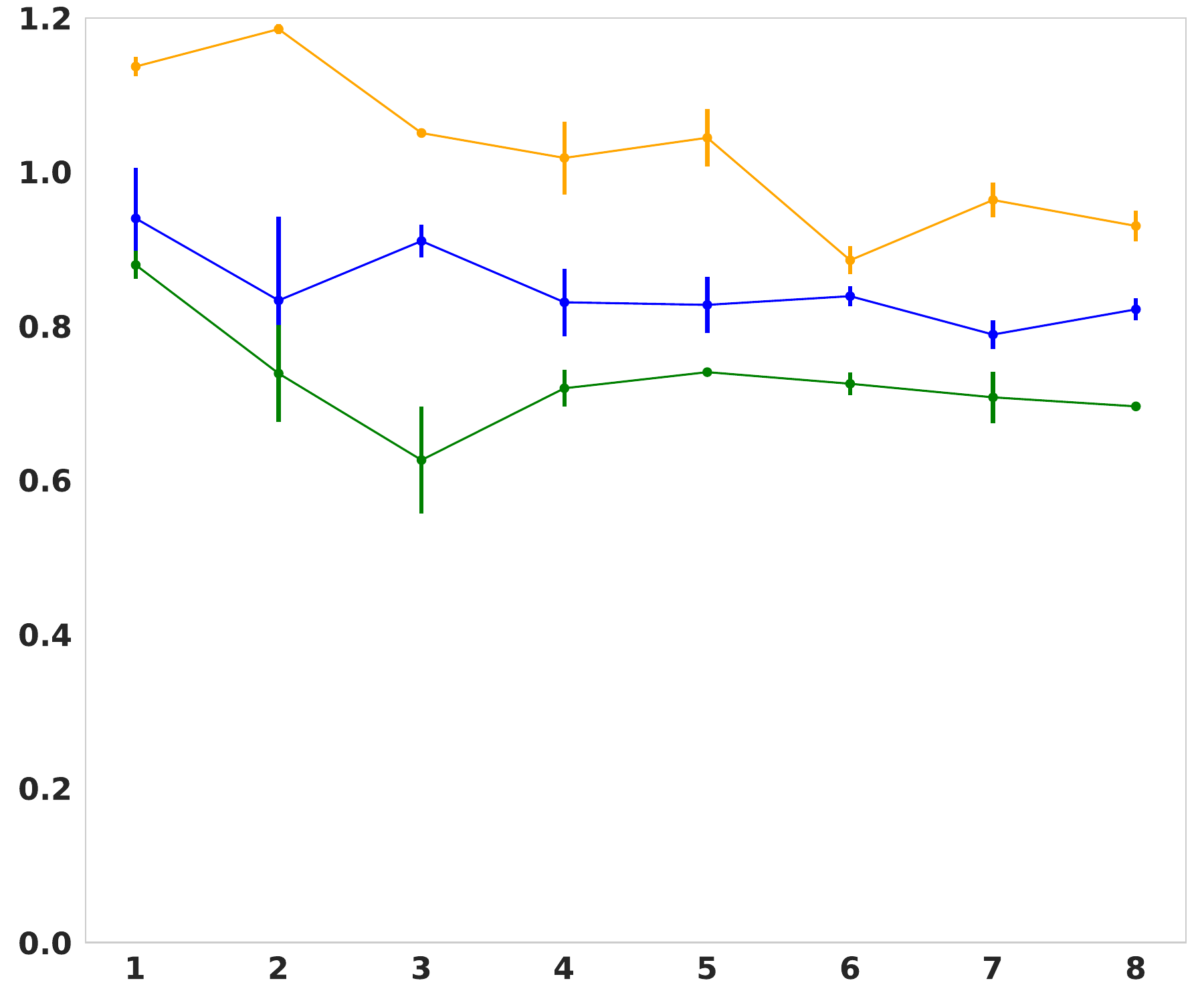}
  \end{subfigure}\hfill
  \begin{subfigure}[t]{0.24\linewidth}
    \centering
    \includegraphics[width=\linewidth]{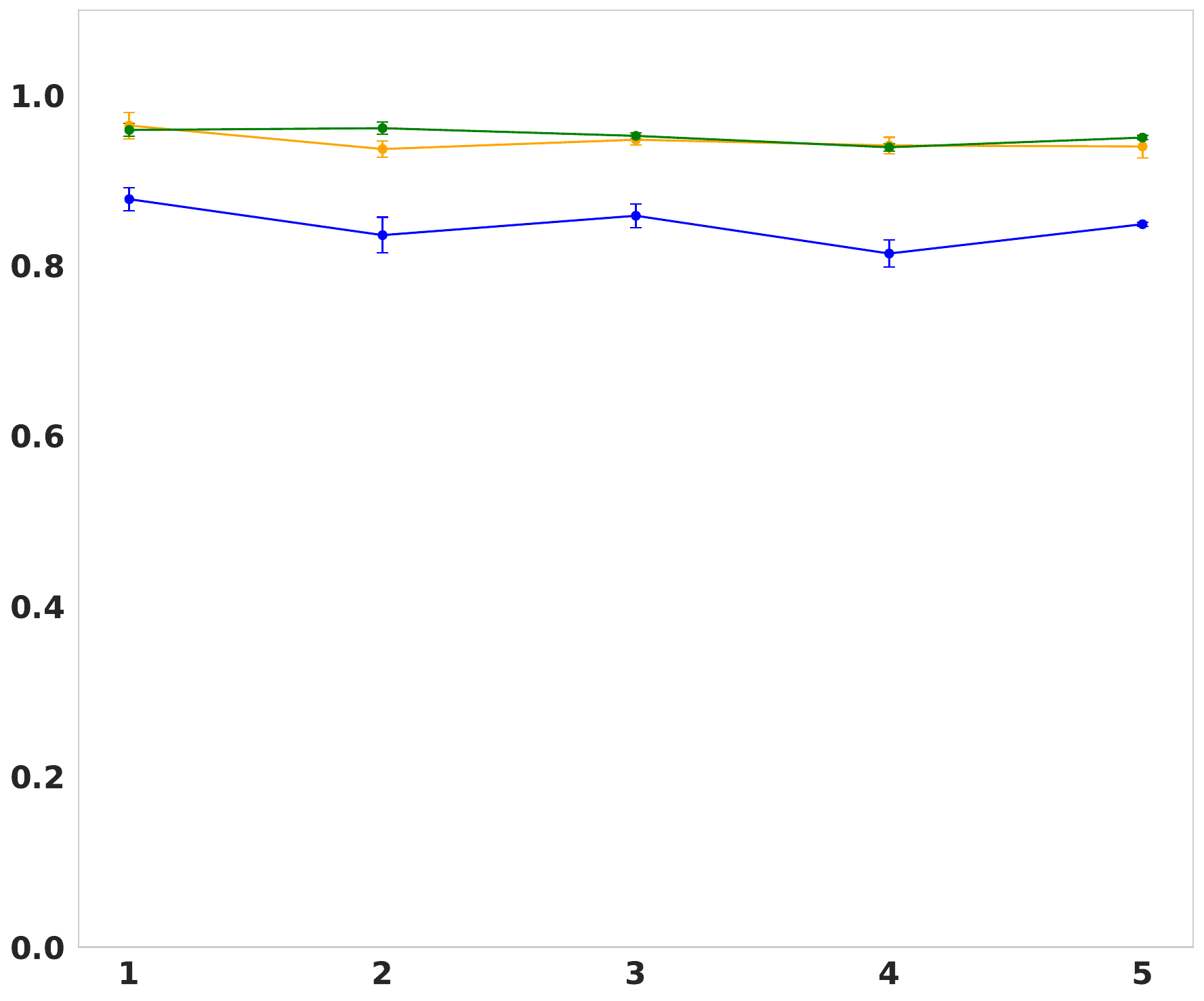}
  \end{subfigure}


  \begin{subfigure}[t]{0.24\linewidth}
    \centering
    \includegraphics[width=\linewidth]{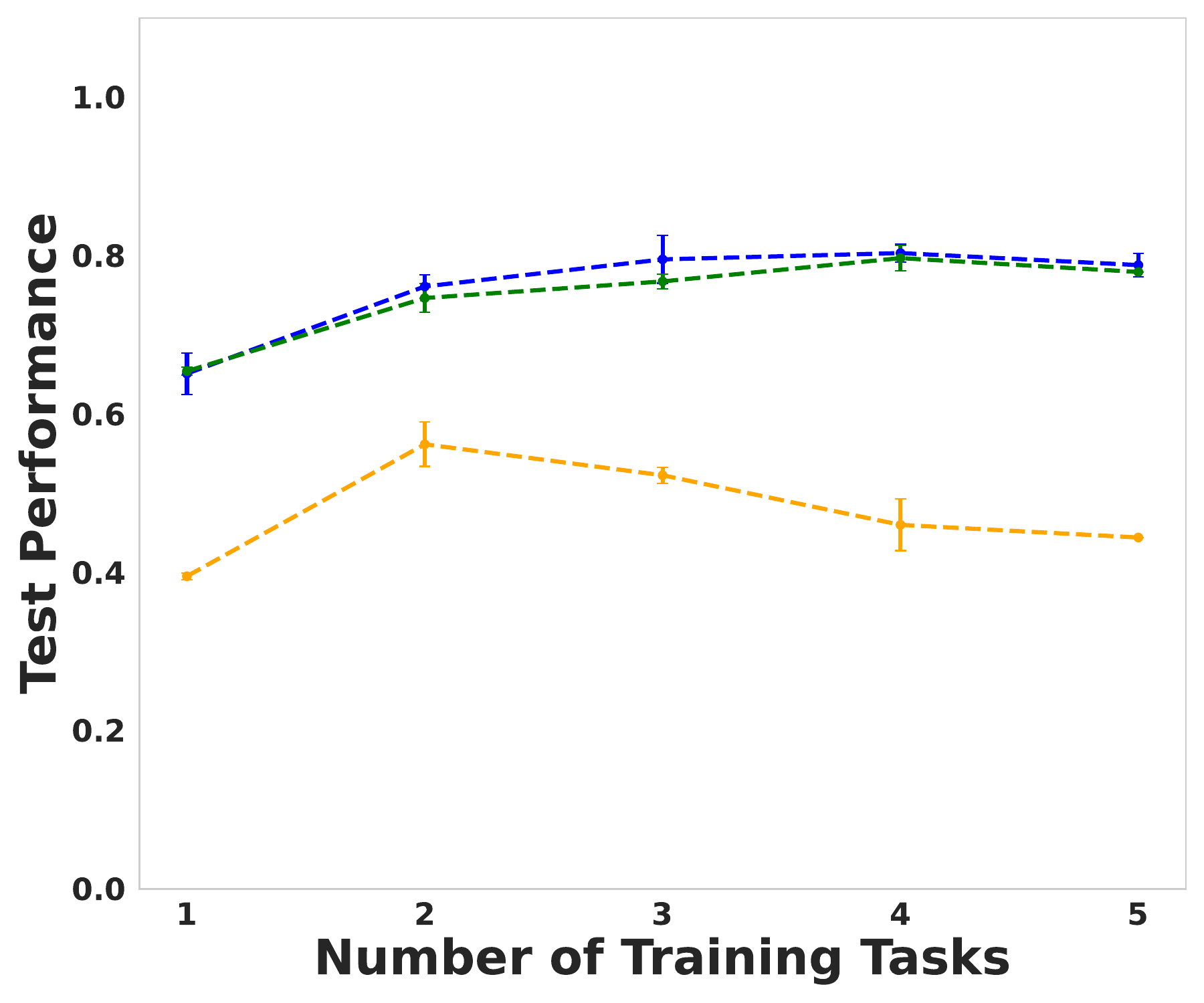}
    \caption{\texttt{LBF}}
  \end{subfigure}\hfill
  \begin{subfigure}[t]{0.24\linewidth}
    \centering
    \includegraphics[width=\linewidth]{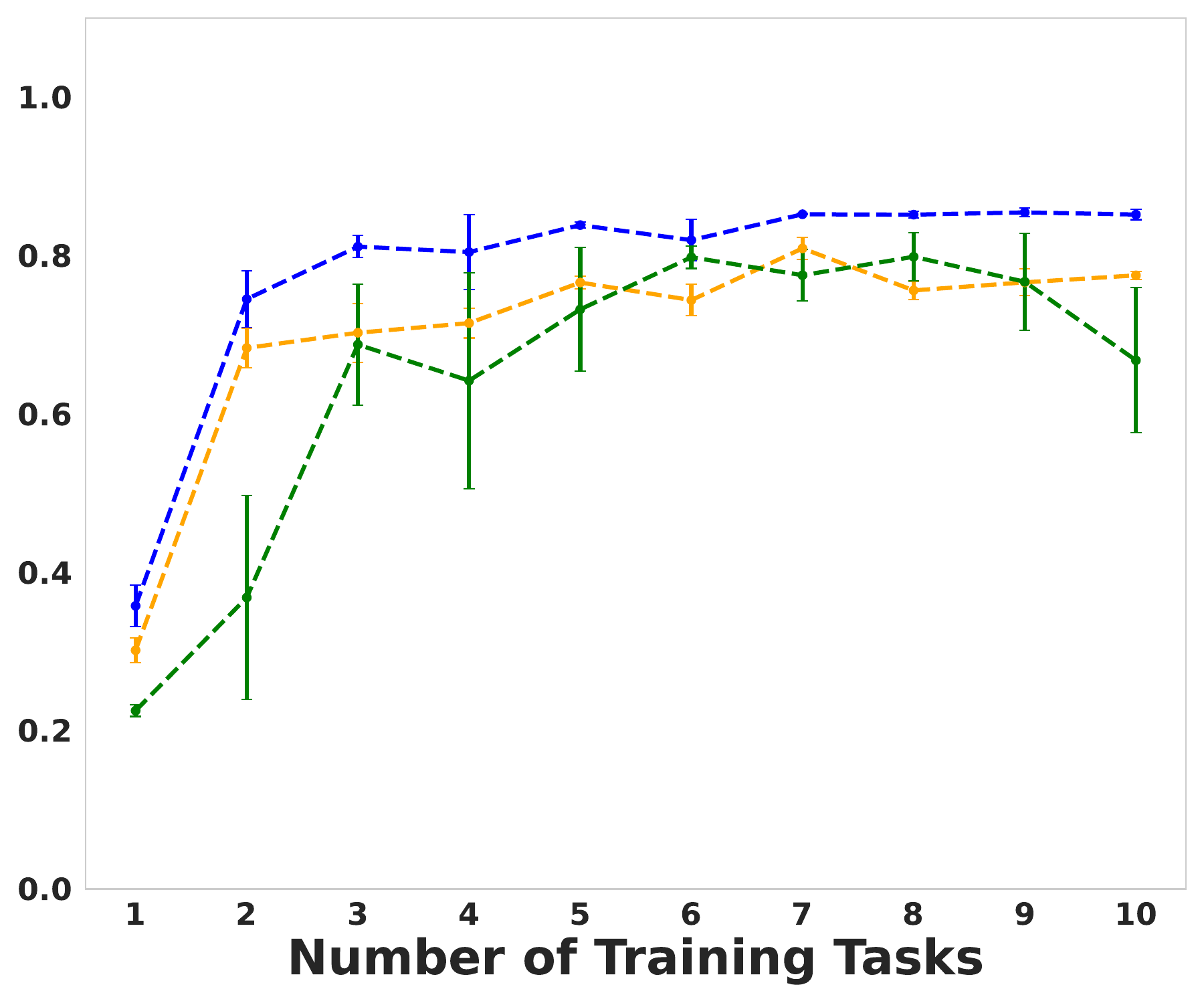}
    \caption{\texttt{Connector}}
  \end{subfigure}\hfill
  \begin{subfigure}[t]{0.24\linewidth}
    \centering
    \includegraphics[width=\linewidth]{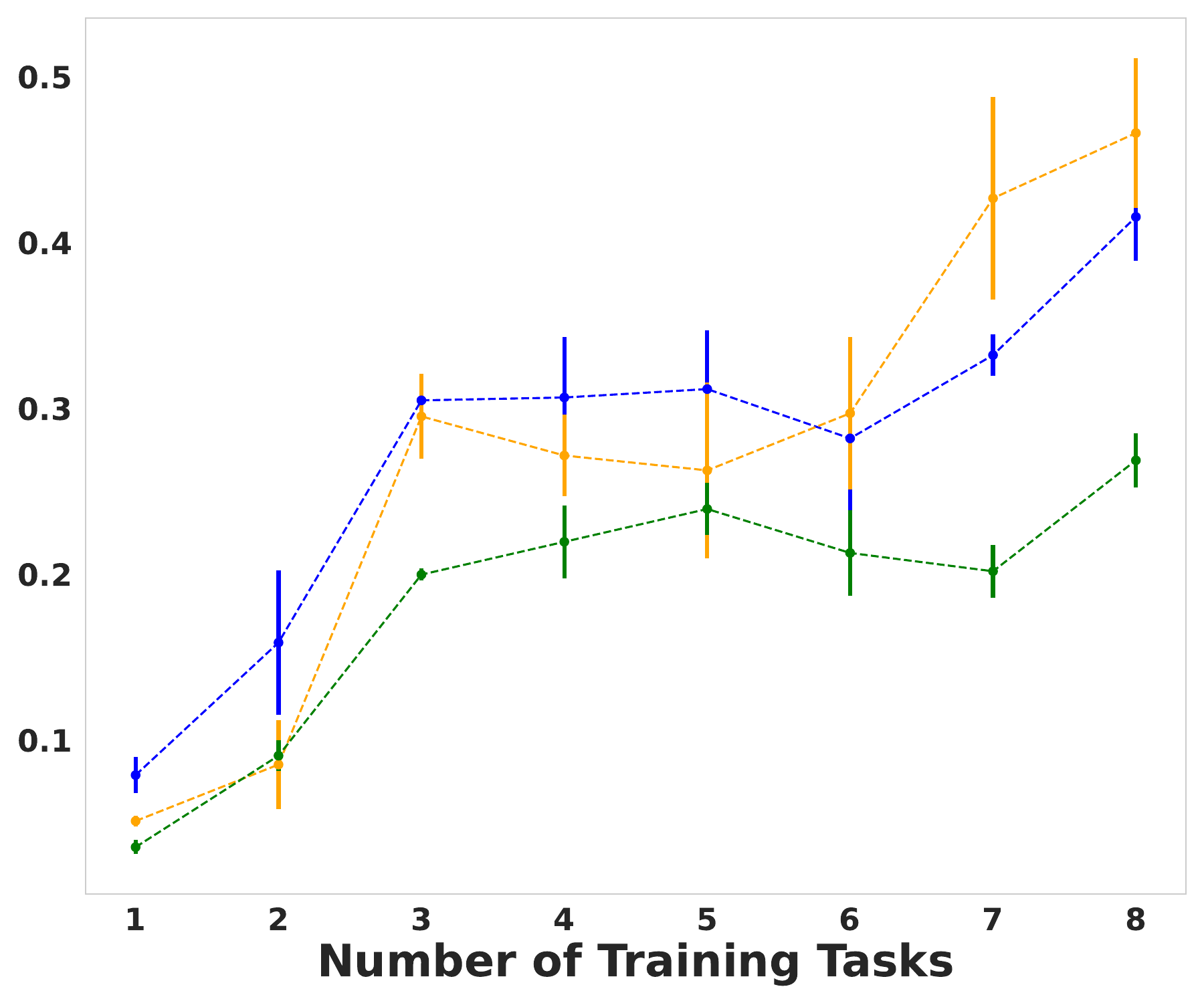}
    \caption{\texttt{RWARE}}
  \end{subfigure}\hfill
  \begin{subfigure}[t]{0.24\linewidth}
    \centering
    \includegraphics[width=\linewidth]{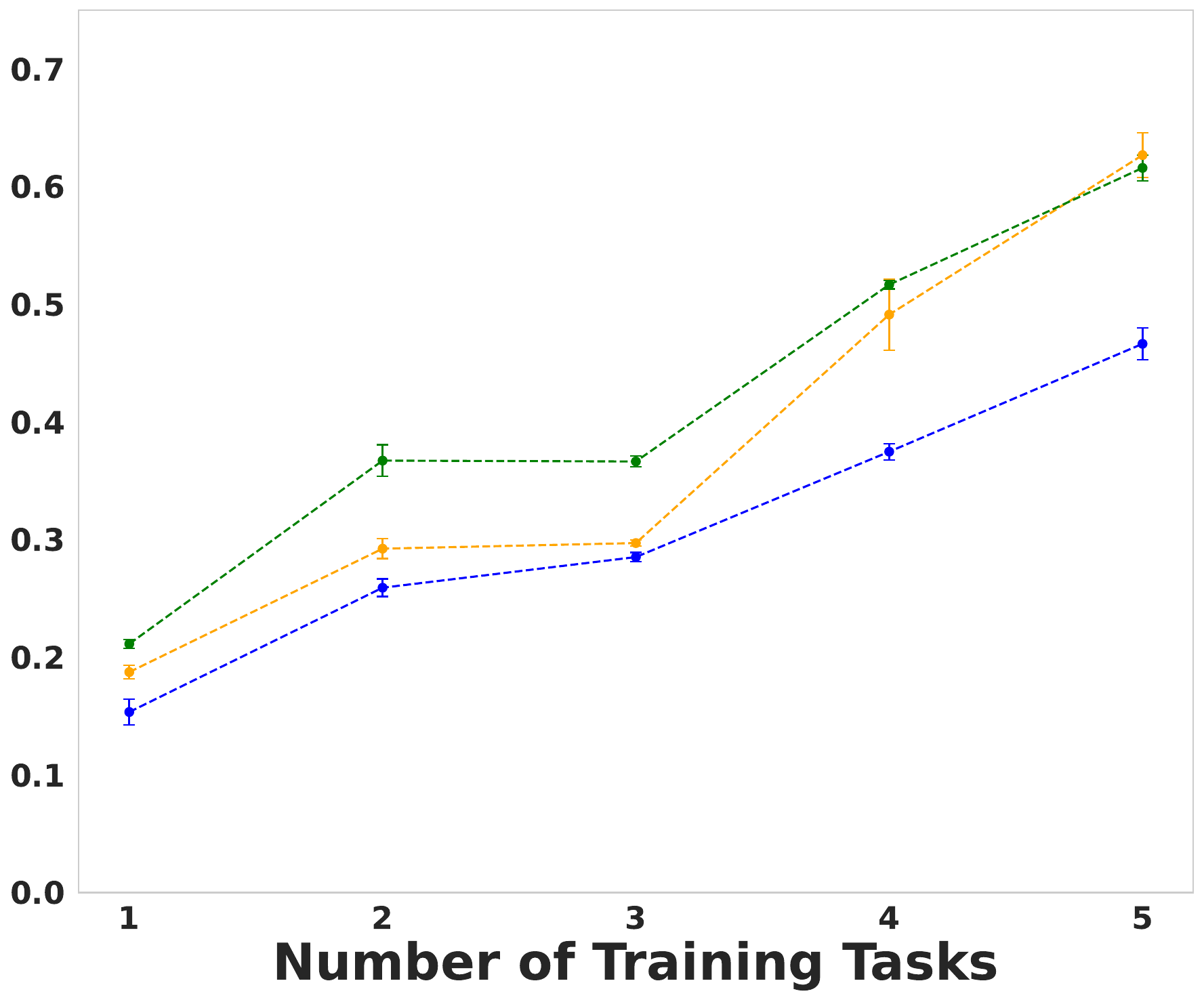}
    \caption{\texttt{SMAX}}
  \end{subfigure}

  \caption{
  \textit{The effect of increasing task diversity on performance.}
  \textbf{Top:} training tasks.
  \textbf{Bottom:} held-out test tasks.
  When training on a single task, performance is high on the training task but generalisation to held-out tasks is poor.
  As the number of training tasks increases, test-task performance steadily improves across all environments.
  }
  \label{fig:side-by-side}
\end{figure*}

\subsection{Multi-Task training improves generalisation}
\textbf{Experiment.} 
We vary the number of tasks in the training set while keeping the test set fixed. Multi-task sequence models are trained on different subsets of the available training data, and their performance is evaluated on both the corresponding training tasks and the held-out test tasks. 

For \texttt{LBF}, we consider a total of \num{5} training tasks; for \texttt{Connector}, \num{10}; for \texttt{RWARE}, \num{8}; and for \texttt{SMAX}, \num{5}. For each environment, we incrementally increase the number of training tasks from \num{1} to the maximum available. Performance as a function of training task count, evaluated on both training and test tasks, is shown in \autoref{fig:side-by-side}.
 
\textbf{Discussion.} We observe that performance on the training tasks remains relatively high across all environments, even as the number of tasks increases. This indicates that the model can successfully learn across multiple tasks simultaneously. Qualitatively, we further observe that the model successfully learns very different multi-agent strategies across distinct tasks, such as switching between exploration and congestion avoidance (see \autoref{sec:viz_mt_policy}). However, in \texttt{RWARE} we note a progressive decline in training performance as the number of training tasks grows. We attribute this to the higher complexity of \texttt{RWARE} tasks and the need to scale model capacity with task diversity to maintain performance. Interestingly, even as train task performance degrades, test task performance improves nearly monotonically as the number of training tasks increases, highlighting the importance of diverse multi-task data for generalisation. On \texttt{LBF}, we observe that MT CQL-Sable’s performance decreases. We hypothesise that this is due to the high proportion of expert trajectories in the \texttt{LBF} dataset, as the data collection policy quickly converges to the optimal behaviour. Prior work has shown that CQL is particularly sensitive to overly narrow or high-quality datasets, and benefits from mixed quality datasets~\citep{schweighofer2022dataset}. To further examine this, we include an ablation on trajectories' quality in \autoref{sec:dataset-quality-ablation}. Additionally, we find that the offline MARL models generalise better than the behaviour cloning model, especially in settings with mixed-quality multi-task data (see \autoref{ap:marl-vs-bc}).

Across all algorithms and environments, performance tends to plateau after a certain number of training tasks. We attribute this saturation to the limits of the current model capacity, pointing to the necessity of scaling up the model size to obtain maximum performance on highly diverse multi-task datasets (see \autoref{sec:datasize_scaling}). To summarise the overall effect of multi-task training with a fixed model size, we measure and report the maximum performance gain on test tasks in \autoref{fig:abstract}. Averaged across all three algorithms, test performance improves by 5.4x on \texttt{RWARE}, 1.3x on \texttt{LBF}, 2.9x on \texttt{Connector}, and 3.2x on \texttt{SMAX}. These results validate the effectiveness of multi-task training as a means of unlocking substantial performance gains on unseen test tasks.

\subsection{Can we further improve generalisation by increasing the size of the datasets and models?}\label{sec:datasize_scaling} 

A natural question that arises is, what is the optimal dataset size and model size for generalisation? Can we improve the generalisation capabilities by simply increasing the size of the dataset for a given set of training tasks? Similarly, can we improve generalisation by increasing the size of the model? To test this, we design two experiments.

\textbf{Experiment (a).} To determine whether increasing the size of the datasets (in terms of number of transitions rather than number of tasks) helps performance, we conducted a sweep over dataset sizes for several multi-task datasets on \texttt{RWARE}.  The results of the sweep are presented in \autoref{fig:dataset_size}. Similar to the results by \citet{medirattageneralization}, we find that there is little evidence that scaling up the number of transitions helps generalisation nearly as much as adding more tasks.

\textbf{Experiment (b).} To study the effect of model size, we train various models with different numbers of parameters, ranging from \num{116}k to \num{13}M, using the \texttt{RWARE} dataset. For simplicity, we mainly vary the embedding dimension of the model's encoder-decoder network from \num{64} (\num{116}k parameters) to \num{768} (\num{13}M parameters). We report the average episode return, normalised by the online performance, on both the training and test tasks in \autoref{fig:network_scaling}. We show in \autoref{sec:connnector_model_scaling} similar results for \texttt{LBF}, \texttt{Connector} and \texttt{SMAX}.

\begin{figure*}[t]
    \centering
    \begin{subfigure}[t]{0.47\linewidth}
        \includegraphics[width=\linewidth]{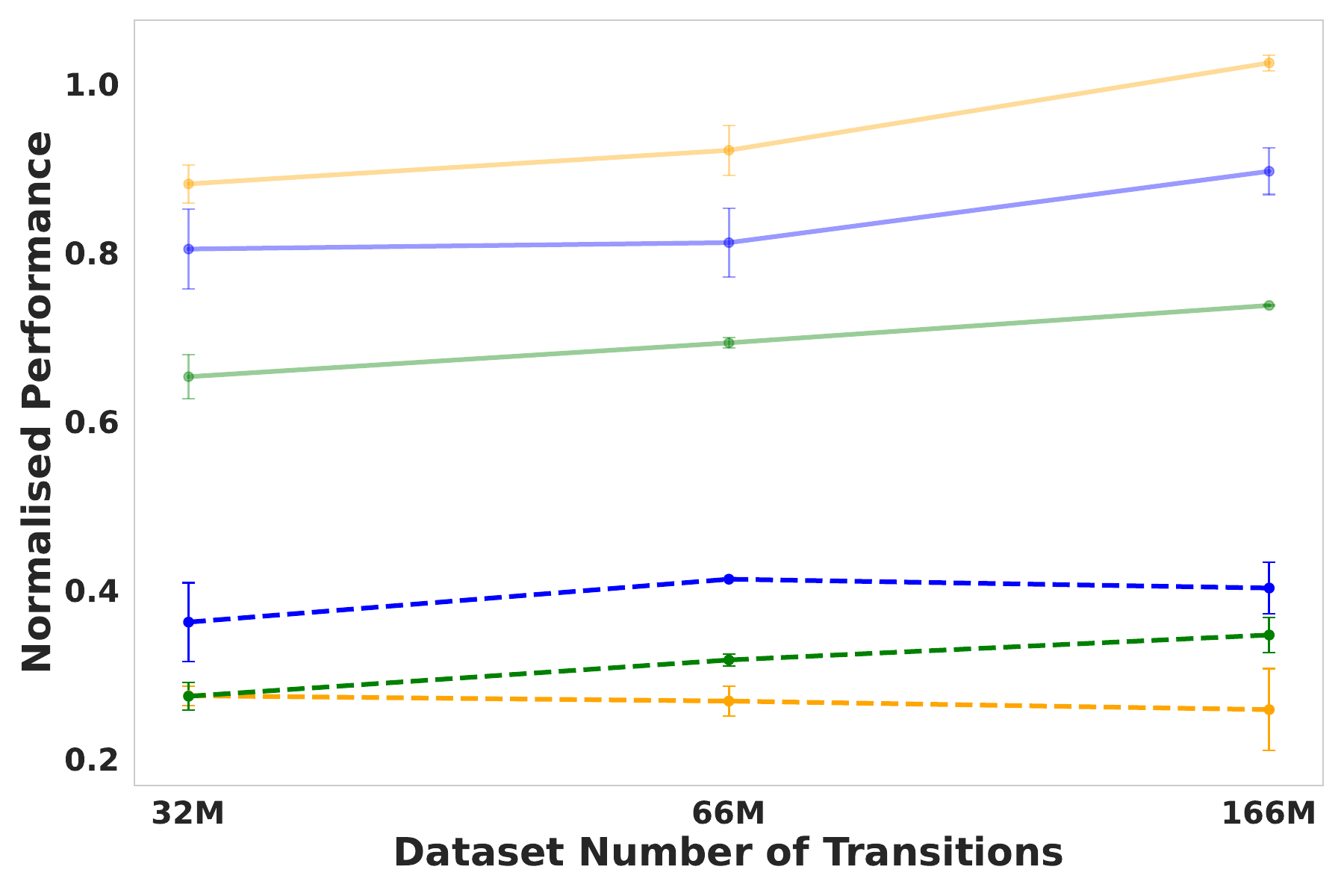}
        \caption{Dataset size scaling.}
        \label{fig:dataset_size}
    \end{subfigure}
    \hspace{5mm}
    \begin{subfigure}[t]{0.47\linewidth}
        \includegraphics[width=\linewidth]{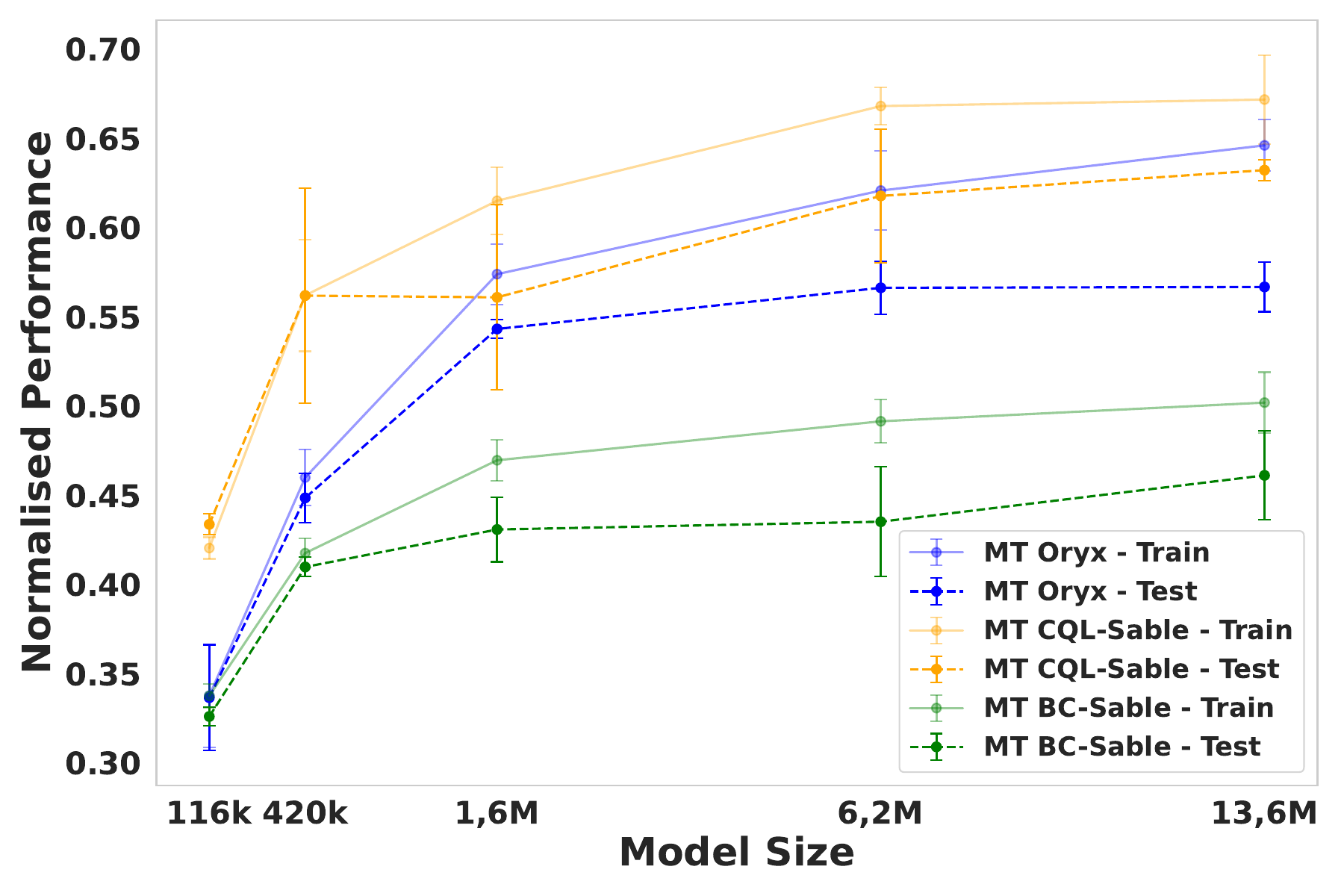}
        \caption{Model size scaling.}
        \label{fig:network_scaling}
    \end{subfigure}
    \caption{\textit{The impact of scaling up dataset \textbf{(left)} and model  size \textbf{(right).}} Fixing the number of \texttt{RWARE} tasks while increasing the number of transitions in the dataset leads to improved training performance, though test performance plateaus. Conversely, fixing the dataset size and varying the model size reveals a clear scaling trend, both training and test performance improve as the model grows.}
     \label{fig:scaling_results}
\end{figure*}


\textbf{Discussion.}
The results in \autoref{fig:dataset_size} indicate that simply increasing the number of transitions in the training dataset improves train task performance but does not lead to better generalisation on held-out test tasks, highlighting the importance of task diversity in multi-task datasets, since from \autoref{fig:side-by-side} we can conclude that adding additional tasks has a greater benefit. In contrast, scaling model capacity (\autoref{fig:network_scaling})---from an embedding dimension of \num{64} (\num{116}k parameters) to \num{512} (\num{6.2}M parameters)---consistently improved both training and test performance. This finding is particularly encouraging: it suggests that large, diverse multi-task datasets may be the missing ingredient needed to make ever-larger and more general offline MARL models viable. Notably, this result contrasts with the single-task setting reported by \citep{formanek2025oryx}, where the optimal embedding dimension was just \num{64}, underscoring the unique potential of multi-task data for enabling scale.

\section{Why Does Multi-Task Training Help MARL Sequence Models to Generalise?}
\label{sec:theory}

We provide a theoretical analysis of the empirical scaling behaviour presented in \autoref{sec:empirical-analysis}. Our main result (\autoref{thm:main}) bounds the generalisation gap of a multi-task offline MARL policy in terms of three quantities: a \textbf{training error}, a \textbf{coverage radius} on the task space, and a \textbf{Lipschitz constant} of the value function with respect to a task metric. The bound captures both qualitative trends observed in our experiments: more diverse training tasks shrink the coverage radius, while a larger, better-fit policy class shrinks the training error.

\paragraph{Setup.} 
Each task $z\in\mathcal Z$ indexes a Dec-POMDP $\mathcal M_z$ with shared discount $\gamma\in[0,1)$ and rewards bounded in $[0,R_{\max}]$. We assume tasks are embedded into a common state-action space $(\mathcal S, \mathcal A)$ via the zero-padding and masking procedure presented in \autoref{ssec:mt-sec-models}, so that any (history-dependent) policy defined on $(\mathcal S, \mathcal A)$ can be evaluated on every $\mathcal M_z$. Let $\Pi$ be the class of such history-dependent policies. For $\pi\in\Pi$, let $J_z(\pi) = \mathbb E_{\pi, P_z, \rho_z}\big[\sum_{t=0}^\infty \gamma^t R_z(s_t,a_t)\big]$, and let $\pi^*_z \in \arg\max_{\pi\in\Pi} J_z(\pi)$ denote a $\Pi$-optimal policy on task $z$ (assumed to exist). We train a single multi-task policy $\pi_\theta\in\Pi$ on a finite training set $\mathcal T_{\text{train}} = \{z_1,\dots,z_N\} \subset \mathcal Z$ using a fixed offline dataset $\mathcal D_{\text{train}} = \{\mathcal D_{z_i}\}_{i=1}^N$, and measure generalisation by the zero-shot regret on a held-out task:
\[
\mathcal R(z_{\text{test}}) \;:=\; J_{z_{\text{test}}}(\pi^*_{z_{\text{test}}}) - J_{z_{\text{test}}}(\pi_\theta).
\]

\paragraph{Task metric.}
We assume $\mathcal Z$ is equipped with a pseudometric $d$. A natural choice is a transition-reward distance,
\[
d(z, z') := \sup_{(s,a)}\Big\{|R_z(s,a) - R_{z'}(s,a)| + \tfrac{\gamma R_{\max}}{1-\gamma}\, D_{TV}(P_z(\cdot|s,a), P_{z'}(\cdot|s,a))\Big\},
\]
defined on the common embedded space; we also assume $\rho_z = \rho_{z'}$ across tasks (or absorb $D_{TV}(\rho_z, \rho_{z'})$ into $d$).

\subsection{Main generalisation bound}

\begin{assumption}[Common embedded space]
\label{ass:common}
All tasks $z\in\mathcal Z$ share a common state-action space $(\mathcal S, \mathcal A)$ and a common initial distribution $\rho$, so that every history-dependent policy $\pi\in\Pi$ can be evaluated on every $\mathcal M_z$.
\end{assumption}

\begin{assumption}[Smoothness on $\Pi$]
\label{ass:lip}
There exists $L \ge 0$ such that for every policy $\pi \in \Pi$ and all $z, z' \in \mathcal Z$, $|J_z(\pi) - J_{z'}(\pi)| \le L\, d(z,z')$.
\end{assumption}

\begin{assumption}[Bounded training error]
\label{ass:train}
There exists $\varepsilon_{\text{train}} \ge 0$ such that $J_{z_i}(\pi^*_{z_i}) - J_{z_i}(\pi_\theta) \le \varepsilon_{\text{train}}$ for every $z_i \in \mathcal T_{\text{train}}$.
\end{assumption}

Assumption~\ref{ass:common} is by construction in our setting. Assumption~\ref{ass:lip} is a substantive structural assumption on the task family: it can be derived from Lipschitz continuity of the reward and transition kernels in the metric $d$ via a simulation lemma (Lemma~\ref{lem:sim} in Appendix~\ref{app:theory}), but it is not automatic for arbitrary task families and may fail for
combinatorial task families with abrupt structural changes between
tasks. 

\begin{theorem}[Coverage--smoothness regret bound]
\label{thm:main}
Under Assumptions~\ref{ass:common}--\ref{ass:train}, for any $z_{\text{test}} \in \mathcal Z$,
\[
\;\mathcal R(z_{\text{test}}) \;\le\; \varepsilon_{\text{train}} \;+\; 2L \cdot \min_{z_i \in \mathcal T_{\text{train}}} d(z_i, z_{\text{test}}).\;
\]
\end{theorem}

The bound formally captures the two empirical findings of \autoref{sec:empirical-analysis}:

\begin{itemize}
    \item \emph{Task diversity reduces the coverage radius.} The bound decreases as $\min_{z_i} d(z_i, z_{\text{test}})$ shrinks, which is achieved by adding training tasks that are close in $d$ to plausible test tasks. This is consistent with the monotonic test-performance gains observed in \autoref{fig:side-by-side}.
    \item \emph{Model capacity reduces $\varepsilon_{\mathrm{train}}$.} Within the bound, scaling model size acts exclusively through $\varepsilon_{\mathrm{train}}$, leaving the coverage term fixed for a given training task set. This is consistent with the monotonic improvement in test performance observed in \autoref{fig:network_scaling}.
\end{itemize}

\subsection{Finite-task coverage corollary}

In our experimental setting $\mathcal Z$ is the finite set of tasks in a benchmark suite. We record the simple deterministic consequence.

\begin{corollary}[Finite-set coverage]
\label{cor:finite}
Under Assumptions~\ref{ass:common}--\ref{ass:train}, if $z_{\text{test}} \in \mathcal T_{\text{train}}$ then $\mathcal R(z_{\text{test}}) \le \varepsilon_{\text{train}}$. More generally, if every test task lies within $d$-distance $r$ of some training task, then $\sup_{z_{\text{test}}\in\mathcal Z_{\text{test}}} \mathcal R(z_{\text{test}}) \le \varepsilon_{\text{train}} + 2Lr$.
\end{corollary}

Corollary~\ref{cor:finite} makes precise the qualitative claim that, for finite task suites, generalisation is controlled by how thoroughly the training set covers the test set in the chosen metric.

\subsection{Proxy coverage: an empirical surrogate}

\begin{figure}[t]
    \centering
    \includegraphics[width=0.6\linewidth]{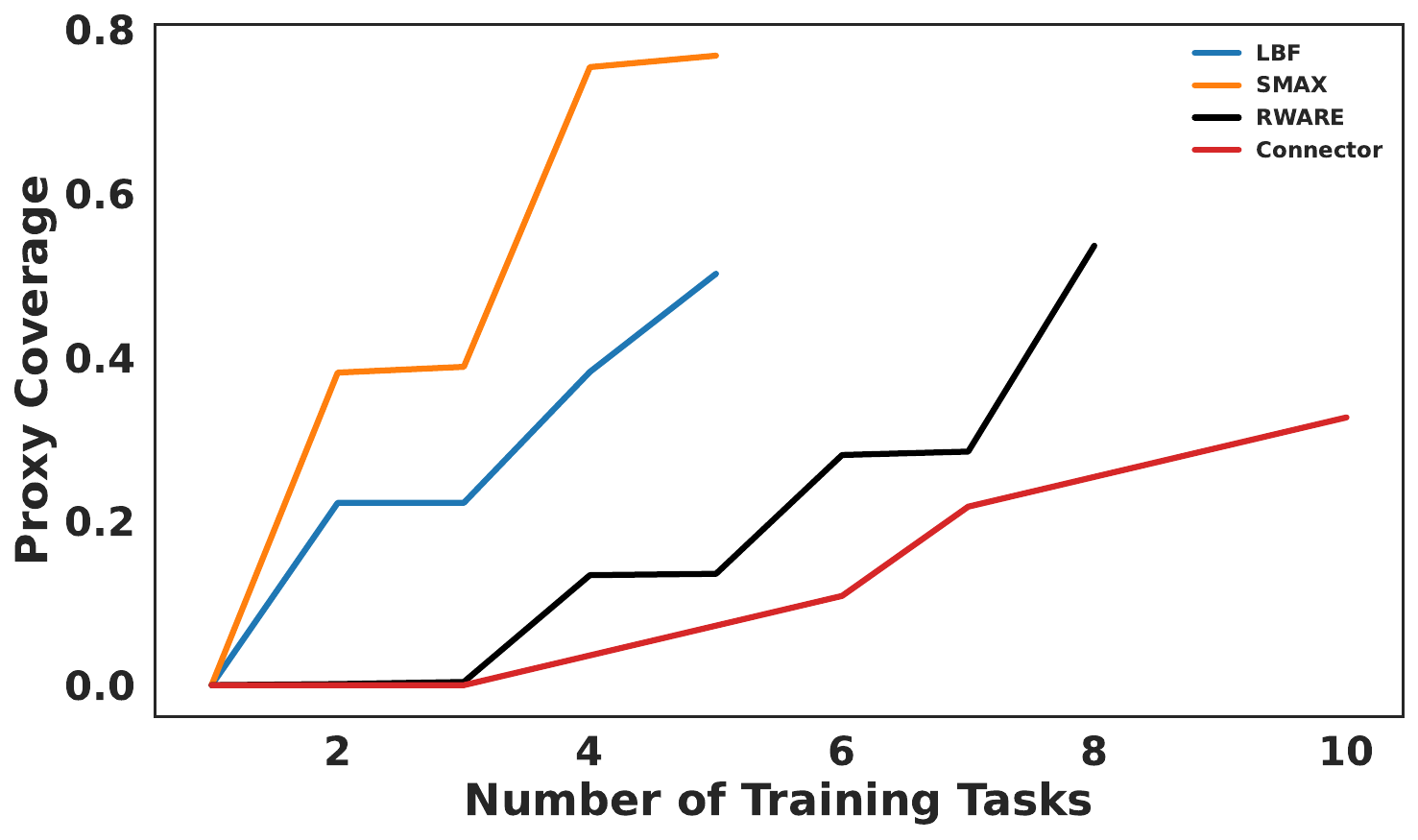} 
    \caption{\textit{Proxy coverage metric for each environment's test set.} This proxy metric, based on the similarity between test and training tasks, provides insight into how performance is likely to improve as additional training tasks are included.}
    \label{fig:coverage_error}
\end{figure}

Since the theoretical coverage term $\min_{z_i} d(z_i, z_{\text{test}})$  is not directly computable, we construct an empirical surrogate. For each environment, we represent each task $z$ by a hand-crafted descriptor $\hat{z}$ capturing key task properties (detailed in \ref{app:proxy}), and define the \textbf{proxy coverage} as

\begin{equation*}
\mathrm{PC} \;=\; 1 - \min_{\hat{z}_i \in \mathcal{T}_{\mathrm{train}}} \lVert \hat{z}_i - \hat{z}_{\mathrm{test}} \rVert,
\end{equation*}

so that higher values indicate better coverage of the test task by the training set. \autoref{fig:coverage_error} plots PC as the number of training tasks grows. Consistent with the bound in \autoref{thm:main}, PC increases monotonically in all environments, though at rates that reflect each environment's task diversity: \texttt{SMAX} and \texttt{LBF} show the steepest gains, while \texttt{Connector}'s training and test tasks are already similar with few tasks, leading to early saturation.

\section{Related Work}
\label{sec:related_work}

\textbf{Offline MARL.} Most prior work in offline MARL uses single-task training and evaluation, while focusing on finding solutions to key challenges particular to offline multi-agent learning. Seminal early papers include \citet{Jiang2021OfflineDM} and \citet{Yang2021BelieveWY}, who introduced multi-agent methods for constrained Q-value estimation. Since then, numerous additional works have aimed to tackle challenges such as extrapolation error~\citep{shao2023counterfactual,Eldeeb2024ConservativeAR}, coordination~\citep{barde2024modelbasedsolutionofflinemultiagent, tilbury2024coordination,Zhou_2025}, offline training stability~\citep{Pan2021PlanBA, Wang2023OfflineMR, Matsunaga2023AlberDICEAO, wu2023conservative, Bui2024ComaDICEOC, Liu2024OfflineMR, li2025dof}, opponent modeling~\citep{jing2024towards}, offline-to-online transfer~\citep{zhong2024offlinetoonlinemultiagentreinforcementlearning,formanek2023reduce} and theoretical understanding~\citep{cui2022provablyefficientofflinemultiagent, cui2022offlinetwoplayerzerosummarkov, zhong2022pessimisticminimaxvalueiteration, zhang2023offlinelearningmarkovgames, xiong2023nearlyminimaxoptimaloffline, wu2023conservative}.

\textbf{Sequence Models for RL.}
Formulating RL as a sequence modelling problem has gained significant attention. \citet{chen2021decision} introduced the Decision Transformer (DT), later extended in various ways \citep{zheng2022online, yamagata2023q, wu2023elastic}. \citet{lee2022multi} trained a multi-task DT that learned across tasks and could be quickly fine-tuned. \citet{meng2023offline} introduced MADT, an extension of the DT to the multi-agent setting. The Multi-Agent Transformer (MAT)~\citep{wen2022multi} addressed the online setting with auto-regressive action selection, and \citet{mahjoub2025sable} improved on MAT with Sable, which replaces the Transformer with a Retentive Network~\citep{sun2023retentive} and adds temporal memory, achieving state-of-the-art results. Building on this line, \citet{formanek2025oryx} proposed Oryx, an offline MARL sequence model derived from an autoregressive version of Implicit Constraint Q-Learning (ICQ)~\citep{yang2021believe} and offline-specific modifications to Sable, also achieving state-of-the-art performance.

\textbf{Multi-Task RL.} 
Multi-task training has most prominently been investigated in single-agent continuous-control and robotics problems with a focus on representation and transfer learning~\citep{xu2020knowledge, kalashnikov2021scaling, kumar2022pre, cheng2022provable}. 
Although shown to be useful in most cases, \citet{yu2021conservative} find that naively adding more multi-task data to an offline RL training dataset can sometimes lead to a decrease in performance on downstream tasks, particularly when the distributional shift between tasks is large. In terms of generalisation,  
\citet{kumar2022offline} and \citet{he2023diffusion} highlight the potential for high-capacity models trained on large and diverse multi-task datasets to produce agents that can generalise more broadly when fine-tuned on previously unseen tasks. Most closely related to our work is that of \citet{medirattageneralization}, who evaluate the zero-shot generalisation capabilities of several offline single-agent RL methods by training them on a set of training tasks and testing them on a set of holdout tasks. They find that current offline RL methods do not generalise well and are
typically outperformed by simple behaviour cloning.

\textbf{Multi-Task MARL.}
Multi-task MARL faces both architectural and evaluation challenges when agents must generalise beyond single-task training, motivating formal definitions and benchmarks for task generalisation\citep{lukas2022rwaremt}. 
\citet{rosen2024optimal} give a formal, goal-oriented theory that proves how a learned world value function can enable provably optimal zero-shot task generalisation in goal-based multi-agent settings.
MaskMA~\citep{liumaskma2024} introduces a mask-based framework that adapts to varying agent- and action-spaces and shows strong zero-shot transfer on unseen SMAC~\citep{samvelyan19smac} maps. Unlike our approach, their work builds on MADT~\citep{meng2023offline}, while we focus on sequence model architectures related to Oryx~\citep{formanek2025oryx}, which have been shown to outperform MADT.
The offline coordination-skill discovery method ODIS~\citep{zhang2023discovering} extracts task-invariant coordination primitives from multi-task trajectories and shows that this can be used to deploy coordination policies to unseen SMAC tasks without additional online interaction. 
Related work, HiSSD~\citep{liu2025learning} proposes a hierarchical separation between common cooperative (temporal) skills and task-specific controllers.
None of the above studies investigates the effect of task diversity on test performance, instead keeping the number of training tasks fixed.

\section{Conclusion}
In this work, we studied the generalisation of offline MARL sequence models. Both our empirical results from large sets of experiments and insights from theoretical analysis showed that task diversity is a key driver for reducing the train--test generalisation gap. Additionally, we showed that significant improvements in zero-shot out-of-distribution performance can be achieved via model scaling, and that this is consistent across various sequence model objective functions. Our findings suggest that future progress in offline MARL should prioritise (i) constructing large and diverse multi-task datasets, and (ii) carefully tuning model capacity for a given data budget to maximise zero-shot generalisation. We release code, datasets, task splits, and training scripts to encourage reproducibility and to establish stronger benchmarks for evaluating generalisation in offline MARL.

\textbf{Limitations and future work.} Our work is limited to centralised sequence model architectures, and although these represent a powerful and performant model class, promising future work could include extending our analysis to decentralised and CTDE algorithms. Additional areas of interest include studying the limits of transfer across environments (not only tasks), and investigating accelerating fine-tuning in safety-critical and data-scarce real-world domains.


\subsubsection*{Acknowledgments}

\bibliography{example_paper}

@article{formanek2025oryx,
  title={Oryx: a Performant and Scalable Algorithm for Many-Agent Coordination in Offline MARL},
  author={Formanek, Claude and Mahjoub, Omayma and Nessir, Louay Ben and Abramowitz, Sasha and de Kock, Ruan and Khlifi, Wiem and Toit, Simon Du and Chalumeau, Felix and Rajaonarivonivelomanantsoa, Daniel and Fokam, Arnol and others},
  journal={Advances in neural information processing systems},
  year={2025}
}

@article{pomerleau1988alvinn,
  title={Alvinn: An autonomous land vehicle in a neural network},
  author={Pomerleau, Dean A},
  journal={Advances in neural information processing systems},
  volume={1},
  year={1988}
}

@article{yang2021believe,
  title={Believe what you see: Implicit constraint approach for offline multi-agent reinforcement learning},
  author={Yang, Yiqin and Ma, Xiaoteng and Li, Chenghao and Zheng, Zewu and Zhang, Qiyuan and Huang, Gao and Yang, Jun and Zhao, Qianchuan},
  journal={Advances in Neural Information Processing Systems},
  volume={34},
  pages={10299--10312},
  year={2021}
}

@article{shao2023counterfactual,
  title={Counterfactual conservative q learning for offline multi-agent reinforcement learning},
  author={Shao, Jianzhun and Qu, Yun and Chen, Chen and Zhang, Hongchang and Ji, Xiangyang},
  journal={Advances in Neural Information Processing Systems},
  volume={36},
  pages={77290--77312},
  year={2023}
}

@inproceedings{bain1995framework,
  title={A Framework for Behavioural Cloning.},
  author={Bain, Michael and Sammut, Claude},
  booktitle={Machine intelligence 15},
  pages={103--129},
  year={1995}
}

@article{wen2022multi,
  title={Multi-agent reinforcement learning is a sequence modeling problem},
  author={Wen, Muning and Kuba, Jakub and Lin, Runji and Zhang, Weinan and Wen, Ying and Wang, Jun and Yang, Yaodong},
  journal={Advances in Neural Information Processing Systems},
  volume={35},
  pages={16509--16521},
  year={2022}
}

@inproceedings{li2025dof,
  title={Dof: A diffusion factorization framework for offline multi-agent reinforcement learning},
  author={Li, Chao and Deng, Ziwei and Lin, Chenxing and Chen, Wenqi and Fu, Yongquan and Liu, Weiquan and Wen, Chenglu and Wang, Cheng and Shen, Siqi},
  booktitle={The Thirteenth International Conference on Learning Representations},
  year={2025}
}

@article{meng2023offline,
  title={Offline pre-trained multi-agent decision transformer},
  author={Meng, Linghui and Wen, Muning and Le, Chenyang and Li, Xiyun and Xing, Dengpeng and Zhang, Weinan and Wen, Ying and Zhang, Haifeng and Wang, Jun and Yang, Yaodong and others},
  journal={Machine Intelligence Research},
  volume={20},
  number={2},
  pages={233--248},
  year={2023},
  publisher={Springer}
}

@inproceedings{medirattageneralization,
  title={The Generalization Gap in Offline Reinforcement Learning},
  author={Mediratta, Ishita and You, Qingfei and Jiang, Minqi and Raileanu, Roberta},
  booktitle={The Twelfth International Conference on Learning Representations},
  year={2024}
}

@inproceedings{bonnetjumanji,
  title={Jumanji: a Diverse Suite of Scalable Reinforcement Learning Environments in JAX},
  author={Bonnet, Cl{\'e}ment and Luo, Daniel and Byrne, Donal John and Surana, Shikha and Abramowitz, Sasha and Duckworth, Paul and Coyette, Vincent and Midgley, Laurence Illing and Tegegn, Elshadai and Kalloniatis, Tristan and others},
  booktitle={The Twelfth International Conference on Learning Representations},
year={2024}
}

@inproceedings{cui2019class,
  title={Class-balanced loss based on effective number of samples},
  author={Cui, Yin and Jia, Menglin and Lin, Tsung-Yi and Song, Yang and Belongie, Serge},
  booktitle={Proceedings of the IEEE/CVF conference on computer vision and pattern recognition},
  pages={9268--9277},
  year={2019}
}

@inproceedings{imani2018improving,
  title={Improving regression performance with distributional losses},
  author={Imani, Ehsan and White, Martha},
  booktitle={International conference on machine learning},
  pages={2157--2166},
  year={2018},
  organization={PMLR}
}

@inproceedings{farebrother2024stop,
  title={Stop regressing: training value functions via classification for scalable deep RL},
  author={Farebrother, Jesse and Orbay, Jordi and Vuong, Quan and Ta{\"\i}ga, Adrien Ali and Chebotar, Yevgen and Xiao, Ted and Irpan, Alex and Levine, Sergey and Castro, Pablo Samuel and Faust, Aleksandra and others},
  booktitle={Proceedings of the 41st International Conference on Machine Learning},
  pages={13049--13071},
  year={2024}
}

@inproceedings{mahjoub2025sable,
  title={Sable: a Performant, Efficient and Scalable Sequence Model for MARL},
  author={Mahjoub, Omayma and Abramowitz, Sasha and de Kock, Ruan John and Khlifi, Wiem and Du Toit, Simon Verster and Daniel, Jemma and Nessir, Louay Ben and Beyers, Louise and Formanek, Juan Claude and Clark, Liam and others},
  booktitle={Forty-second International Conference on Machine Learning},
  year={2025}
}

@article{kumar2020conservative,
  title={Conservative q-learning for offline reinforcement learning},
  author={Kumar, Aviral and Zhou, Aurick and Tucker, George and Levine, Sergey},
  journal={Advances in neural information processing systems},
  volume={33},
  pages={1179--1191},
  year={2020}
}

@inproceedings{kumar2022offline,
  title={Offline Q-learning on Diverse Multi-Task Data Both Scales And Generalizes},
  author={Kumar, Aviral and Agarwal, Rishabh and Geng, Xinyang and Tucker, George and Levine, Sergey},
  booktitle={Deep Reinforcement Learning Workshop NeurIPS 2022},
    year={2022}
}

@article{yu2021conservative,
  title={Conservative data sharing for multi-task offline reinforcement learning},
  author={Yu, Tianhe and Kumar, Aviral and Chebotar, Yevgen and Hausman, Karol and Levine, Sergey and Finn, Chelsea},
  journal={Advances in Neural Information Processing Systems},
  volume={34},
  pages={11501--11516},
  year={2021}
}

@article{cheng2022provable,
  title={Provable benefit of multitask representation learning in reinforcement learning},
  author={Cheng, Yuan and Feng, Songtao and Yang, Jing and Zhang, Hong and Liang, Yingbin},
  journal={Advances in Neural Information Processing Systems},
  volume={35},
  pages={31741--31754},
  year={2022}
}

@article{xu2020knowledge,
  title={Knowledge transfer in multi-task deep reinforcement learning for continuous control},
  author={Xu, Zhiyuan and Wu, Kun and Che, Zhengping and Tang, Jian and Ye, Jieping},
  journal={Advances in Neural Information Processing Systems},
  volume={33},
  pages={15146--15155},
  year={2020}
}

@article{kumar2022pre,
  title={Pre-training for robots: Offline rl enables learning new tasks from a handful of trials},
  author={Kumar, Aviral and Singh, Anikait and Ebert, Frederik and Nakamoto, Mitsuhiko and Yang, Yanlai and Finn, Chelsea and Levine, Sergey},
  journal={arXiv preprint arXiv:2210.05178},
  year={2022}
}

@inproceedings{kalashnikov2021scaling,
  title={Scaling up multi-task robotic reinforcement learning},
  author={Kalashnikov, Dmitry and Varley, Jake and Chebotar, Yevgen and Swanson, Benjamin and Jonschkowski, Rico and Finn, Chelsea and Levine, Sergey and Hausman, Karol},
  booktitle={5th Annual Conference on Robot Learning},
  year={2021}
}

@article{he2023diffusion,
  title={Diffusion model is an effective planner and data synthesizer for multi-task reinforcement learning},
  author={He, Haoran and Bai, Chenjia and Xu, Kang and Yang, Zhuoran and Zhang, Weinan and Wang, Dong and Zhao, Bin and Li, Xuelong},
  journal={Advances in neural information processing systems},
  volume={36},
  pages={64896--64917},
  year={2023}
}

@article{chen2021decision,
  title={Decision transformer: Reinforcement learning via sequence modeling},
  author={Chen, Lili and Lu, Kevin and Rajeswaran, Aravind and Lee, Kimin and Grover, Aditya and Laskin, Misha and Abbeel, Pieter and Srinivas, Aravind and Mordatch, Igor},
  journal={Advances in neural information processing systems},
  volume={34},
  pages={15084--15097},
  year={2021}
}

@inproceedings{zheng2022online,
  title={Online decision transformer},
  author={Zheng, Qinqing and Zhang, Amy and Grover, Aditya},
  booktitle={international conference on machine learning},
  pages={27042--27059},
  year={2022},
  organization={PMLR}
}

@article{wu2023elastic,
  title={Elastic decision transformer},
  author={Wu, Yueh-Hua and Wang, Xiaolong and Hamaya, Masashi},
  journal={Advances in neural information processing systems},
  volume={36},
  pages={18532--18550},
  year={2023}
}

@inproceedings{yamagata2023q,
  title={Q-learning decision transformer: Leveraging dynamic programming for conditional sequence modelling in offline rl},
  author={Yamagata, Taku and Khalil, Ahmed and Santos-Rodriguez, Raul},
  booktitle={International Conference on Machine Learning},
  pages={38989--39007},
  year={2023},
  organization={PMLR}
}

@article{lee2022multi,
  title={Multi-game decision transformers},
  author={Lee, Kuang-Huei and Nachum, Ofir and Yang, Mengjiao Sherry and Lee, Lisa and Freeman, Daniel and Guadarrama, Sergio and Fischer, Ian and Xu, Winnie and Jang, Eric and Michalewski, Henryk and others},
  journal={Advances in neural information processing systems},
  volume={35},
  pages={27921--27936},
  year={2022}
}

@article{sun2023retentive,
  title={Retentive network: A successor to transformer for large language models},
  author={Sun, Yutao and Dong, Li and Huang, Shaohan and Ma, Shuming and Xia, Yuqing and Xue, Jilong and Wang, Jianyong and Wei, Furu},
  journal={arXiv preprint arXiv:2307.08621},
  year={2023}
}

@inproceedings{de2025exponentially,
  title={Is an Exponentially Growing Action Space Really that Bad? Validating a Core Assumption for using Multi-Agent RL},
  author={de Kock, Ruan and Pretorius, Arnu and Shock, Jonathan},
  booktitle={Proceedings of the 24th International Conference on Autonomous Agents and Multiagent Systems},
  pages={2490--2492},
  year={2025}
}

@article{daniel2024multi,
  title={Multi-Agent Reinforcement Learning with Selective State-Space Models},
  author={Daniel, Jemma and de Kock, Ruan and Nessir, Louay Ben and Abramowitz, Sasha and Mahjoub, Omayma and Khlifi, Wiem and Formanek, Claude and Pretorius, Arnu},
  journal={arXiv preprint arXiv:2410.19382},
  year={2024}
}

@article{kirk2023survey,
  title={A survey of zero-shot generalisation in deep reinforcement learning},
  author={Kirk, Robert and Zhang, Amy and Grefenstette, Edward and Rockt{\"a}schel, Tim},
  journal={Journal of Artificial Intelligence Research},
  volume={76},
  pages={201--264},
  year={2023}
}

@article{levine2020offline,
  title={Offline reinforcement learning: Tutorial, review, and perspectives on open problems},
  author={Levine, Sergey and Kumar, Aviral and Tucker, George and Fu, Justin},
  journal={arXiv preprint arXiv:2005.01643},
  year={2020}
}

@inproceedings{chebotar2023q,
  title={Q-transformer: Scalable offline reinforcement learning via autoregressive q-functions},
  author={Chebotar, Yevgen and Vuong, Quan and Hausman, Karol and Xia, Fei and Lu, Yao and Irpan, Alex and Kumar, Aviral and Yu, Tianhe and Herzog, Alexander and Pertsch, Karl and others},
  booktitle={Conference on Robot Learning},
  pages={3909--3928},
  year={2023},
  organization={PMLR}
}

@article{gorsane2022towards,
  title={Towards a standardised performance evaluation protocol for cooperative marl},
  author={Gorsane, Rihab and Mahjoub, Omayma and de Kock, Ruan John and Dubb, Roland and Singh, Siddarth and Pretorius, Arnu},
  journal={Advances in Neural Information Processing Systems},
  volume={35},
  pages={5510--5521},
  year={2022}
}

@article{zhong2024heterogeneous,
  title={Heterogeneous-agent reinforcement learning},
  author={Zhong, Yifan and Kuba, Jakub Grudzien and Feng, Xidong and Hu, Siyi and Ji, Jiaming and Yang, Yaodong},
  journal={Journal of Machine Learning Research},
  volume={25},
  number={32},
  pages={1--67},
  year={2024}
}

@article{kaelbling1998planning,
  title={Planning and acting in partially observable stochastic domains},
  author={Kaelbling, Leslie Pack and Littman, Michael L and Cassandra, Anthony R},
  journal={Artificial intelligence},
  year={1998},
  publisher={Elsevier}
}

@inproceedings{
papoudakis2021benchmarking,
title={Benchmarking Multi-Agent Deep Reinforcement Learning Algorithms in Cooperative Tasks},
author={Georgios Papoudakis and Filippos Christianos and Lukas Sch{\"a}fer and Stefano V Albrecht},
booktitle={Thirty-fifth Conference on Neural Information Processing Systems Datasets and Benchmarks Track},
year={2021},
}

@misc{Yang2021BelieveWY,
      title={Believe What You See: Implicit Constraint Approach for Offline Multi-Agent Reinforcement Learning}, 
      author={Yiqin Yang and Xiaoteng Ma and Chenghao Li and Zewu Zheng and Qiyuan Zhang and Gao Huang and Jun Yang and Qianchuan Zhao},
      year={2021},
      eprint={2106.03400},
      archivePrefix={arXiv},
      primaryClass={cs.AI},
      url={https://arxiv.org/abs/2106.03400}, 
}

@misc{Pan2021PlanBA,
      title={Plan Better Amid Conservatism: Offline Multi-Agent Reinforcement Learning with Actor Rectification}, 
      author={Ling Pan and Longbo Huang and Tengyu Ma and Huazhe Xu},
      year={2022},
      eprint={2111.11188},
      archivePrefix={arXiv},
      primaryClass={cs.LG},
      url={https://arxiv.org/abs/2111.11188}, 
}

@inproceedings{Wang2023OfflineMR,
 author = {Wang, Xiangsen and Xu, Haoran and Zheng, Yinan and Zhan, Xianyuan},
 booktitle = {Advances in Neural Information Processing Systems},
 editor = {A. Oh and T. Naumann and A. Globerson and K. Saenko and M. Hardt and S. Levine},
 pages = {52413--52429},
 publisher = {Curran Associates, Inc.},
 title = {Offline Multi-Agent Reinforcement Learning with Implicit Global-to-Local Value Regularization},
 volume = {36},
 year = {2023}
}

@misc{Matsunaga2023AlberDICEAO,
      title={AlberDICE: Addressing Out-Of-Distribution Joint Actions in Offline Multi-Agent RL via Alternating Stationary Distribution Correction Estimation}, 
      author={Daiki E. Matsunaga and Jongmin Lee and Jaeseok Yoon and Stefanos Leonardos and Pieter Abbeel and Kee-Eung Kim},
      year={2023},
      eprint={2311.02194},
      archivePrefix={arXiv},
      primaryClass={cs.LG},
      url={https://arxiv.org/abs/2311.02194}, 
}

@inproceedings{Bui2024ComaDICEOC,
      title={ComaDICE: Offline Cooperative Multi-Agent Reinforcement Learning with Stationary Distribution Shift Regularization}, 
      author={The Viet Bui and Thanh Hong Nguyen and Tien Mai},
      year={2025},
      booktitle={International Conference on Learning Representations},
}

@inproceedings{
formanek2023reduce,
title={Reduce, Reuse, Recycle: Selective Reincarnation in Multi-Agent Reinforcement Learning},
author={Juan Claude Formanek and Callum Rhys Tilbury and Jonathan Phillip Shock and Kale-ab Tessera and Arnu Pretorius},
booktitle={Workshop on Reincarnating Reinforcement Learning at ICLR 2023},
year={2023},
url={https://openreview.net/forum?id=_Nz9lt2qQfV}
}

@misc{Jiang2021OfflineDM,
      title={Offline Decentralized Multi-Agent Reinforcement Learning}, 
      author={Jiechuan Jiang and Zongqing Lu},
      year={2021},
      eprint={2108.01832},
      archivePrefix={arXiv},
      primaryClass={cs.LG},
      url={https://arxiv.org/abs/2108.01832}, 
}

@article{Eldeeb2024ConservativeAR,
   title={Conservative and Risk-Aware Offline Multi-Agent Reinforcement Learning},
   ISSN={2372-2045},
   url={http://dx.doi.org/10.1109/TCCN.2024.3499357},
   DOI={10.1109/tccn.2024.3499357},
   journal={IEEE Transactions on Cognitive Communications and Networking},
   publisher={Institute of Electrical and Electronics Engineers (IEEE)},
   author={Eldeeb, Eslam and Sifaou, Houssem and Simeone, Osvaldo and Shehab, Mohammad and Alves, Hirley},
   year={2024},
   pages={1–1} }

@misc{Liu2024OfflineMR,
      title={Offline Multi-Agent Reinforcement Learning via In-Sample Sequential Policy Optimization}, 
      author={Zongkai Liu and Qian Lin and Chao Yu and Xiawei Wu and Yile Liang and Donghui Li and Xuetao Ding},
      year={2024},
      eprint={2412.07639},
      archivePrefix={arXiv},
      primaryClass={cs.AI},
      url={https://arxiv.org/abs/2412.07639}, 
}

@article{Zhou_2025, title={Cooperative Policy Agreement: Learning Diverse Policy for Offline MARL}, volume={39}, url={https://ojs.aaai.org/index.php/AAAI/article/view/34465}, DOI={10.1609/aaai.v39i21.34465}, number={21}, journal={Proceedings of the AAAI Conference on Artificial Intelligence}, author={Zhou, Yihe and Zheng, Yuxuan and Hu, Yue and Chen, Kaixuan and Zheng, Tongya and Song, Jie and Song, Mingli and Liu, Shunyu}, year={2025}, month={Apr.}, pages={23018-23026} }

@inproceedings{
    liu2025learning,
    title={Learning Generalizable Skills from Offline Multi-Task Data for Multi-Agent Cooperation},
    author={Sicong Liu and Yang Shu and Chenjuan Guo and Bin Yang},
    booktitle={The Thirteenth International Conference on Learning Representations},
    year={2025},
}

@inproceedings{
    zhang2023discovering,
    title={Discovering Generalizable Multi-agent Coordination Skills from Multi-task Offline Data},
    author={Fuxiang Zhang and Chengxing Jia and Yi-Chen Li and Lei Yuan and Yang Yu and Zongzhang Zhang},
    booktitle={The Eleventh International Conference on Learning Representations },
    year={2023},
    url={https://openreview.net/forum?id=53FyUAdP7d}
}

@inproceedings{
wu2023conservative,
title={Conservative Offline Policy Adaptation in Multi-Agent Games},
author={Chengjie Wu and Pingzhong Tang and Jun Yang and Yujing Hu and Tangjie Lv and Changjie Fan and Chongjie Zhang},
booktitle={Thirty-seventh Conference on Neural Information Processing Systems},
year={2023},
url={https://openreview.net/forum?id=C8pvL8Qbfa}
}

@inproceedings{jaxmarl,
  author = {Rutherford, Alexander and
          Ellis, Benjamin and
          Gallici, Matteo and
          Cook, Jonathan and
          Lupu, Andrei and
          Ingvarsson, Gar{\dh}ar and
          Willi, Timon and
          Hammond, Ravi and
          Khan, Akbir and
          de Witt, Christian Schroeder and
          Souly, Alexandra and
          Bandyopadhyay, Saptarashmi and
          Samvelyan, Mikayel and
          Jiang, Minqi and
          Lange, Robert and
          Whiteson, Shimon and
          Lacerda, Bruno and
          Hawes, Nick and
          Rockt{\"a}schel, Tim and
          Lu, Chris and
          Foerster, Jakob},
  booktitle = {Advances in Neural Information Processing Systems},
  editor = {
    A. Globerson and
    L. Mackey and
    D. Belgrave and
    A. Fan and
    U. Paquet and
    J. Tomczak and
    C. Zhang
  },
  pages = {50925--50951},
  publisher = {Curran Associates, Inc.},
  title = {JaxMARL: Multi-Agent RL Environments and Algorithms in JAX},
  volume = {37},
  year = {2024}
}

@article{samvelyan19smac,
  title = {{The} {StarCraft} {Multi}-{Agent} {Challenge}},
  author = {Mikayel Samvelyan and Tabish Rashid and Christian Schroeder de Witt and Gregory Farquhar and Nantas Nardelli and Tim G. J. Rudner and Chia-Man Hung and Philiph H. S. Torr and Jakob Foerster and Shimon Whiteson},
  journal = {CoRR},
  volume = {abs/1902.04043},
  year = {2019},
}

@inproceedings{
jing2024towards,
title={Towards Offline Opponent Modeling with In-context Learning},
author={Yuheng Jing and Kai Li and Bingyun Liu and Yifan Zang and Haobo Fu and QIANG FU and Junliang Xing and Jian Cheng},
booktitle={The Twelfth International Conference on Learning Representations},
year={2024},
}

@inproceedings{barde2024modelbasedsolutionofflinemultiagent,
      title={A Model-Based Solution to the Offline Multi-Agent Reinforcement Learning Coordination Problem}, 
      author={Paul Barde and Jakob Foerster and Derek Nowrouzezahrai and Amy Zhang},
      year={2024},
      booktitle={International Conference on Autonomous Agents and Multiagent Systems},
}

@inproceedings{cui2022provablyefficientofflinemultiagent,
 author = {Cui, Qiwen and Du, Simon S},
 booktitle = {Advances in Neural Information Processing Systems},
 editor = {S. Koyejo and S. Mohamed and A. Agarwal and D. Belgrave and K. Cho and A. Oh},
 publisher = {Curran Associates, Inc.},
 title = {Provably Efficient Offline Multi-agent Reinforcement Learning via Strategy-wise Bonus},
 volume = {35},
 year = {2022}
}

@inproceedings{cui2022offlinetwoplayerzerosummarkov,
 author = {Cui, Qiwen and Du, Simon S},
 booktitle = {Advances in Neural Information Processing Systems},
 editor = {S. Koyejo and S. Mohamed and A. Agarwal and D. Belgrave and K. Cho and A. Oh},
 publisher = {Curran Associates, Inc.},
 title = {When are Offline Two-Player Zero-Sum Markov Games Solvable?},
 volume = {35},
 year = {2022}
}

@misc{zhong2022pessimisticminimaxvalueiteration,
      title={Pessimistic Minimax Value Iteration: Provably Efficient Equilibrium Learning from Offline Datasets}, 
      author={Han Zhong and Wei Xiong and Jiyuan Tan and Liwei Wang and Tong Zhang and Zhaoran Wang and Zhuoran Yang},
      year={2022},
      eprint={2202.07511},
      archivePrefix={arXiv},
      primaryClass={cs.LG},
      url={https://arxiv.org/abs/2202.07511}, 
}

@misc{zhang2023offlinelearningmarkovgames,
      title={Offline Learning in Markov Games with General Function Approximation}, 
      author={Yuheng Zhang and Yu Bai and Nan Jiang},
      year={2023},
      eprint={2302.02571},
      archivePrefix={arXiv},
      primaryClass={cs.LG},
      url={https://arxiv.org/abs/2302.02571}, 
}

@misc{xiong2023nearlyminimaxoptimaloffline,
      title={Nearly Minimax Optimal Offline Reinforcement Learning with Linear Function Approximation: Single-Agent MDP and Markov Game}, 
      author={Wei Xiong and Han Zhong and Chengshuai Shi and Cong Shen and Liwei Wang and Tong Zhang},
      year={2023},
      eprint={2205.15512},
      archivePrefix={arXiv},
      primaryClass={cs.LG},
      url={https://arxiv.org/abs/2205.15512}, 
}

@misc{zhong2024offlinetoonlinemultiagentreinforcementlearning,
      title={Offline-to-Online Multi-Agent Reinforcement Learning with Offline Value Function Memory and Sequential Exploration}, 
      author={Hai Zhong and Xun Wang and Zhuoran Li and Longbo Huang},
      year={2024},
      eprint={2410.19450},
      archivePrefix={arXiv},
      primaryClass={cs.AI},
      url={https://arxiv.org/abs/2410.19450}, 
}

@misc{samvelyan2019starcraftmultiagentchallenge,
      title={The StarCraft Multi-Agent Challenge}, 
      author={Mikayel Samvelyan and Tabish Rashid and Christian Schroeder de Witt and Gregory Farquhar and Nantas Nardelli and Tim G. J. Rudner and Chia-Man Hung and Philip H. S. Torr and Jakob Foerster and Shimon Whiteson},
      year={2019},
      eprint={1902.04043},
      archivePrefix={arXiv},
      primaryClass={cs.LG},
}

@inproceedings{
    tilbury2024coordination,
    title={Coordination Failure in Cooperative Offline {MARL}},
    author={Callum Rhys Tilbury and Juan Claude Formanek and Louise Beyers and Jonathan Phillip Shock and Arnu Pretorius},
    booktitle={ICML 2024 Workshop: Aligning Reinforcement Learning Experimentalists and Theorists},
    year={2024},
}

@inproceedings{
    rosen2024optimal,
    title={Optimal Task Generalisation in Multi-Agent Reinforcement Learning},
    author={Simon Rosen and Abdel Mfougouon Njupoun and Geraud Nangue Tasse and Steven James and Benjamin Rosman},
    booktitle={Coordination and Cooperation for Multi-Agent Reinforcement Learning Methods Workshop},
    year={2024},
}

@inproceedings{lukas2022rwaremt,
    author = {Sch\"{a}fer, Lukas},
    title = {Task Generalisation in Multi-Agent Reinforcement Learning},
    year = {2022},
    publisher = {International Foundation for Autonomous Agents and Multiagent Systems},
    booktitle = {Proceedings of the 21st International Conference on Autonomous Agents and Multiagent Systems},
    pages = {1863–1865},
    numpages = {3},
    series = {AAMAS '22}
    }

@misc{hilton2023scalinglawssingleagentreinforcement,
      title={Scaling laws for single-agent reinforcement learning}, 
      author={Jacob Hilton and Jie Tang and John Schulman},
      year={2023},
      eprint={2301.13442},
      archivePrefix={arXiv},
      primaryClass={cs.LG},
      url={https://arxiv.org/abs/2301.13442}, 
}

@inproceedings{schweighofer2022dataset,
  title={A dataset perspective on offline reinforcement learning},
  author={Schweighofer, Kajetan and Dinu, Marius-constantin and Radler, Andreas and Hofmarcher, Markus and Patil, Vihang Prakash and Bitto-Nemling, Angela and Eghbal-Zadeh, Hamid and Hochreiter, Sepp},
  booktitle={Conference on Lifelong Learning Agents},
  pages={470--517},
  year={2022},
  organization={PMLR}
}

@article{liumaskma2024,
  title={MaskMA: Towards Zero-Shot Multi-Agent Decision Making with Mask-Based Collaborative Learning},
  author={Liu, Jie and Zhang, Yinmin and Li, Chuming and You, Zhiyuan and Zhou, Zhanhui and Yang, Chao and Yang, Yaodong and Liu, Yu and Ouyang, Wanli},
  journal={Transactions on Machine Learning Research},
year={2024}
}
\bibliographystyle{iclr2026_conference}

\newpage
\appendix
\newpage
\section{Proofs and Auxiliary Results}
\label{app:theory}

\subsection{Proof of Theorem~\ref{thm:main}}

\begin{proof}
Let $z_{NN} = \arg\min_{z\in\mathcal T_{\text{train}}} d(z, z_{\text{test}})$ and write $r := d(z_{NN}, z_{\text{test}})$. By Assumption~\ref{ass:common}, $\pi^*_{z_{\text{test}}}, \pi^*_{z_{NN}}, \pi_\theta \in \Pi$ are all evaluable on both $\mathcal M_{z_{NN}}$ and $\mathcal M_{z_{\text{test}}}$. Applying Assumption~\ref{ass:lip} to $\pi^*_{z_{\text{test}}}$ and to $\pi_\theta$ separately, and using the symmetry of $d$,
\begin{align*}
J_{z_{\text{test}}}(\pi^*_{z_{\text{test}}}) - J_{z_{\text{test}}}(\pi_\theta)
&= \big[J_{z_{\text{test}}}(\pi^*_{z_{\text{test}}}) - J_{z_{NN}}(\pi^*_{z_{\text{test}}})\big] \\
&\quad + \big[J_{z_{NN}}(\pi^*_{z_{\text{test}}}) - J_{z_{NN}}(\pi_\theta)\big] \\
&\quad + \big[J_{z_{NN}}(\pi_\theta) - J_{z_{\text{test}}}(\pi_\theta)\big] \\
&\le Lr + \big[J_{z_{NN}}(\pi^*_{z_{\text{test}}}) - J_{z_{NN}}(\pi_\theta)\big] + Lr.
\end{align*}
Since $\pi^*_{z_{\text{test}}}\in\Pi$ and $\pi^*_{z_{NN}}$ is $\Pi$-optimal on $z_{NN}$, we have $J_{z_{NN}}(\pi^*_{z_{\text{test}}}) \le J_{z_{NN}}(\pi^*_{z_{NN}})$, so
\[
J_{z_{NN}}(\pi^*_{z_{\text{test}}}) - J_{z_{NN}}(\pi_\theta) \le J_{z_{NN}}(\pi^*_{z_{NN}}) - J_{z_{NN}}(\pi_\theta) \le \varepsilon_{\text{train}},
\]
by Assumption~\ref{ass:train}. Combining gives $\mathcal R(z_{\text{test}}) \le \varepsilon_{\text{train}} + 2Lr$.
\end{proof}

\subsection{Proof of Corollary~\ref{cor:finite}}

\begin{proof}
If $z_{\text{test}}\in\mathcal T_{\text{train}}$ then $\min_i d(z_i, z_{\text{test}}) = 0$ and Theorem~\ref{thm:main} gives $\mathcal R(z_{\text{test}}) \le \varepsilon_{\text{train}}$. The general case follows by taking a supremum over $z_{\text{test}}\in\mathcal Z_{\text{test}}$ in Theorem~\ref{thm:main}.
\end{proof}

\subsection{Sufficient conditions for Assumption~\ref{ass:lip}}
\label{app:lip-sufficient}

\begin{lemma}[Simulation lemma for the task metric]
\label{lem:sim}
Suppose Assumption~\ref{ass:common} holds, rewards lie in $[0,R_{\max}]$, and there exist $L_R, L_P \ge 0$ such that for all $(s,a)$ and all $z, z' \in \mathcal Z$,
\begin{enumerate}
  \item $|R_z(s,a) - R_{z'}(s,a)| \le L_R\, d(z,z')$,
  \item $D_{TV}(P_z(\cdot|s,a), P_{z'}(\cdot|s,a)) \le L_P\, d(z,z')$.
\end{enumerate}
Then for every history-dependent policy $\pi \in \Pi$ and all $z, z' \in \mathcal Z$,
\[
|J_z(\pi) - J_{z'}(\pi)| \;\le\; L\, d(z,z'),
\qquad L := \frac{L_R}{1-\gamma} + \frac{\gamma\, L_P\, R_{\max}}{(1-\gamma)^2}.
\]
\end{lemma}

\begin{proof}
Fix $\pi\in\Pi$ and $z,z'\in\mathcal Z$. Write $V^\pi_z(s) := \mathbb E_{\pi,P_z}\!\big[\sum_{t\ge 0}\gamma^t R_z(s_t,a_t)\,|\,s_0=s\big]$ and similarly $V^\pi_{z'}$. Since rewards lie in $[0,R_{\max}]$, $V^\pi_{z'}$ takes values in $[0,R_{\max}/(1-\gamma)]$, and in particular
\[
\mathrm{osc}(V^\pi_{z'}) := \sup V^\pi_{z'} - \inf V^\pi_{z'} \;\le\; \frac{R_{\max}}{1-\gamma}. \tag{$\ast$}
\]
For history-dependent $\pi$, we work with the random history $h_t=(s_0,a_0,\dots,s_t)$ and define $V^\pi_z(h_t) := \mathbb E_{\pi,P_z}\!\big[\sum_{k\ge 0}\gamma^k R_z(s_{t+k},a_{t+k})\,|\,h_t\big]$; the argument below carries through verbatim with $s$ replaced by $h_t$, since the transition kernels $P_z, P_{z'}$ depend only on $(s_t,a_t)$ and the policy $\pi(\cdot|h_t)$ is the same on both tasks. We write $s$ for readability.

By the Bellman equation,
\begin{align*}
V^\pi_z(s) - V^\pi_{z'}(s)
&= \mathbb E_{a\sim\pi(\cdot|s)}\Big[(R_z(s,a)-R_{z'}(s,a)) \\
&\qquad\qquad + \gamma\big(\mathbb E_{P_z(\cdot|s,a)}[V^\pi_z(s')] - \mathbb E_{P_{z'}(\cdot|s,a)}[V^\pi_{z'}(s')]\big)\Big].
\end{align*}
Adding and subtracting $\gamma\,\mathbb E_{P_z(\cdot|s,a)}[V^\pi_{z'}(s')]$,
\begin{align*}
\mathbb E_{P_z}[V^\pi_z] - \mathbb E_{P_{z'}}[V^\pi_{z'}]
&= \mathbb E_{P_z}[V^\pi_z - V^\pi_{z'}] + \big(\mathbb E_{P_z}[V^\pi_{z'}] - \mathbb E_{P_{z'}}[V^\pi_{z'}]\big).
\end{align*}
Using $|\mathbb E_p[f] - \mathbb E_q[f]| \le \mathrm{osc}(f)\,D_{TV}(p,q)$ together with $(\ast)$ and hypothesis (2),
\[
\big|\mathbb E_{P_z}[V^\pi_{z'}] - \mathbb E_{P_{z'}}[V^\pi_{z'}]\big| \;\le\; \frac{R_{\max}}{1-\gamma}\cdot L_P\, d(z,z').
\]
Combining with hypothesis (1),
\[
\big|V^\pi_z(s) - V^\pi_{z'}(s)\big| \;\le\; L_R\,d(z,z') + \gamma\,\mathbb E_{P_z}\!\big|V^\pi_z(s') - V^\pi_{z'}(s')\big| + \frac{\gamma\,L_P R_{\max}}{1-\gamma}\,d(z,z').
\]
Taking the sup over $s$ and writing $u := \|V^\pi_z - V^\pi_{z'}\|_\infty$, $c := \big(L_R + \tfrac{\gamma L_P R_{\max}}{1-\gamma}\big)d(z,z')$, this gives $u \le c + \gamma u$, hence
\[
u \;\le\; \frac{c}{1-\gamma} \;=\; \left(\frac{L_R}{1-\gamma} + \frac{\gamma L_P R_{\max}}{(1-\gamma)^2}\right)d(z,z') \;=\; L\,d(z,z').
\]
Finally, $|J_z(\pi)-J_{z'}(\pi)| = |\mathbb E_{\rho}[V^\pi_z - V^\pi_{z'}]| \le \|V^\pi_z - V^\pi_{z'}\|_\infty \le L\,d(z,z')$, using the common initial distribution from Assumption~\ref{ass:common}.
\end{proof}

Lemma~\ref{lem:sim} provides one set of sufficient conditions under which Assumption~\ref{ass:lip} holds; Theorem~\ref{thm:main} only requires Assumption~\ref{ass:lip} itself, irrespective of how it is justified.

\subsection{Proxy Coverage: Computation Details}
\label{app:proxy}

For each environment, we represent each task $z$ by a hand-crafted descriptor $\hat{z} \in \mathbb{R}^k$ built from key task properties, listed in \autoref{tab:descriptors}.

\begin{table}[h]
\centering
\begin{tabular}{ll}
\toprule
\textbf{Environment} & \textbf{Descriptor components} \\
\midrule
\texttt{LBF}       & Maximum agent level, number of agents \\
\texttt{SMAX}      & Number of ally agents, number of enemy agents \\
\texttt{RWARE}     & Number of shelves, number of agents \\
\texttt{Connector} & Agent density, grid size \\
\bottomrule
\end{tabular}
\caption{Task descriptors $\hat{z}$ used to compute the proxy coverage.}
\label{tab:descriptors}
\end{table}

Given a test task $z_{\text{test}}$ and a training set $\mathcal{T}_{\text{train}} = \{z_1, \dots, z_N\}$, the \textbf{proxy coverage error} is the distance from the test descriptor to its nearest training neighbour,
\[
\text{PCE}(z_{\text{test}}) \;=\; \min_{z_i \in \mathcal{T}_{\text{train}}} \lVert \hat{z}_i - \hat{z}_{\text{test}} \rVert_2,
\]
and the proxy coverage is $\mathrm{PC} = 1 - \mathrm{PCE}$, so that higher values indicate better coverage. To allow comparison across environments with descriptors of different scales, all descriptor components are normalised to $[0,1]$ using the range observed across the union of training and test tasks. The PC values reported in \autoref{fig:coverage_error} are further normalised per environment by their value at $N=1$, so that curves start at the same reference point and differences in slope reflect task diversity rather than descriptor scale.

\newpage
\section{Environment Details}

\subsection{LBF}

\begin{figure}[h]
    \centering
    \begin{subfigure}[t]{0.3\linewidth}
        \includegraphics[width=\linewidth]{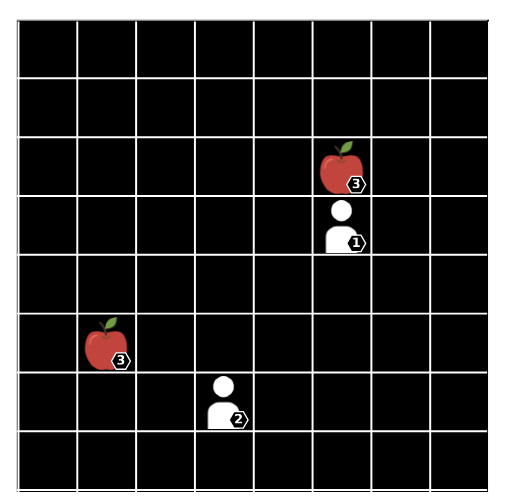}
        \caption{\texttt{8x8-2p-2f}}
        \label{fig:lbf1}
    \end{subfigure}
    \hspace{1cm}
    \begin{subfigure}[t]{0.29\linewidth}
        \includegraphics[width=\linewidth]{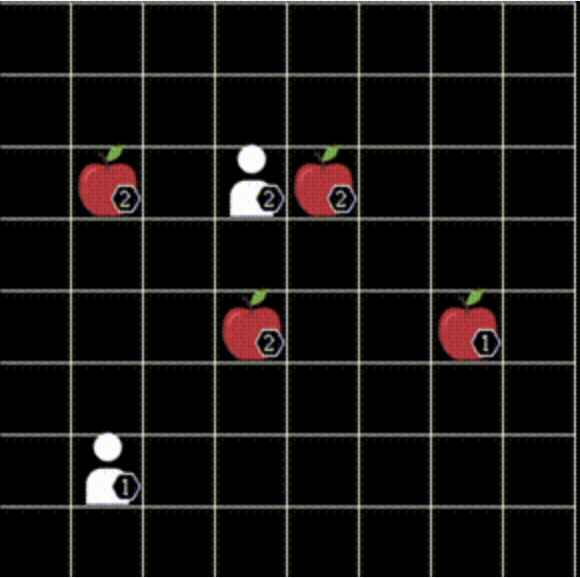}
         \caption{\texttt{8x8-2p-4f}}
        \label{fig:lbf2}
    \end{subfigure}
    \caption{LBF}
    \label{fig:lbf}
\end{figure}

In the Level-Based Foraging (LBF) environment, which is a JAX-based implementation from the Jumanji suite~\citep{bonnetjumanji} of the original framework by \citet{papoudakis2021benchmarking}, agents with assigned levels navigate a grid world to collect food items that can only be consumed if the sum of adjacent agent levels exceeds the food's level. These tasks are defined by the naming convention \texttt{<x size>x<y size>-<n agents>p-<food>f}, specifying the grid dimensions, agent and food counts. Agents observe a limited  $5\times 5$ square grid centered on their location which reveals the positions and levels of nearby items. Operating via a discrete action space of six options that includes no-operation, loading food, and movement in the four cardinal directions, agents receive rewards calculated as the sum of collected food levels divided by the level of the contributing agents.

\subsection{Connector}

\begin{figure}[h]
    \centering
    \begin{subfigure}[t]{0.3\linewidth}
        \includegraphics[width=\linewidth]{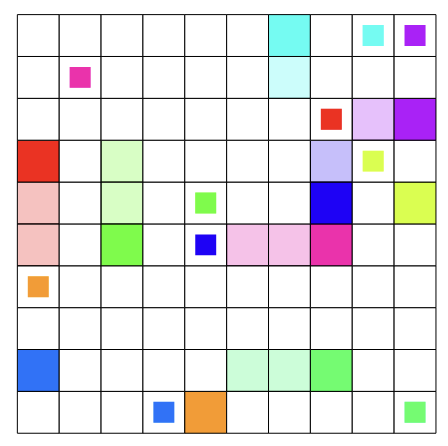}
        \caption{\texttt{con-10x10-10a}}
        \label{fig:con1}
    \end{subfigure}
    \hspace{1cm}
    \begin{subfigure}[t]{0.3\linewidth}
        \includegraphics[width=\linewidth]{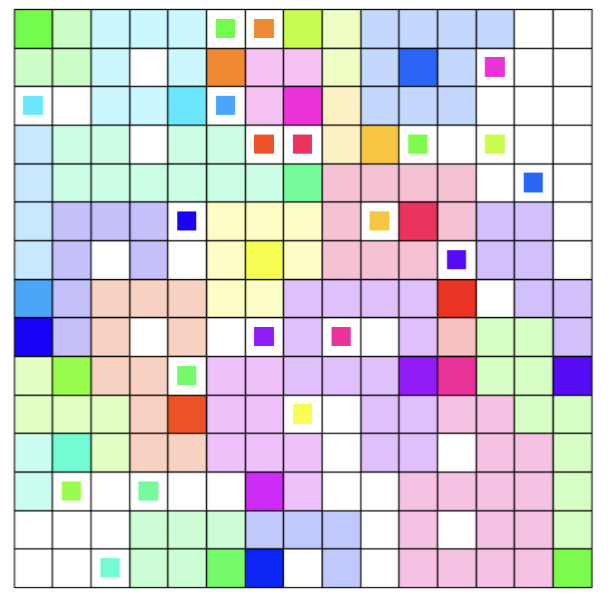}
         \caption{\texttt{con-15x15-23a}}
        \label{fig:con2}
    \end{subfigure}
    \caption{Connector}
    \label{fig:connector}
\end{figure}

In the Connector~\citep{bonnetjumanji} environment, multiple agents are randomly initialized within a grid world to connect assigned start and end points in the minimum number of steps, a task complicated by the fact that movement creates permanent, impassable trails which necessitate cooperation to avoid blocking teammates. These tasks follow the naming convention \texttt{con-<x-size>x<y-size>-<num\_agents>a} to specify grid dimensions and agent count. Agents operate within this system by observing an $n\times n$ local view centered on their location that reveals trails and all target locations, while also accessing the global $(x,y)$ coordinates of their current position and specific destination. Acting through a discrete space of five options including up, down, left, right, and stop, agents are guided by a reward function that yields $+1$ at the moment of connection and a penalty of $-0.03$ for every other step until completion.

\subsection{RWARE}

\begin{figure}[h]
    \centering
    \begin{subfigure}[t]{0.3\linewidth}
        \includegraphics[width=\linewidth]{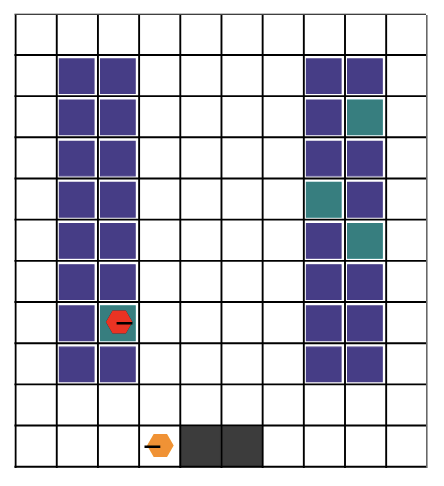}
        \caption{\texttt{tiny-2ag}}
        \label{fig:tiny2ag}
    \end{subfigure}
    \hspace{1cm}
    \begin{subfigure}[t]{0.3\linewidth}
        \includegraphics[width=\linewidth]{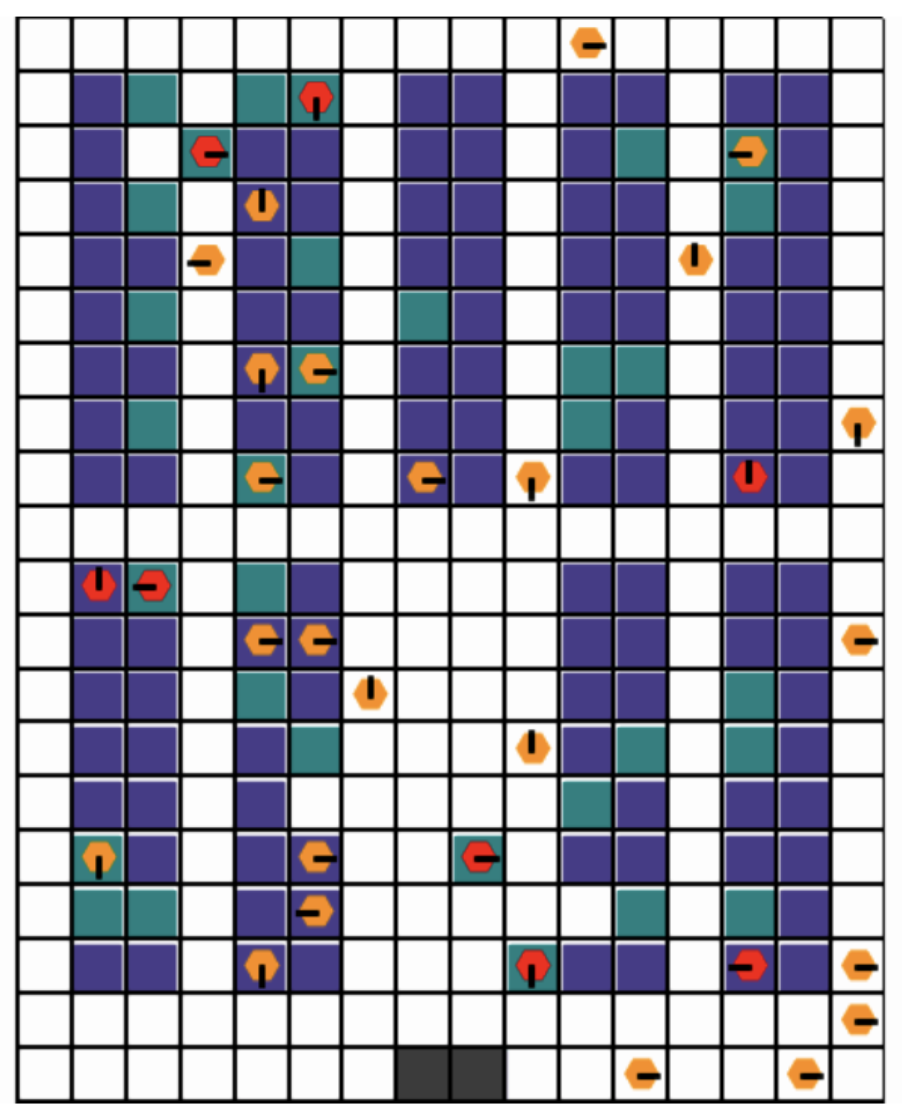}
         \caption{\texttt{medium-32ag}}
        \label{fig:medium32ag}
    \end{subfigure}
    \caption{RWARE}
    \label{fig:rware}
\end{figure}

The Robot Warehouse (RWARE) environment simulates a logistics scenario where a team of autonomous robots must fetch requested goods from shelves and deliver them to workstations to maximize throughput. We utilize the JAX-based implementation from the Jumanji suite~\citep{bonnetjumanji} based on the original work by~\citet{papoudakis2021benchmarking}, which notably terminates episodes immediately upon agent collision rather than attempting to resolve the conflict. Tasks follow the convention \texttt{<size>-<num agents>ag}, where the size determines the shelf layout. Agents operate under partial observability within a $3\times 3$ view centered on their position that reveals self and peer states alongside shelf status, using a discrete action space of five commands for navigation and loading to achieve a sparse reward of $+1$ granted solely for successful deliveries.

\subsection{SMAX}

\begin{figure}[h]
    \centering
    \begin{subfigure}[t]{0.3\linewidth}
        \includegraphics[width=\linewidth]{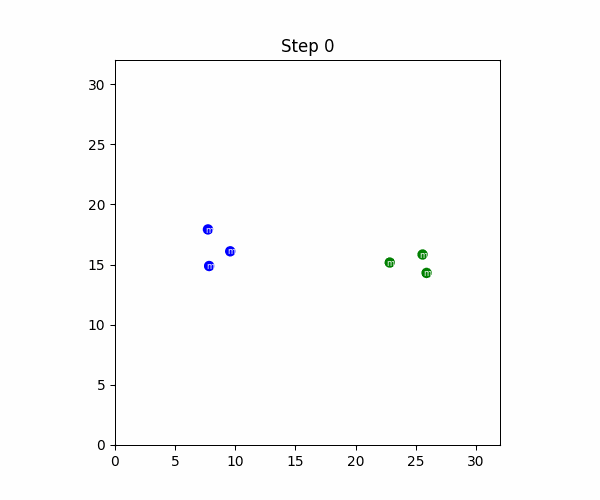}
        \caption{\texttt{3m}}
        \label{fig:3m}
    \end{subfigure}
    \hspace{1cm}
    \begin{subfigure}[t]{0.3\linewidth}
        \includegraphics[width=\linewidth]{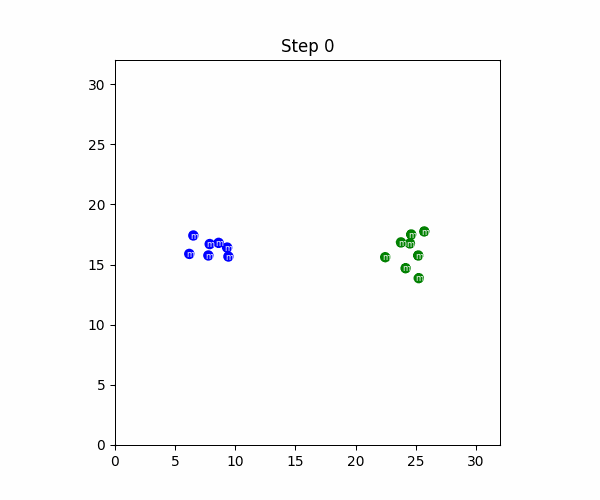}
         \caption{\texttt{7m vs 8m}}
        \label{fig:7mvs8m}
    \end{subfigure}
    \caption{SMAX}
    \label{fig:rware-env}
\end{figure}

SMAX is a simplified StarCraft micromanagement environment from the JaxMARL suite~\citep{jaxmarl}, inspired by the popular SMAC benchmark~\citep{samvelyan2019starcraftmultiagentchallenge}. In SMAX, teams of allied units must coordinate to defeat enemy units controlled by a heuristic AI. While SMAX supports a variety of unit types and team compositions, we focus on marine-only scenarios for this work. Tasks follow the convention \texttt{<N>m} for symmetric scenarios with $N$ marines on each side, or \texttt{<N>m\_vs\_<M>m} for asymmetric scenarios where $N$ allied marines face $M$ enemies. Agents operate under partial observability with a limited sight range, observing ally and enemy positions, health, and unit types within their field of view. The discrete action space includes movement in four cardinal directions, stopping, and attacking visible enemies. Agents receive a team reward of $+1$ for winning the battle, with optional shaped rewards for dealing damage and eliminating enemy units. Episodes terminate when all units on one side are eliminated or after a maximum of 100 timesteps.

\newpage
\section{Multi-Task Offline MARL can Generalise better than Behaviour Cloning}\label{ap:marl-vs-bc}
The findings from~\citet{medirattageneralization} paint a bleak outlook for the generalisation capabilities of Offline RL algorithms compared to simple behaviour cloning. 
To establish if we observe a similar trend, we aggregate the normalised episode returns across all test tasks from \texttt{LBF}, \texttt{RWARE}, \texttt{Connector} and \texttt{SMAX} when trained using the full training set, to compare the three algorithms. In \autoref{tab: aggregate_performance}, we show the mean and standard error for each algorithm. 

We investigate which offline training objective: behaviour cloning (BC), conservative Q-learning (CQL), or the autoregressive ICQ (Oryx) achieves the strongest generalisation to held-out test tasks. 
On \texttt{LBF} and \texttt{Connector}, Oryx performs best, followed by BC, with CQL performing worst. In contrast, on \texttt{RWARE}, CQL achieves the strongest generalisation, followed by ICQ and then BC. Finally, on \texttt{SMAX}, CQL again performs best, closely followed by BC, with Oryx performing the worse.

We hypothesize that these findings differ from those of~\citet{medirattageneralization} due to differences in dataset composition. While their study relies on purely \texttt{Expert} data, we train on mixed replay datasets. Expert data is well suited to behaviour cloning, whereas many offline RL methods such as CQL~\citep{schweighofer2022dataset} benefit from datasets with diverse behaviour quality. Indeed, the \texttt{LBF} and \texttt{Connector} datasets are heavily skewed towards \texttt{Expert} trajectories due to the relative simplicity of these tasks, which likely explains the weaker performance of CQL in these environments. In contrast, the \texttt{RWARE} datasets exhibit the greatest diversity in data quality, aligning with CQL’s superior performance.

Overall, our results suggest that in settings with mixed data quality, offline MARL methods exhibit stronger zero-shot generalisation than behaviour cloning.


\definecolor{oryxcolor}{RGB}{0, 0, 255}    
\definecolor{cqlcolor}{RGB}{255, 165, 0}     
\definecolor{bccolor}{RGB}{0, 128, 0}      

\begin{table}[h]
    \small
    \centering
    \caption{\textit{Comparison of test task performance of all three models.}The mean and standard error of the performance across all test tasks on \texttt{RWARE}, \texttt{LBF}, \texttt{Connector} and \texttt{SMAX} for each of the multi-task algorithms (largest mean highlighted with bold). In the final column the combined mean across all tasks from the four environments is computed. \textbf{In contrast to the findings by \citet{medirattageneralization}, we find that on each environment the best performing algorithm is an Offline RL method (MT CQL-Sable or MT Oryx), rather than the BC model. When aggregated across all the test tasks combined, MT Oryx performs the best.}}
    
    \begin{tabular}{cl||cccc|c}
    \toprule
    & Algorithm & RWARE & LBF & Connector & SMAX & Combined \\
    \midrule
    \textcolor{oryxcolor}{\Large$\bullet$} & MT Oryx      &  $0.587 \pm 0.054$ &  $\mathbf{0.803 \pm 0.026}$ &  $\mathbf{0.852 \pm 0.002}$ &  $0.73 \pm 0.012$ & $\mathbf{0.752 \pm 0.023}$\\
    \textcolor{cqlcolor}{\Large$\bullet$} & MT CQL-Sable &  $\mathbf{0.620 \pm 0.066}$ &  $0.562 \pm 0.029$ &  $0.668 \pm 0.018$ & $\mathbf{0.89 \pm 0.027}$  &  $0.697 \pm 0.024$ \\
    \textcolor{bccolor}{\Large$\bullet$}  & MT BC-Sable  &  $0.415 \pm 0.050$ &  $0.797 \pm 0.030$ &  $0.775 \pm 0.004$ & $0.88 \pm 0.01$ & $0.718 \pm 0.027$ \\
    \bottomrule
    \end{tabular}
    \label{tab: aggregate_performance}
\end{table}

\newpage
\section{Dataset Quality Ablation} \label{sec:dataset-quality-ablation}

\begin{figure}[h!]
\centering
\includegraphics[width=0.6\linewidth]{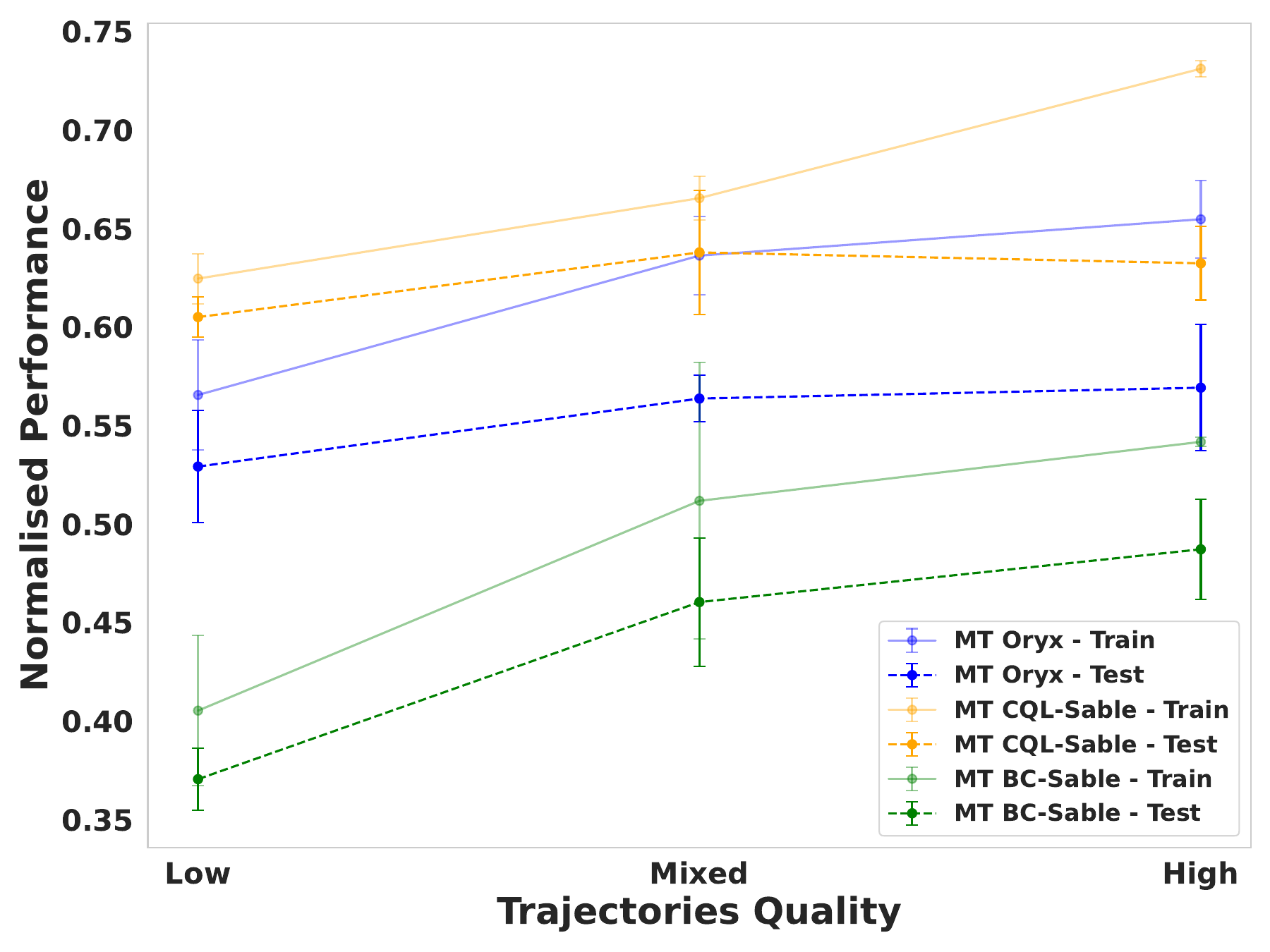}
\caption{\textit{Performance of MT-Oryx, MT-CQL-Sable, and MT-BC-Sable on RWARE with different trajectory subsets}. \textbf{High-quality trajectories improve training performance, particularly for MT-CQL-Sable, but these gains do not transfer to the test tasks. Low-quality trajectories consistently yield the worst results.}}
\label{fig:trajectory_quality}
\end{figure}

\textbf{Do higher quality trajectories improve generalisation?}
As observed in \autoref{sec:datasize_scaling}, increasing dataset size does not lead to significant improvements in generalization to unseen tasks. A natural follow-up question is: \emph{how does the quality of trajectories in the dataset affect training and test performance?} To investigate this, we conduct an experiment where training is performed with trajectories sampled from specific subsets of our dataset. Low-quality trajectories are those collected during the first two-thirds of the online training phase, while High-quality trajectories are those from the final third. Results on RWARE are shown in \autoref{fig:trajectory_quality}. For all algorithms, training performance improves with High-quality trajectories, though the gains on test tasks remain marginal. Across all three algorithms, training with Low-quality trajectories consistently yields the worst results on both training and test tasks. These results suggest that the most effective strategy is to prioritize High-quality trajectories while retaining a small fraction of Low-quality ones as negative examples.

\newpage
\section{Scaling Analysis on \texttt{LBF} and \texttt{Connector}} \label{sec:connnector_model_scaling}

In this section, we complement the experiments presented in \autoref{sec:datasize_scaling}. We verify whether the model-size scaling trends observed in \texttt{RWARE} also extends to \texttt{LBF}, \texttt{SMAX}, and \texttt{Connector}. As shown in \autoref{fig:connector_scaling}, we observe similar behavior: performance improves with model size up to a critical point. However, All \texttt{LBF}, \texttt{Connector}, and \texttt{SMAX} are considerably easier than \texttt{RWARE}, and therefore their performance saturates at much smaller model sizes. Furthermore, although there is a large performance gap between BC-Sable and the other algorithms on \texttt{LBF}, the overall scaling trend remains visible, albeit more marginal.
\begin{figure}[h]
    \centering
    \begin{subfigure}[t]{0.47\linewidth}
        \includegraphics[width=\linewidth]{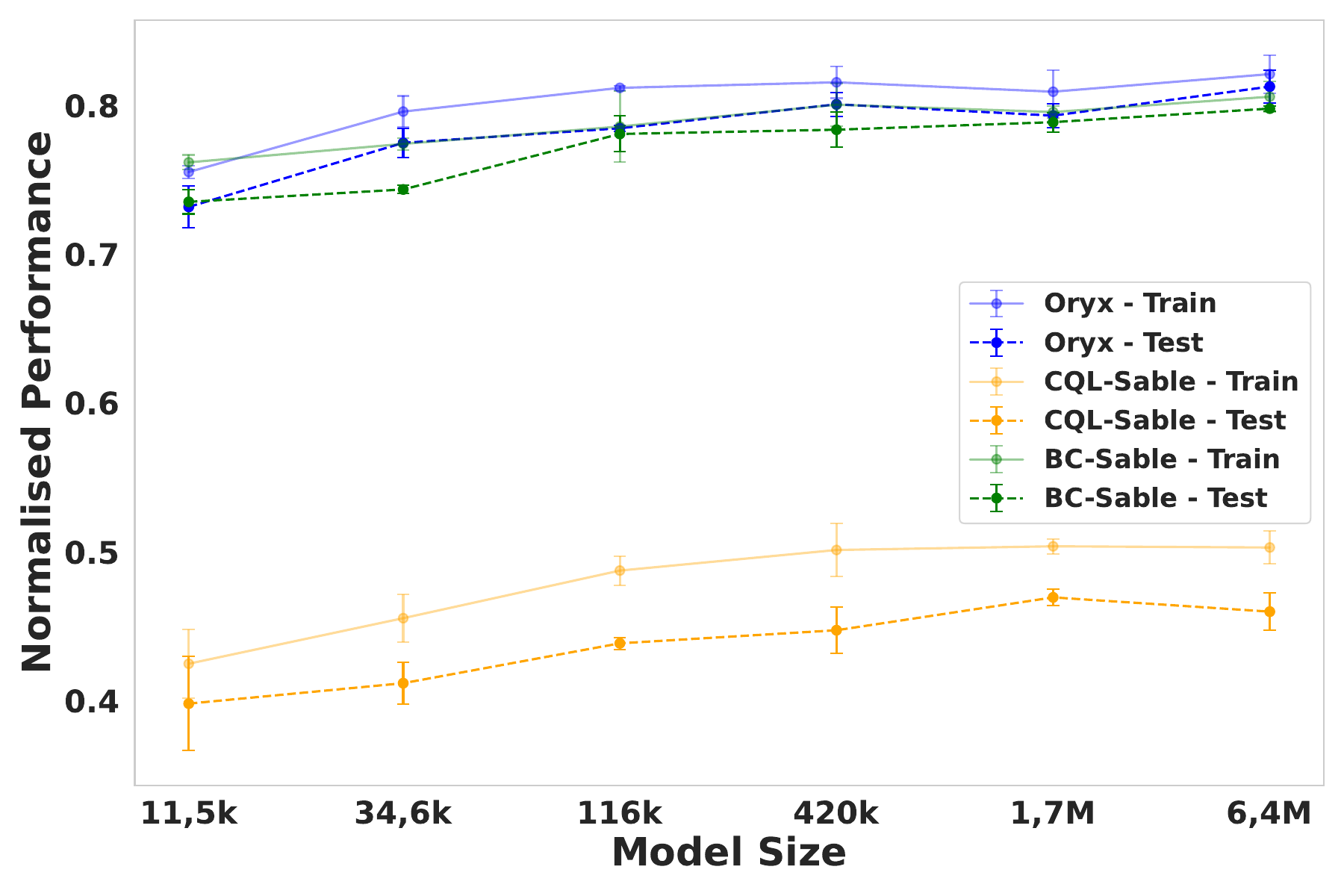}
        \caption{\texttt{LBF}}
        \label{fig:network_lbf}
    \end{subfigure}
    \hfill
    \begin{subfigure}[t]{0.47\linewidth}
        \includegraphics[width=\linewidth]{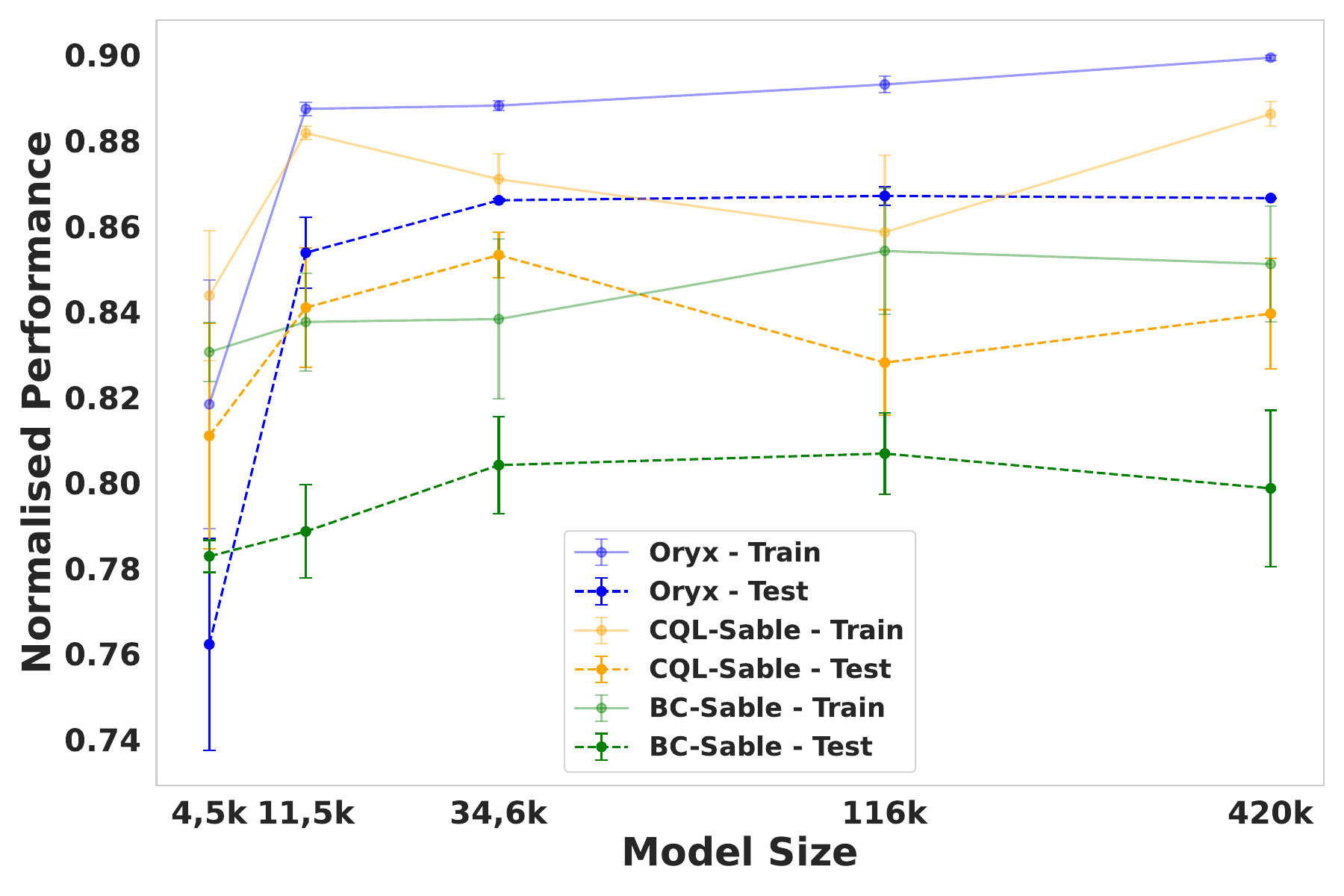}
         \caption{\texttt{Connector}}
        \label{fig:network_connector}
    \end{subfigure}
    \begin{subfigure}[t]{0.47\linewidth}
        \includegraphics[width=\linewidth]{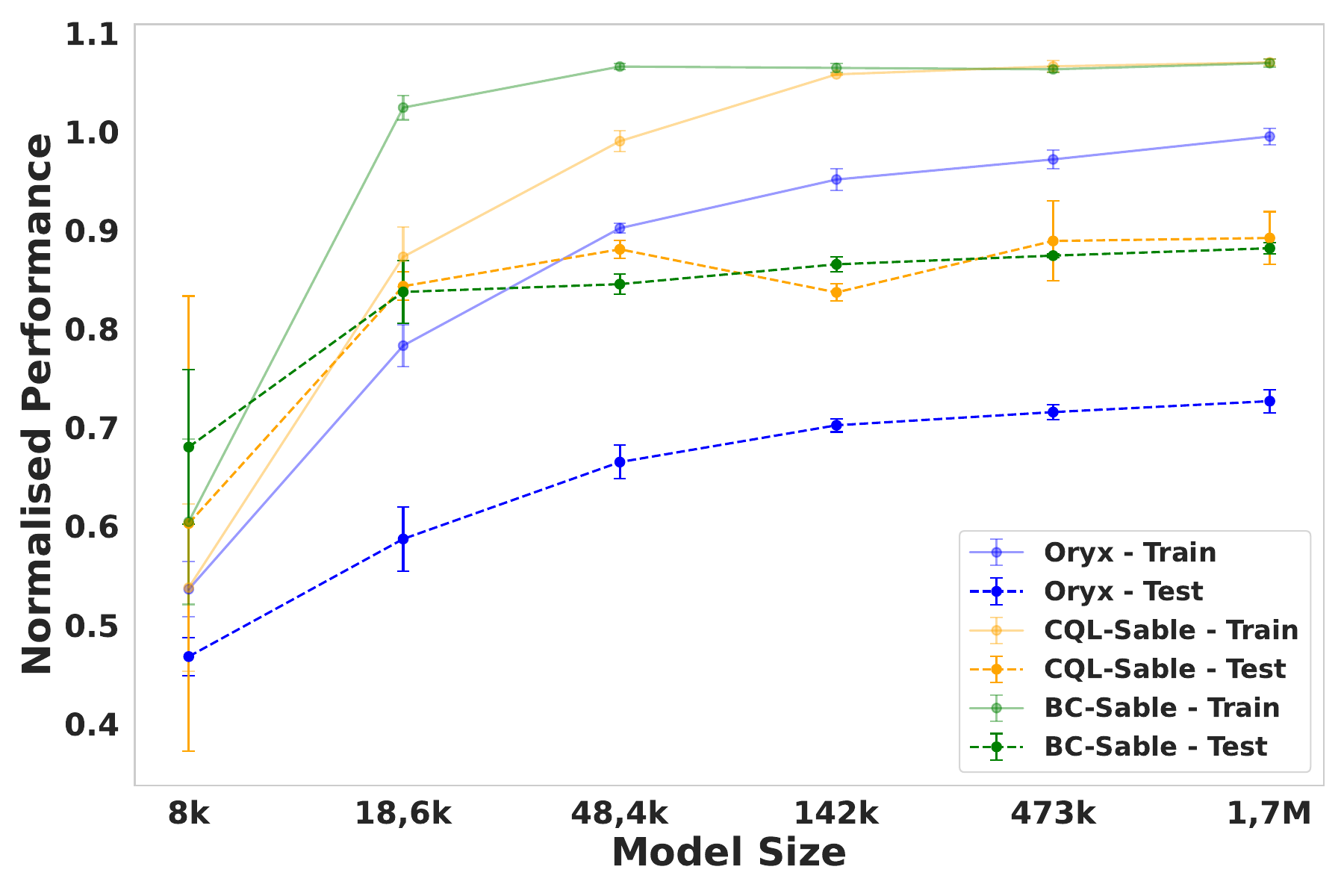}
        \caption{\texttt{SMAX}}
        \label{fig:network_smac}
    \end{subfigure}
    \caption{\textit{Performance of MT-Oryx, MT-CQL-Sable, and MT-BC-Sable on \texttt{LBF},  \texttt{Connector}, and \texttt{SMAX} with different model sizes}. \textbf{Both train and test performance of all algorithms improve with increasing model size up to a critical threshold, beyond which performance plateaus.}}
     \label{fig:connector_scaling}
\end{figure}

\clearpage
\section{Visualisation of Multi-Task Policy}
\label{sec:viz_mt_policy}

In order to qualitativly validate that the MT models have learn multiple team strategies which are quite distinct across tasks we visually inspected roll-outs across tasks. Here we visualise the learn strategy on two very distinct tasks \texttt{medium-2ag} and \texttt{medium-32ag}\footnote{GIFs available on website: \url{https://sites.google.com/view/multi-task-marl}}. The main challenge in the first task is the sparsity of the warehouse. Accordingly the model learnt a strategy whereby the two agents rapidly traverse the warehouse to explore efficiently and find the shelf to be collected. In contrast, the central challenge on the second task is that the warehouse is very congested. If the agents collide the episode ends. Accordingly the model learnt a smart strategy of moving completed agents out of the way by sending them to the bottom right-hand corner. Importantly, a single MT model learn both of these different multi-agent strategies simultaneously.

\begin{figure}[h]
    \centering
    \includegraphics[width=0.9\linewidth]{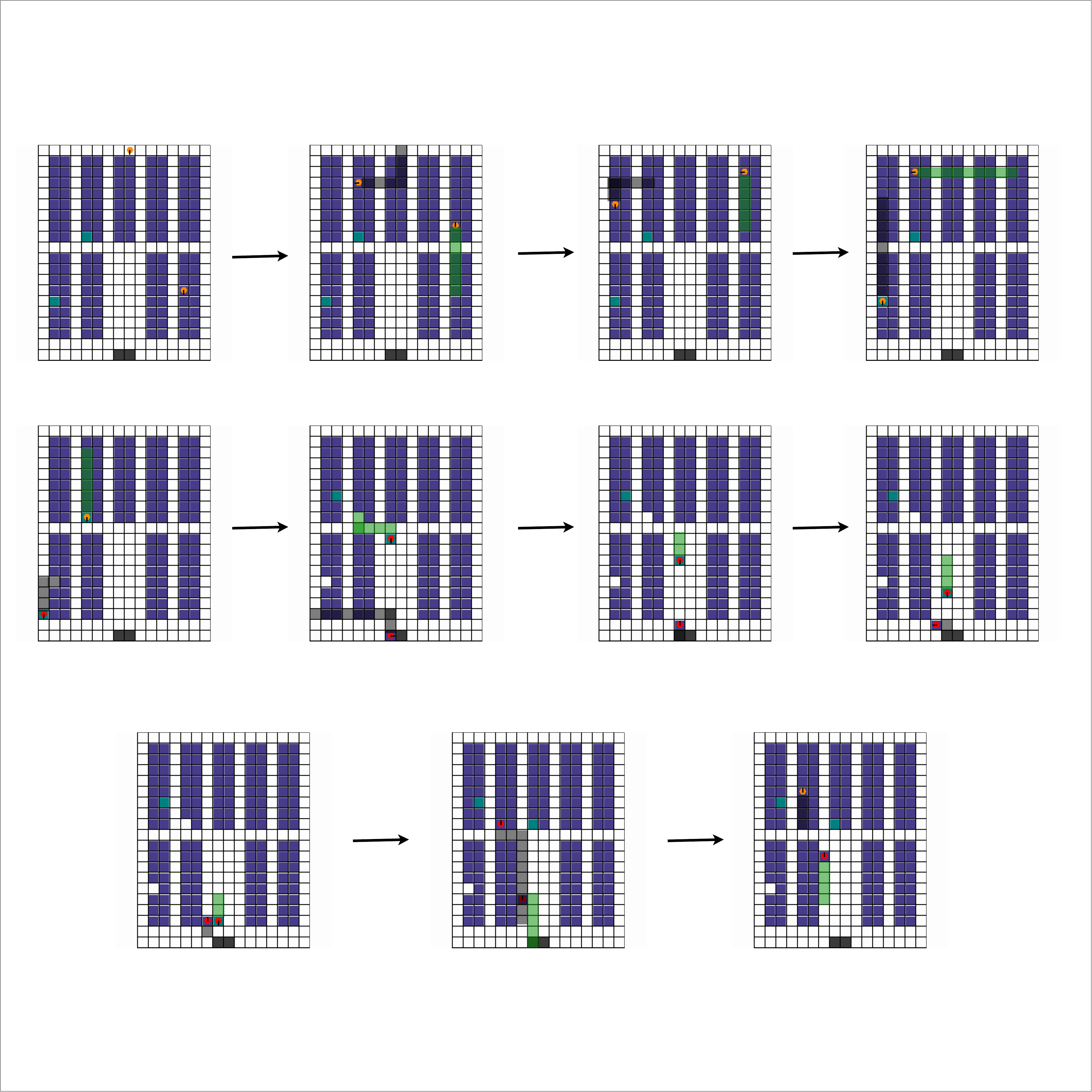}
    \caption{Visualisation of team strategy on \texttt{medium-2ag}. Frames should be read left to right, top to bottom.}
    \label{fig:medium-2ag}
\end{figure}

\clearpage
\begin{figure}
    \centering
    \includegraphics[width=0.9\linewidth]{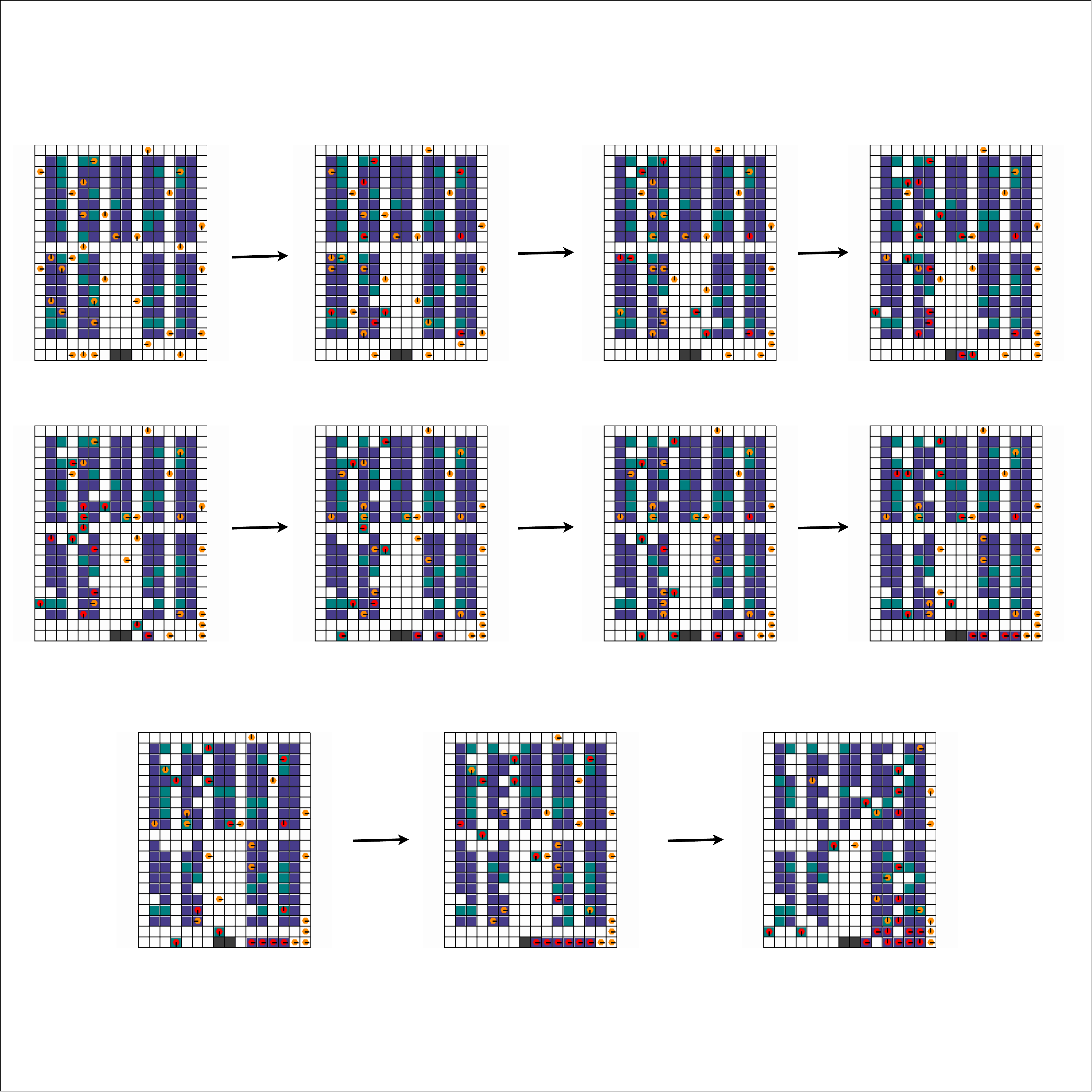}
    \caption{Visualisation of team strategy on \texttt{medium-32ag}. Frames should be read left to right, top to bottom.}
    \label{fig:medium-32ag}
\end{figure}

\clearpage
\section{Computational Requirements}
\label{sec:compute_requirements}

All experiments were conducted on a high-performance computing cluster utilizing the Jobset operator for orchestration. Each experimental run was allocated a single worker node equipped with one \textbf{NVIDIA A100-SXM4 GPU} (80\,GB VRAM) and 24 logical cores of an \textbf{AMD EPYC 7742} processor.

The maximum wall-clock time for individual experiments was approximately 18 hours. We observed that computational resource usage remained consistent across all baselines, primarily because our setup avoids the use of task-specific heads. Furthermore, the retentive architecture inherent to the SABLE backbone—and by extension, Oryx—enables efficient scaling with respect to the number of agents. Consequently, our multi-task variants retain this computational efficiency even as environment complexity increases.

\newpage
\section{Primary Task Splits}
\label{sec:task-splits}

To evaluate the generalization capabilities of our approach, we curated distinct sets of training and testing scenarios for each environment. The specific scenarios comprising each train/test split are detailed in  \autoref{tab:task-splits}.

\begin{table}[htbp]
\centering
\caption{\textbf{Train/Test Task Splits for All Environments.} We list the specific scenarios used for training and out-of-distribution generalization testing.}
\label{tab:task-splits}
\begin{tabular}{@{}llcp{\dimexpr 0.6\textwidth-2\tabcolsep}@{}}
\toprule
\textbf{Environment} & \textbf{Split} & \textbf{\# of Tasks} & \textbf{Scenarios} \\
\midrule
\multirow{2}{*}{LBF} & Train & 5 & \texttt{\{8x8-2p-2f, 10x10-3p-3f, 15x15-3p-5f, 15x15-4p-5f, 16x16-5p-6f\}} \\
\cmidrule(l){2-4}
 & Test & 4 & \texttt{\{12x12-4p-5f, 14x14-3p-3f, 17x17-6p-8f, 17x17-8p-10f\}} \\
\midrule
\multirow{2}{*}{RWARE} & Train & 8 & \texttt{\{tiny-2ag, tiny-4ag, tiny-8ag, small-2ag, small-4ag, medium-2ag, medium-8ag, xlarge-8ag\}} \\
\cmidrule(l){2-4}
 & Test & 12 & \texttt{\{tiny-16ag, small-16ag, small-32ag, medium-16ag, medium-32ag, large-16ag, xlarge-16ag, xlarge-32ag,  giant-32ag, colossal-32ag, titanic-16ag, titanic-32ag\}} \\
\midrule
\multirow{2}{*}{SMAX} & Train & 5 & \texttt{\{3m, 6m, 5m vs 6m, 10m, 7m vs 8m\}} \\
\cmidrule(l){2-4}
 & Test & 7 & \texttt{\{4m, 5m , 9m vs 10m, 8m vs 9m, 10m vs 11m, 10m vs 12m, 13m vs 15m\}} \\
\midrule
\multirow{2}{*}{Connector} & Train & 10 & \texttt{\{12x12x4a, 15x15x3a, 18x18x4a, 21x21x5a, 24x24x6a, 27x27x7a, 30x30x10a, 33x33x11a, 36x36x12a, 39x39x13a\}} \\
\cmidrule(l){2-4}
 & Test & 11 & \texttt{\{42x42x18a, 45x45x23a, 48x48x20a, 51x51x28a, 54x54x30a, 57x57x32a, 60x60x33a, 63x63x35a, 66x66x40a, 69x69x43a, 72x72x45a\}} \\
\bottomrule
\end{tabular}
\end{table}

\newpage
\section{Hyperparameters}
\label{sec:hyperparameters}

This section details the hyperparameters used for our experiments. 

\begin{table}[h!]
\centering
\small
\caption{Default network settings for each environment.}
\label{tab:network-hyperparams}
\begin{tabular}{@{}lcccc@{}}
\toprule
\textbf{Parameter} & \textbf{LBF} & \textbf{Connector} & \textbf{RWARE}  & \textbf{SMAX}\\
\midrule
Model embedding dimension       & 512       & 512       & 512   & 512\\
Number of transformer heads     & 4         & 4         & 4         & 4 \\
Number of transformer blocks    & 1         & 1         & 1         & 1 \\
HL-Gauss value support & [-1, 1]   & [-1, 1]   & [-20, 20]  & [-0.5, 2.5] \\
HL-Gauss number of bins         & 51        & 51        & 51        & 51 \\
Sable's decay scaling factor    & 0.8       & 0.8       & 0.8       & 0.8 \\
\bottomrule
\end{tabular}
\end{table}

\begin{table}[h!]
\centering
\small
\caption{Default training settings.}
\label{tab:training-hyperparams}
\begin{tabular}{@{}lc@{}}
\toprule
\textbf{Hyperparameter} & \textbf{Value} \\
\midrule
Number of training updates & \num{60000} \\
Number of evaluations & 600 \\
Number of evaluation episodes & 32 \\
Number of absolute evaluation episodes & 320 \\
Learning rate & \num{1e-3} \\
Discount ($\gamma$) & 0.99 \\
Polyak averaging coefficient ($\tau$) & 0.005 \\
Maximum gradient norm & 10 \\
Sample sequence length & 20 \\
Sample batch size & 480 \\
Value temperature & 1000 \\
Policy temperature & 0.1 \\
Critic loss coefficient & 1 \\
\bottomrule
\end{tabular}
\end{table}

\begin{table}[h!]
\centering
\small
\caption{MT-Oryx specific settings.}
\label{tab:oryx-hyperparams}
\begin{tabular}{@{}lc@{}}
\toprule
\textbf{Hyperparameter} & \textbf{Value} \\
\midrule
Value temperature & 1000 \\
Policy temperature & 0.1 \\
Critic loss coefficient & 1 \\
HL-Gauss smoothing ratio & 0.75 \\
\bottomrule
\end{tabular}
\end{table}

\begin{table}[h!]
\centering
\small
\caption{MT-CQL-Sable specific settings.}
\label{tab:cqlsable-hyperparams}
\begin{tabular}{@{}lc@{}}
\toprule
\textbf{Hyperparameter} & \textbf{Value} \\
\midrule
CQL loss coefficient & 10 \\
HL-Gauss smoothing ratio & 0.75 \\
\bottomrule
\end{tabular}
\end{table}

\newpage
\section{Architecture design ablations}
\label{sec:design-choices-ablations}

\textbf{HL-Gauss.}
To test the effect of using HL-Gauss~\citep{farebrother2024stop} for multi-task learning, we conduct an ablation on the full set of \texttt{RWARE} training tasks where we run MT Oryx and MT CQL-Sable with and without HL-Gauss for value function learning (e.g. standard TD mean-squared-error). We compare the algorithms on multi-task \texttt{RWARE} since the task-to-task variance in episode returns is significant and therefore more challenging to accurately learn a multi-task value function. As shown in \autoref{fig:hl_gauss}, using HL-Gauss leads to slightly better performance ($\approx$ \num{8}\% improvement) on test tasks for MT Oryx, while the effect on MT CQL-Sable is marginal.

\begin{figure}[h]
    \centering
    \begin{subfigure}[t]{0.3\linewidth}
        \includegraphics[width=\linewidth]{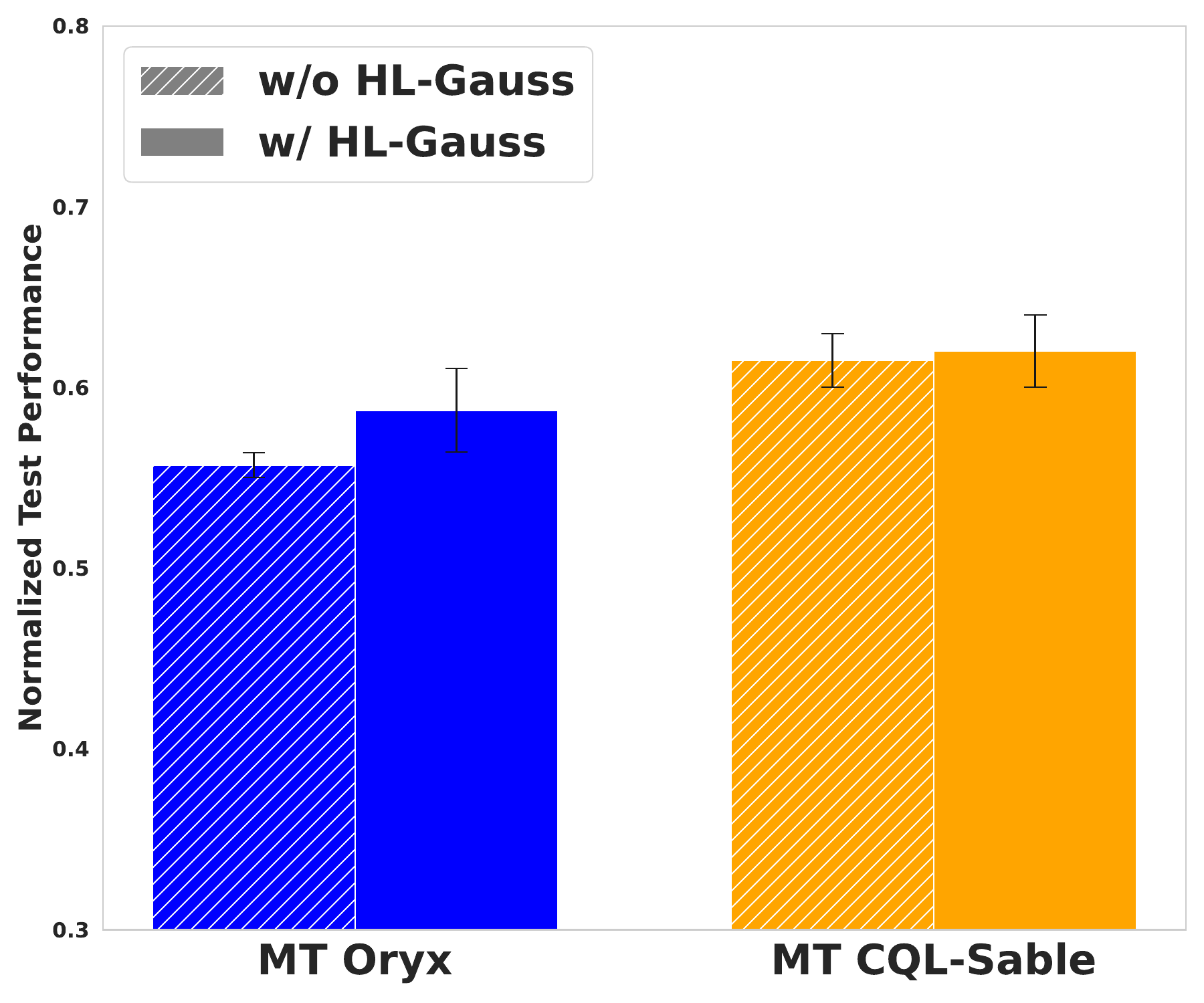}
        \caption{HL-Gauss}
        \label{fig:hl_gauss}
    \end{subfigure}
    \hfill
    \begin{subfigure}[t]{0.3\linewidth}
        \includegraphics[width=\linewidth]{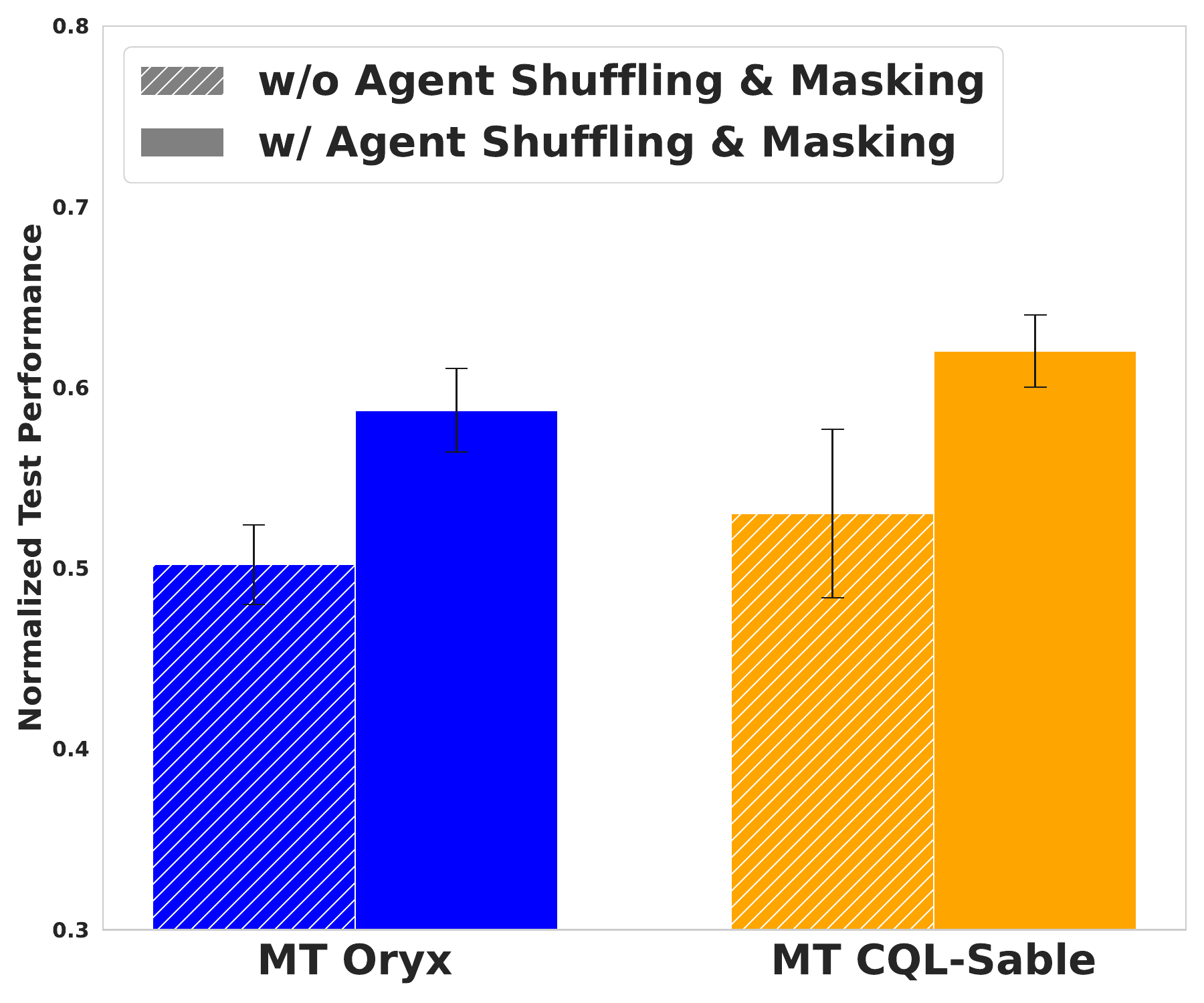}
        \caption{Agent Masking \& Shuffling}
        \label{fig:ablation_suffling}
    \end{subfigure}
    \hfill
    \begin{subfigure}[t]{0.3\linewidth}
        \includegraphics[width=\linewidth]{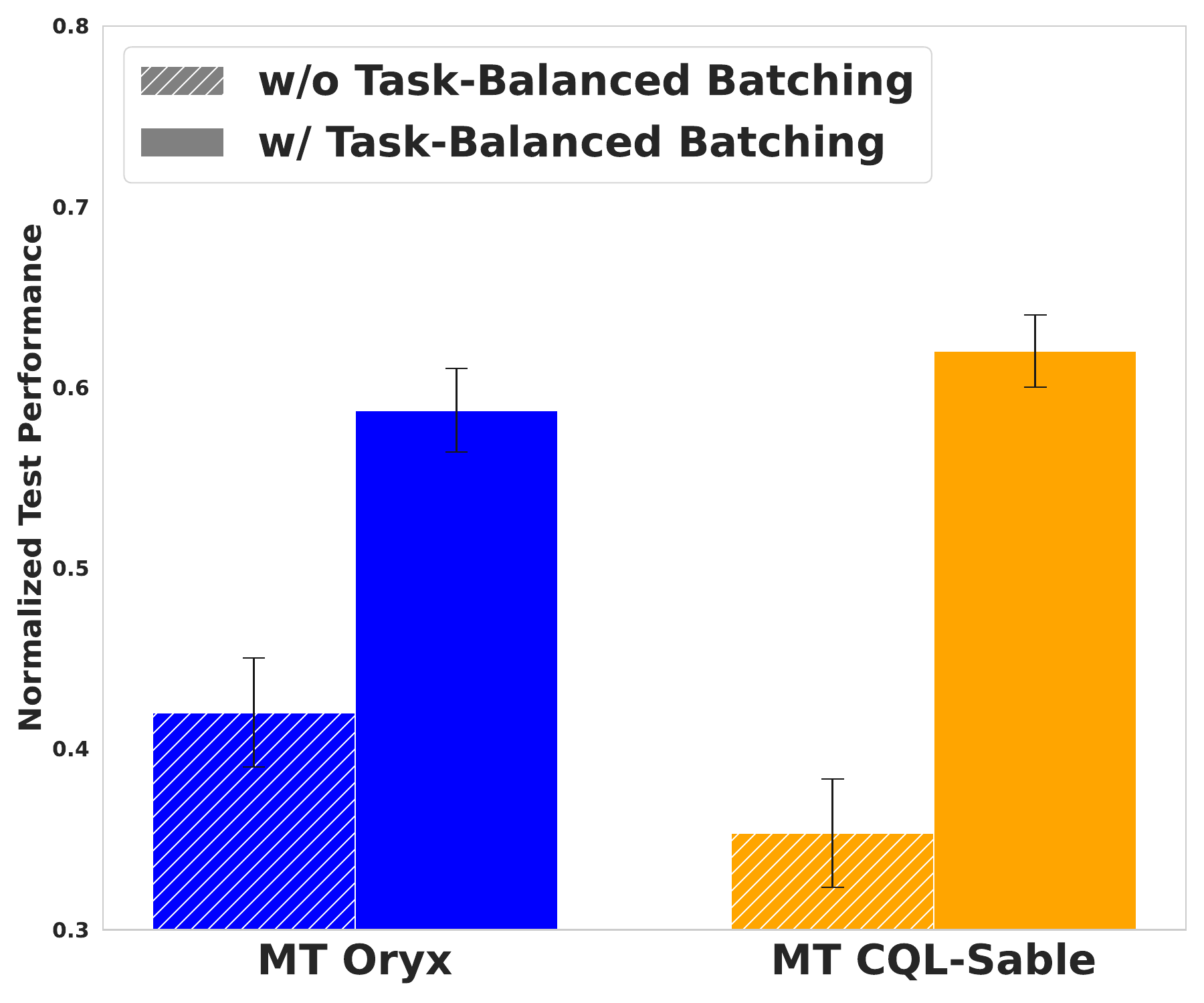}
        \caption{Task-Balanced Batching}
        \label{fig:ablation_balance}
    \end{subfigure}
    \caption{\textit{Ablation studies}. \textbf{Left:} Using HL-Gauss improves test performance for MT Oryx by $\approx$8\%, while the effect on MT CQL-Sable is marginal. \textbf{Middle:} Disabling agent masking and shuffling reduces test performance by $\approx$16\% on average for both algorithms. \textbf{Right:} Removing task-balanced batching has the highest impact with $\approx$ 37\% drop in test performance on average for both MT Oryx and MT CQL-Sable.}
\end{figure}

\textbf{Agent shuffling and masking.} To test the impact of \emph{not} masking and shuffling agents we conduct a similar ablation to above on \texttt{RWARE}. We observe decrease in performance of $\approx$ 16\% on average for both algorithms on the test tasks, when we do not mask and shuffle agents (see \autoref{fig:ablation_suffling}).

\textbf{Task-balanced batching.} Finally, we conducted an ablation on how we sample data from the multi-task dataset. In the first case we use our proposed task-balanced batching method, which includes a fair mix of samples from each task in every batch. In the alternative approach we choose a random task at each update step and sample a full batch from the chose single task. The results in \autoref{fig:ablation_balance} shows a 37\% decrease in test performance on average for both MT Oryx and MT CQL-Sable without task-balanced batching.

\clearpage
\section{Datasets}
\label{sec:appendix_datasets}

\subsection{Dataset Release Plan}
To guarantee the long-term reproducibility of this project, we will upload all of our datasets to a public HuggingFace repository\footnote{\url{https://sites.google.com/view/multi-task-marl}}. This will be done upon publication of this work. 

\subsection{Dataset Statistics}
The following sections detail the statistics of the offline datasets for the RWARE, Connector, SMAX, and LBF environments used in our experiments. Datasets were generated by recording rollouts from an online Sable \citep{mahjoub2025sable} agent at different intervals during its training. All data is collected from fixed intervals over training using an evaluation policy to vary the amount of data collected while maintaining a standard set of policies to sample from. For RWARE, we also create multiple datasets of different sizes by varying the number of evaluations sampled in order to perform our data-scaling experiments.

\subsubsection{RWARE}
For our data-scaling experiments in the RWARE environment, we generated three offline datasets of varying sizes. The datasets were constructed by collecting 122, 244, and 610 evaluation rollouts from a pre-trained online Sable agent \citep{mahjoub2025sable}. \autoref{tab:rware_dataset_stats} provides detailed statistics for each dataset size across all RWARE scenarios, illustrating how the number of episodes and transitions scales with the number of collected rollouts.

\subsubsection{Connector}

For the Connector environment, we generated 10 distinct offline datasets, one for each training scenario. Each dataset contains approximately 10 million transitions. The data collection process involved recording evaluation rollouts at 50 different checkpoints during the training of an online Sable agent. At each checkpoint, we generated 160 rollouts of 1280 timesteps each, resulting in a total of $50 \times 160 \times 1280 \approx 10.24$ million transitions per scenario. The ten scenarios used to create these datasets are listed in \autoref{tab:connector_dataset_stats}.

\subsubsection{SMAX}
For the SMAX environment, we generated five distinct offline datasets corresponding to each training scenario, with each dataset comprising approximately 50 million transitions. We collected this data by recording evaluation rollouts from a Sable agent during online training.

\subsubsection{LBF}
For LBF we collected all the the training data from an online Sable run for each LBF scenario.

\begin{table}[htbp]
\centering
\caption{
\textbf{RWARE dataset statistics across different data collection checkpoints.}
We report the total number of episodes and timesteps (transitions) for each scenario, corresponding to datasets created from 122, 244, and 610 evaluation rollouts.
}
\label{tab:rware_dataset_stats}
\resizebox{\textwidth}{!}{
\sisetup{group-separator={,}, group-minimum-digits=4}
\begin{tabular}{@{}l S[table-format=5.0] S[table-format=7.0] S[table-format=5.0] S[table-format=8.0] S[table-format=5.0] S[table-format=8.0]@{}}
\toprule
& \multicolumn{2}{c}{\textbf{122 Rollouts}} & \multicolumn{2}{c}{\textbf{244 Rollouts}} & \multicolumn{2}{c}{\textbf{610 Rollouts}} \\
\cmidrule(lr){2-3} \cmidrule(lr){4-5} \cmidrule(lr){6-7}
\textbf{Scenario Name} & {\textbf{Episodes}} & {\textbf{Timesteps}} & {\textbf{Episodes}} & {\textbf{Timesteps}} & {\textbf{Episodes}} & {\textbf{Timesteps}} \\
\midrule
\texttt{tiny-2ag}      & 15616 & 7493913  & 31232 & 14934862 & 78080 & 37382071 \\
\texttt{tiny-4ag}      & 15616 & 6492381  & 31232 & 13208433 & 78080 & 33110502 \\
\texttt{tiny-8ag}      & 15616 & 4704862  & 31232 & 9748756  & 78080 & 24647669 \\
\texttt{small-2ag}     & 15616 & 7511771  & 31232 & 15091627 & 78080 & 37504501 \\
\texttt{small-4ag}     & 15616 & 6611283  & 31232 & 13496720 & 78080 & 33733571 \\
\texttt{medium-2ag}    & 15616 & 7320030  & 31232 & 15062621 & 78080 & 37700791 \\
\texttt{medium-8ag}    & 15616 & 2502476  & 31232 & 5148947  & 78080 & 12747091 \\
\texttt{xlarge-8ag}    & 15616 & 5816385  & 31232 & 11008538 & 78080 & 29167762 \\
\bottomrule
\end{tabular}
}
\end{table}

\begin{table}[htbp]
    \centering
    \caption{
        \textbf{Connector dataset statistics.} We generated a separate dataset of approximately 10.24 million transitions for each of the ten training scenarios.
    }
    \label{tab:connector_dataset_stats}
    \begin{tabular}{@{}ll@{}}
        \toprule
        \textbf{Scenario Name} & \textbf{Total Timesteps} \\
        \midrule
        \texttt{12x12x4a}     & $\approx 10.24 \times 10^6$ \\
        \texttt{15x15x3a}     & $\approx 10.24 \times 10^6$ \\
        \texttt{18x18x4a}     & $\approx 10.24 \times 10^6$ \\
        \texttt{21x21x5a}     & $\approx 10.24 \times 10^6$ \\
        \texttt{24x24x6a}     & $\approx 10.24 \times 10^6$ \\
        \texttt{27x27x7a}     & $\approx 10.24 \times 10^6$ \\
        \texttt{30x30x10a}    & $\approx 10.24 \times 10^6$ \\
        \texttt{33x33x11a}    & $\approx 10.24 \times 10^6$ \\
        \texttt{36x36x12a}    & $\approx 10.24 \times 10^6$ \\
        \texttt{39x39x13a}    & $\approx 10.24 \times 10^6$ \\
        \bottomrule
    \end{tabular}
\end{table}

\begin{table}[htbp]
    \centering
    \caption{
        \textbf{SMAX dataset statistics.} We generated a separate dataset of approximately 49.5 million timesteps for each of the five training scenarios.
    }
    \label{tab:smax_dataset_stats}
    \begin{tabular}{@{}ll@{}}
        \toprule
        \textbf{Scenario Name} & \textbf{Total Timesteps} \\
        \midrule
        \texttt{3m}           & $\approx 49.5 \times 10^6$ \\
        \texttt{6m}           & $\approx 49.5 \times 10^6$ \\
        \texttt{5m vs 6m}   & $\approx 49.5 \times 10^6$ \\
        \texttt{10m}          & $\approx 49.5 \times 10^6$ \\
        \texttt{7m vs 8m}   & $\approx 49.5 \times 10^6$ \\
        \bottomrule
    \end{tabular}
\end{table}

\begin{table}[htbp]
    \centering
    We observe that performance on the training tasks remains high across all environments, even as the number of tasks increases
    \caption{
        \textbf{LBF dataset statistics.} We generated a separate dataset of approximately 4 million transitions for each of the 5 training scenarios.
    }
    \label{tab:connector_dataset_stats2}
    \begin{tabular}{@{}ll@{}}
        \toprule
        \textbf{Scenario Name} & \textbf{Total Timesteps} \\
        \midrule
        \texttt{8x8-2p-2f}     & $\approx 3.99 \times 10^6$ \\
        \texttt{10x10-3p-3f}     & $\approx 3.99 \times 10^6$ \\
        \texttt{15x15-3p-3f}     & $\approx 3.99 \times 10^6$ \\
        \texttt{15x15-4p-5f}     & $\approx 3.99 \times 10^6$\\
        \texttt{16x16-5p-6f}     & $\approx 3.99 \times 10^6$ \\
        \bottomrule
    \end{tabular}
\end{table}

\end{document}